%% file: neurips_2026.tex
\documentclass{article}

 \usepackage[main, final]{neurips_2026}

\usepackage[utf8]{inputenc} 
\usepackage[T1]{fontenc}    
\usepackage{hyperref}       
\usepackage{url}            
\usepackage{booktabs}       
\usepackage{amsfonts}       
\usepackage{nicefrac}       
\usepackage{microtype}      
\usepackage{xcolor}         

\usepackage{amssymb}
\usepackage{bm}
\usepackage{caption}
\usepackage{enumitem}
\usepackage{subfigure}
\usepackage{graphicx}

\usepackage{amssymb}
\usepackage{amsmath}
\usepackage{appendix}
\usepackage{wrapfig}
\usepackage{threeparttable}
\usepackage{amsthm}
\newtheorem{myPro}{Proposition}
\newtheorem{appendixPro}{Proposition}
\usepackage{algorithm}
\usepackage{algpseudocode}
\usepackage{multirow}
\usepackage{makecell}
\usepackage{balance}
\usepackage{wrapfig}
\usepackage{float}
\usepackage{xspace}
\usepackage{subcaption} 
\usepackage{pifont}

\usepackage{graphicx}

\usepackage{caption}  

\title{Paloma: Learning Residual Deviations via Phase\\–Amplitude Modulation for Time Series Forecasting}
\title{Paloma: Phase-Conditioned Residual Modulation for Time Series Forecasting}

\author{
Jingru Fei$^1$, Wei Fan$^2$, Kun Yi$^3$, Qi Zhang$^4$, Zhendong Niu$^1$\thanks{Corresponding author}\\
$^1$Beijing Institute of Technology,
$^2$University of Auckland,\\
$^3$State Information Center,
$^4$Tongji University,\\
\{jingrufei, zniu\}@bit.edu.cn, wei.fan@auckland.ac.nz, \\
kunyi.cn@gmail.com, zhangqi\_cs@tongji.edu.cn
}

\begin{document}

\maketitle

\begin{abstract}
\input{0_abstract}
\end{abstract}

\input{1_introduction_v4}
\input{2_related_work_v2}

\input{3_preliminary_v2}
\input{4_method_v2}
\input{5_experiments}

\input{6_conclusion_v2}

\bibliography{ref}
\bibliographystyle{unsrt}

\newpage
\appendix
\input{7_appendix}

\clearpage
\input{checklist.tex}

\end{document}

%% file: 0_abstract.tex
Explicitly modeling periodicity has proven to be an efficient approach for time series forecasting, where the residual, obtained after removing the periodic component from the time series, is typically treated as unstructured noise.
However, in this work, we show that such residuals still exhibit structured, \textit{phase-dependent} deviations, as supported by both empirical observations and theoretical analysis.
This suggests that residuals should not be treated merely as unstructured noise, but instead as \textit{phase-conditioned signals} whose variations arise in both direction and magnitude.
Motivated by this perspective and inspired by phase and amplitude modulation in classical signal processing, we propose \textit{Paloma}, a simple yet effective framework for modeling residual deviations.
Specifically, Paloma comprises two complementary modules: \textit{a Phase Rotation module}, which performs phase-conditioned rotations to capture directional changes in the feature space, and \textit{an Amplitude Scaling module}, which applies phase-conditioned affine transformations to model magnitude variations.
We further abstract the two modules into a general-purpose plugin, termed the \textit{Paloma technique}, which can be readily integrated into existing forecasting architectures to modulate input sequences in a phase-aware manner, enabling backbone models to capture phase-dependent variations in both direction and magnitude.
Extensive experiments on twelve datasets demonstrate that Paloma achieves state-of-the-art and robust forecasting performance. Additional results show that the Paloma technique consistently improves a wide range of baseline models, validating its effectiveness as a general-purpose component.
Our Code is available at \url{https://github.com/TSTS13/Paloma}.


%% file: 1_introduction_v4.tex
\section{Introduction}


Time series forecasting plays a fundamental role in many real-world applications, including energy systems~\cite{zhou2022fedformer,fan2023dish,fei2026olivia}, urban transportation planning~\cite{bai2020adaptive,wang2025investigating,qiu2024tfb}, and financial markets~\cite{liu2021pyraformer,feng2019temporal,qiu2025DBLoss,qiu2026bridging}. In these scenarios, accurately capturing the temporal dynamics is critical for effective decision-making, resource allocation, and system reliability. Recent advances in deep learning have developed various forecasting architectures for modeling temporal dynamics, including Transformer-based~\cite{nie2022time,liu2023itransformer,qiu2025duet,qiu2026dag}, MLP-based~\cite{zeng2023transformers,fei2025amplifier}, CNN-based~\cite{luo2024moderntcn,liu2022scinet}, and RNN-based~\cite{lin2025segrnn} models. Built upon different modeling principles, these methods have substantially enriched the methodological landscape of time series forecasting.

Among the diverse temporal structures in real-world time series, periodicity is one of the most prominent and informative patterns~\cite{lin2024sparsetsf,fan2022depts}. Explicitly modeling periodicity has therefore attracted increasing attention~\cite{lin2024cyclenet,zhang2026freqcycle,cao2026enhancing}, as it provides a principled and parameter-efficient way to capture recurring temporal structures while achieving competitive performance with substantially lower computational cost than Transformer-based models. For example, CycleNet~\cite{lin2024cyclenet} represents periodic components as learnable recurrent cycles, decomposing a time series into a phase-aligned periodic component and a residual term. The residual is then modeled by a simple predictor, such as a single-layer linear model or a shallow MLP. This design implicitly assumes that, once the periodic component is removed, the remaining residual is largely unstructured noise and can be adequately captured by a lightweight predictor. \textit{However, it remains unclear whether such residuals are truly unstructured, or whether they still contain systematic, phase-dependent patterns}.

\begin{figure}[t]
    \centering
    \includegraphics[width=1\linewidth]{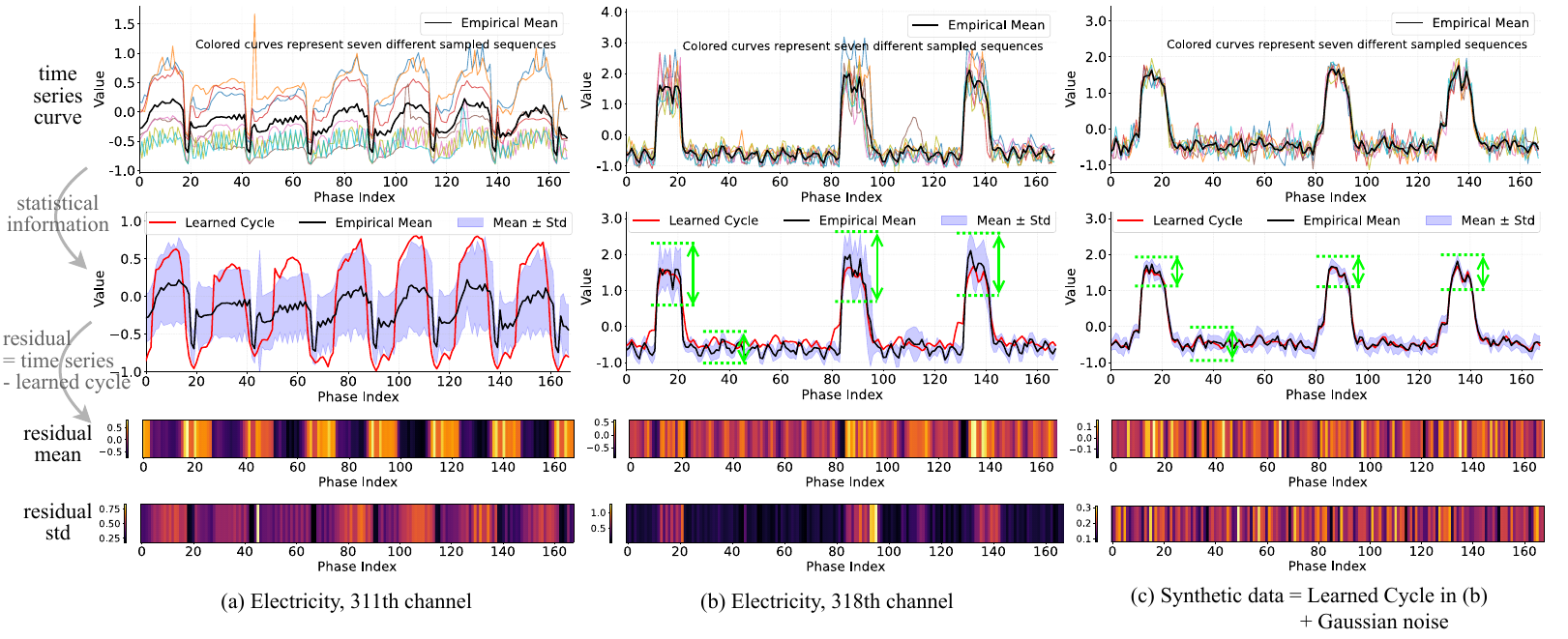}
    \vspace{-4mm}
    \caption{Visualization of structured residual deviations.
 (a) and (b) present two representative channels from the Electricity dataset.
From top to bottom, the rows show sampled sequences, learned cycles from CycleNet, and the residual mean and standard deviation visualized as one-dimensional heatmaps under a cycle length of 168. (c) is constructed by combining the learned cycle from (b) with Gaussian noise, with the row-wise definitions consistent with those in (a) and (b).
Additional experiments for (c) and visualizations on the Weather dataset are provided in Appendix~\ref{appendix_motivation}.}
    \label{fig:motivation}
    \vspace{-6.5mm}
\end{figure}



To investigate this question, we visualize the learned cycles produced by CycleNet~\cite{lin2024cyclenet} and the corresponding residual terms on the real-world Electricity dataset, as shown in Figure~\ref{fig:motivation}(a) and (b), which present two representative yet distinct cases.
In Figure~\ref{fig:motivation}(a), the first row shows multiple cycle instances with significantly varying peak magnitudes, leading to inconsistent periodic patterns and hindering the learning of a unified periodic structure; consequently, the learned cycle in the second row (red curve) deviates substantially from the empirical mean (black curve).
This clearly indicates that the residual term is not unstructured noise, but instead exhibits structured deviations, as evidenced by the heatmap visualizations of its mean and standard deviation, which reveal clear phase-dependent patterns.
In contrast to Figure~\ref{fig:motivation}(a), Figure~\ref{fig:motivation}(b) presents a markedly different scenario.
The multiple cycle instances in the first row follow a consistent periodic pattern, and the learned cycle in the second row closely matches the empirical mean, indicating that the periodic component is well captured.
Nevertheless, further analysis still reveals persistent phase-dependent structures in the residual, evident in both its mean and standard deviation, particularly around phase indices near 20, 90, and 140, with more pronounced effects observed in the standard deviation.




We further examine this phenomenon from a counterfactual perspective by assuming that the residual term is purely random noise, and analyzing the resulting time series and its statistical properties. To this end, we conduct a controlled experiment, as illustrated in Figure~\ref{fig:motivation}(c), where synthetic sequences are constructed by combining the learned cycle from Figure~\ref{fig:motivation}(b) with Gaussian noise as the residual component.
A clear contrast can be observed between Figure~\ref{fig:motivation}(b) and (c) at the locations highlighted by the green arrows: in (c), the standard deviation remains nearly invariant across phases, whereas in (b), it consistently increases around peak regions.
In summary, Figure~\ref{fig:motivation}(a) and (b) show that \textit{the residual consistently exhibits phase-dependent deviations even after the periodic component is effectively removed}, while the comparison between Figure~\ref{fig:motivation}(b) and (c) further suggests that \textit{the residual cannot be regarded as purely random noise}.
These observations are theoretically supported in Section~\ref{sec_theoretical_analysis}, where we show that such structured deviations inherently persist in the residual.

To this end, we interpret the residual as a phase-conditioned signal with structured variations across phases. In classical signal processing~\cite{anderson2013digital,joris2004neural,proakis2001digital,oppenheim1999discrete}, 
phase refers to the position of a periodic signal within its cycle at a specific moment and is usually represented by an angle, where different phase angles correspond to different directional states, while amplitude controls the strength of signal variations. Motivated by this perspective, we propose \textit{Paloma}, a simple yet effective framework that models residual deviations via phase--amplitude modulation through two complementary modules: 
(i) a \textit{Phase Rotation module} that applies phase-conditioned rotations to model directional transformations in the feature space, and (ii) an \textit{Amplitude Scaling module} that performs phase-conditioned affine transformations to capture variation intensity across phases. 
Moreover, the two modules can be combined into a general-purpose plugin, termed the \textit{Paloma technique}.
Beyond modeling residuals, Paloma technique can be readily integrated into existing forecasting architectures and directly applied to input sequences, enabling them to capture phase-dependent variations in both direction and magnitude that are otherwise overlooked.
Extensive experiments on twelve datasets demonstrate that Paloma achieves state-of-the-art forecasting performance, while additional plug-in experiments further verify the effectiveness and generality of the \textit{Paloma technique} across diverse forecasting architectures. In summary, our contributions are as follows:
\begin{itemize}[leftmargin=*]
    \item We identify that the residual, obtained after removing the periodic component in time series, exhibits structured, phase-dependent deviations rather than unstructured noise, and provide theoretical analysis explaining their persistence.

    \item We propose \textit{Paloma}, a simple framework that models residual deviations via phase–amplitude modulation, featuring a Phase Rotation module and an Amplitude Scaling module, and further abstract the two modules as a general-purpose plugin, termed the \textit{Paloma technique}.
    
    \item Extensive experiments demonstrate that Paloma achieves state-of-the-art and robust performance, and validate the effectiveness of Paloma technique across diverse forecasting models.
\end{itemize}
\vspace{-2mm}

%% file: 2_related_work_v2.tex
\vspace{-3mm}
\section{Related Work}
\vspace{-2mm}
\paragraph{Temporal Periodicity Modeling.}
Temporal periodicity has long been recognized as a key factor in improving forecasting accuracy by capturing recurring patterns in time series data.
Representative models, including Autoformer~\cite{wu2021autoformer}, FEDformer~\cite{zhou2022fedformer}, and DLinear~\cite{zeng2023transformers}, leverage seasonal–trend decomposition (STD) to disentangle the original time series into seasonal and trend components, followed by independent modeling of each component.
Building on these approaches, Leddam~\cite{yu2024revitalizing} replaces the fixed moving-average kernel in STD with a learnable decomposition kernel, providing greater flexibility and yielding improved performance.
Recent advances have explored leveraging the intrinsic periodicity of time series as prior knowledge for efficient forecasting. 
SparseTSF~\cite{lin2024sparsetsf} introduces a cross-period sparse forecasting technique to decouple cycles and trends, achieving competitive performance with low computational cost. 
Along this line of research, CycleNet~\cite{lin2024cyclenet} introduces Residual Cycle Forecasting (RCF), which models periodic patterns using learnable recurrent cycles and predicts the resulting residuals. 
While effective and efficient, the residual component is typically modeled using a simple predictor, such as a linear layer or a shallow MLP, implicitly assuming that it mainly consists of unstructured noise. However, our analysis suggests that the residual may still contain structured patterns beyond the modeled periodic component.
This motivates a more refined modeling of the residual component.


\vspace{-2mm}
\paragraph{Modulation–Inspired Time Series Modeling.}
Recent advances in time series modeling have increasingly drawn inspiration from classical signal processing~\cite{proakis2001digital,oppenheim1999discrete,asselin1972frequency,yi2025survey}.
FilterNet~\cite{yi2024filternet} introduces learnable frequency filters to enhance the utilization of full-spectrum information, effectively performing frequency-domain shaping.
In contrast, TimeEmb~\cite{xia2025timeemb} applies frequency filtering specifically to the time-varying component, targeting the modeling of non-stationary temporal dynamics.
Subsequently, FreqCycle~\cite{zhang2026freqcycle} enhances mid-to-high frequency components via learnable filters and adaptive weighting, improving the modeling of short-term variations.
These filtering processes can be viewed as a form of implicit spectral modulation.
Unlike these approaches, which implicitly realize modulation through frequency filtering, our method explicitly modulates the signal by separately modeling its phase and amplitude variations.
Beyond frequency filtering and spectral modulation, PhaseFormer~\cite{niu2025phaseformer} explores the role of phase information in time series forecasting by replacing conventional patch-based tokenization with phase-based tokens, emphasizing phase as a more compact and coherent representation of temporal dynamics.
Different from PhaseFormer, Paloma treats phase as a conditional variable for controlling residual deviations.


%% file: 3_preliminary_v2.tex
\section{Preliminaries}
\subsection{Problem Formulation}
\paragraph{Time Series Forecasting.}
Given a time series input $\mathbf{X} = \{\mathbf{x}_l\}_{l=0}^{T-1} \in \mathbb{R}^{T \times N}$ and its corresponding future sequence $\mathbf{Y} = \{\mathbf{x}_l\}_{l=T}^{T+\tau-1} \in \mathbb{R}^{\tau \times N}$, where $N$ denotes the number of channels; $T$ and $\tau$ denote the lookback window length and prediction horizon, respectively.
The time series forecasting task is to learn a parametric model $f_{\theta}(\cdot)$ that maps the historical observations $\mathbf{X}$ to the predicted future sequence $\hat{\mathbf{Y}}$, i.e., $\hat{\mathbf{Y}}=f_{\theta}(\mathbf{X})$, with $\theta$ being the model parameters.

\paragraph{Phase Index Construction.}\label{section_phase_index}
Given a time series input $\mathbf{X} = \{\mathbf{x}_l\}_{l=0}^{T-1} \in \mathbb{R}^{T \times N}$, we define a corresponding phase index sequence 
$\mathbf{i} = \{i_l\}_{l=0}^{T-1}$, 
where $i_l \in \{0, 1, \ldots, W-1\}$, and $W$ denotes the cycle length.
For a sequence starting at absolute time $t_{0}$, the phase index at relative step $l$ is defined as 
$i_{l}=\left[\left(t_{0} \bmod W\right)+ l \right] \bmod W,\quad l=0,1, \ldots, T-1$.
This formulation assigns each time step to a specific phase within the cycle, enabling alignment of temporal observations across different periods and facilitating phase-conditioned modeling of temporal dynamics.

\subsection{Understanding Phase-Dependent Deviations in Residuals}\label{sec_theoretical_analysis}
Given a time series input $\mathbf{X}\in \mathbb{R}^{T \times N}$ with phase index sequence $\mathbf{i} = \{i_l\}_{l=0}^{T-1}$, a common strategy for explicitly modeling periodic 
components, adopted by methods such as CycleNet~\cite{lin2024cyclenet}, is to employ a learnable prototype matrix $\mathbf{Q} \in \mathbb{R}^{W \times N}$, where the $k$-th row $\mathbf{Q}_k \in \mathbb{R}^{N}$ represents the reference pattern 
associated with phase index $k$.
To better understand the role of the phase-dependent prototype matrix $\mathbf{Q}$, we consider a simplified prediction setting in which the prediction depends solely on $\mathbf{Q}$, i.e., $\hat{\mathbf{x}}_l = \mathbf{Q}_{i_l}$, where $\hat{\mathbf{x}}_l \in \mathbb{R}^N$ denotes the predicted value at time step $l$ (with $\mathbf{x}_l \in \mathbb{R}^N$ being the corresponding ground-truth).

\begin{myPro}[\textbf{Optimal Phase Prototype under MSE}]
\label{proposition_1}
Under the above setting, consider the widely used mean squared error (MSE) objective
$\mathcal{L}(\mathbf{Q}) = \mathbb{E}\left[\left\|\mathbf{x}_l - \mathbf{Q}_{i_l}\right\|_2^2\right]$.
Then the optimal solution satisfies
$\mathbf{Q}_k^* = \mathbb{E}\left[\mathbf{x}_l \mid i_l = k\right], \quad \forall k \in \{0,1,\dots,W-1\}$.
\end{myPro}

Proposition~\ref{proposition_1} shows that, under the MSE objective, the prototype matrix $\mathbf{Q}$ learns the phase-conditioned mean of the time series, i.e., the dominant first-order periodic structure associated with each phase.
However, removing this structure alone is insufficient to guarantee statistical invariance of the residual across phases. 
In particular, higher-order statistics, such as variance, may still depend on the phase index, as illustrated in Figure~\ref{fig:motivation}. 
As a result, the residual may retain structured, phase-dependent deviations beyond the mean, reflecting intrinsic heterogeneity in the underlying data distribution. 
We formalize this phenomenon in the following proposition.

\begin{myPro}[\textbf{Exact Preservation of Phase-Conditional Second-Order Structure}]
\label{proposition_2}
Let $\mathbf{Q}^*$ be the optimal solution given in Proposition~\ref{proposition_1}, namely,
$\mathbf{Q}_k^* = \mathbb{E}[\mathbf{x}_l \mid i_l = k]$ for all $k \in \{0,1,\dots,W-1\}$ such that $\mathbb{P}(i_l = k) > 0$.
Define the residual as $\mathbf{r}_l = \mathbf{x}_l - \mathbf{Q}^*_{i_l}$.
Assume that the conditional second moments of $\mathbf{x}_l$ given $i_l$ exist.
Then, for any $k \in \{0,1,\dots,W-1\}$ such that $\mathbb{P}(i_l = k) > 0$, the residual satisfies
$\mathbb{E}[\mathbf{r}_l \mid i_l = k] = \mathbf{0}$, and
$\mathrm{Cov}(\mathbf{r}_l \mid i_l = k) = \mathrm{Cov}(\mathbf{x}_l \mid i_l = k)$.
\end{myPro}

Propositions~\ref{proposition_1} and~\ref{proposition_2} together imply that, even when the phase-conditioned mean is optimally eliminated under the MSE objective (with the MAE counterpart provided in Appendix~\ref{appendix_proposition_MAE}), the residual inherently preserves higher-order phase-dependent structures. 
Consequently, when the conditional covariance $\mathrm{Cov}(\mathbf{x}_l \mid i_l = k)$ varies across phases, this variability is reflected in the residual, giving rise to the structured deviations observed in Figure~\ref{fig:motivation}. 
In practice, such phase-dependent effects are often further amplified, as the learned periodic pattern $\mathbf{Q}$ may deviate from the optimal solution due to finite data, optimization challenges, and model capacity limitations.

%% file: 4_method_v2.tex
\section{Method}

In this section, we elaborate on our proposed method, \textit{Paloma}, which is inspired by phase and amplitude modulation in signal processing~\cite{anderson2013digital,joris2004neural,proakis2001digital,oppenheim1999discrete} and improves time series forecasting by explicitly modeling phase- and amplitude-dependent deviations in the residual.

\vspace{-1.5mm}
\subsection{Overview of Paloma}
\vspace{-1.5mm}
As illustrated in Figure~\ref{fig:framework}, given a time series input $\mathbf{X} \in \mathbb{R}^{T \times N}$, we first apply instance normalization~\cite{kim2021reversible} to mitigate non-stationarity~\cite{kim2021reversible,liu2022non} in time series data, 
and construct the phase indices $\mathbf{i}$ based on the cycle length $W$, following the formulation described in Section~\ref{section_phase_index}.
Based on $\mathbf{i}$, we retrieve phase-aligned rows from three learnable matrices: the prototype matrix $\mathbf{Q}$, the phase control matrix $\mathbf{P}$, and the amplitude control matrix $\mathbf{A}$, obtaining $\tilde{\mathbf{Q}}$, $\tilde{\mathbf{P}}$, and $\tilde{\mathbf{A}}$, respectively, as described in Sections~\ref{section_matrix_Q},~\ref{section_matrix_P},~\ref{section_matrix_A}.
We then remove the periodic component to obtain the residual term $\mathbf{R} = \mathbf{X} - \tilde{\mathbf{Q}}$.
The residual is processed by two complementary modules: the \textit{Phase Rotation Module}, which uses $\tilde{\mathbf{P}}$ to model directional variations via phase-conditioned rotations, and the \textit{Amplitude Scaling Module}, which uses $\tilde{\mathbf{A}}$ to capture magnitude variations through adaptive scaling and shifting.
Finally, the outputs of the two modules are fused and projected, combined with the input embedding via a residual connection, and mapped to the prediction horizon, followed by inverse instance normalization to produce the final forecasting result $\hat{\mathbf{Y}} \in \mathbb{R}^{\tau \times N}$.

\vspace{-2mm}
\paragraph{Obtaining the Residual Term.}\label{section_matrix_Q}
Before introducing the two proposed modules designed to model the residual, we first provide additional details on how the residual is obtained.
Given $N$ channels and a priori cycle length $W$, we first define a learnable prototype matrix $\mathbf{Q} \in \mathbb{R}^{W \times N}$, where the $k$-th row $\mathbf{Q}_k \in \mathbb{R}^{N}$ represents the reference pattern associated with phase index $k$.
The matrix $\mathbf{Q}$ is initialized with a standard Gaussian distribution and updated via backpropagation.
Then the residual term $\mathbf{R} \in \mathbb{R}^{T \times N}$ is defined as $\mathbf{R} = \mathbf{X} - \tilde{\mathbf{Q}}$,
where $\tilde{\mathbf{Q}} = [\mathbf{Q}_{i_0}, \mathbf{Q}_{i_1}, \ldots, \mathbf{Q}_{i_{T-1}}]^\top \in \mathbb{R}^{T \times N}$ denotes the phase-aligned prototype matrix.


\begin{figure}[t]
    \centering
    \includegraphics[width=0.92\linewidth]{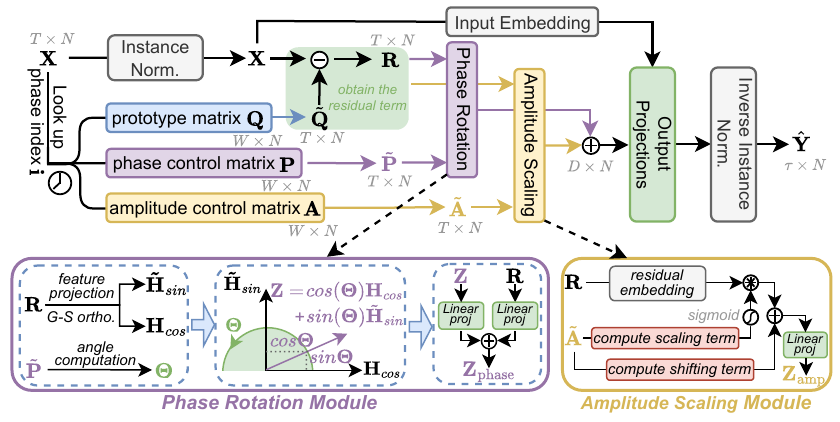}
    \vspace{-1.5mm}
    \caption{Overall architecture of Paloma. Paloma is a lightweight framework for modeling residual deviations, centered on (i) the Phase Rotation module, which captures phase-conditioned directional variations via rotation in the feature space, where G-S ortho. denotes Gram–Schmidt orthogonalization~\cite{golub2013matrix,greub2012linear}, and (ii) the Amplitude Scaling module, which models phase-dependent magnitude variations through affine modulation.}
    \label{fig:framework}
    \vspace{-4.5mm}
\end{figure}

\vspace{-2mm}
\subsection{Phase Rotation Module}\label{section_matrix_P}
\vspace{-1.5mm}
Motivated by Proposition~\ref{proposition_2}, the residual term exhibits phase-dependent second-order structure, suggesting that the residual may undergo phase-dependent geometric variations in the feature space. To capture such variations, we design a phase rotation module that performs a phase-conditioned transformation of residual features.
Given the residual term $\mathbf{R} \in \mathbb{R}^{T \times N}$, 
we first map it into two latent feature representations along the temporal dimension:
\begin{equation}
\mathbf{H}_{\cos} = \operatorname{Linear}(\mathbf{R}) \in \mathbb{R}^{D \times N}, \quad
\mathbf{H}_{\sin} = \operatorname{Linear}(\mathbf{R}) \in \mathbb{R}^{D \times N},
\end{equation}
where $D$ denotes the hidden dimension.
The subscripts $\cos$ and $\sin$ are used to emphasize that these two components will be combined through a trigonometric rotation, forming a phase-dependent transformation.
To ensure a well-defined rotation, we construct an orthogonal basis from $(\mathbf{H}_{\cos}, \mathbf{H}_{\sin})$ using Gram–Schmidt orthogonalization~\cite{golub2013matrix,greub2012linear}:
\begin{equation}
\mathbf{H}_{\sin}^{\perp}
=
\mathbf{H}_{\sin}
-
\frac{\langle \mathbf{H}_{\sin}, \mathbf{H}_{\cos} \rangle}{\|\mathbf{H}_{\cos}\|^2}
\mathbf{H}_{\cos} \in \mathbb{R}^{D \times N},
\end{equation}
where $\langle \cdot, \cdot \rangle$ denotes the inner product, and the resulting $\mathbf{H}_{\sin}^{\perp}$ is orthogonal to $\mathbf{H}_{\cos}$.
This orthogonalization stabilizes the subsequent rotation and avoids degenerate solutions where the two directions become collinear~\cite{golub2013matrix,greub2012linear}.

To generate the phase-conditioned rotation angle, we introduce a learnable phase control matrix 
$\mathbf{P} \in \mathbb{R}^{W \times N}$, where the $k$-th row $\mathbf{P}_k \in \mathbb{R}^{N}$ 
encodes phase-specific information that modulates the rotation angle for phase $k$.
The matrix $\mathbf{P}$ is defined and learned in the same manner as $\mathbf{Q}$.
Accordingly, the phase-aligned control matrix is constructed as
$\tilde{\mathbf{P}} = [\mathbf{P}_{i_0}, \mathbf{P}_{i_1}, \ldots, \mathbf{P}_{i_{T-1}}]^\top \in \mathbb{R}^{T \times N}$.
We then compute the phase-dependent rotation angle $\boldsymbol{\Theta}$ and obtain the rotated representation $\mathbf{Z}$ as follows:
\begin{equation}
\boldsymbol{\Theta} = \pi \cdot \tanh(\operatorname{Linear}(\tilde{\mathbf{P}})) \in \mathbb{R}^{D \times N},
\end{equation}
\begin{equation}
\mathbf{Z} = \cos(\boldsymbol{\Theta}) \odot \mathbf{H}_{\cos} + \sin(\boldsymbol{\Theta}) \odot \mathbf{H}_{\sin}^{\perp} \in \mathbb{R}^{D \times N}.
\end{equation}
This formulation performs a phase-conditioned rotation of residual features within a learned orthogonal subspace, enabling the model to capture phase-dependent geometric variations.
The rotated representation $\mathbf{Z}$ is further processed and integrated with $\mathbf{R}$ through a residual connection:
\begin{equation}
\mathbf{Z}_{\text{phase}} = \mathrm{LayerNorm}\left(\operatorname{Linear}(\mathbf{Z}) + \operatorname{Linear}(\mathbf{R})\right) \in \mathbb{R}^{D \times N},
\end{equation}
where $\mathrm{LayerNorm}(\cdot)$ denotes layer normalization applied along the hidden dimension $D$.

\vspace{-1.5mm}
\subsection{Amplitude Scaling Module}\label{section_matrix_A}
\vspace{-1.5mm}
Proposition~\ref{proposition_2} indicates that the residual exhibits phase-dependent second-order structure.
From a covariance perspective, this structure is characterized by both principal directions and variance magnitudes, corresponding to the eigenvectors and eigenvalues of the covariance matrix, respectively.
While the phase rotation module models phase-dependent variations in feature direction, variations in residual magnitude also play an important role in characterizing such structure. 
To capture this complementary aspect, we introduce an amplitude scaling module based on phase-conditioned affine modulation.
Specifically, we introduce a learnable amplitude control matrix $\mathbf{A} \in \mathbb{R}^{W \times N}$, which parameterizes phase-conditioned scaling and shifting of residual features, where the $k$-th row $\mathbf{A}_k \in \mathbb{R}^{N}$ encodes the amplitude modulation parameters for phase $k$.
The matrix $\mathbf{A}$ shares the same learning mechanism as the prototype matrix $\mathbf{Q}$ and the phase control matrix $\mathbf{P}$, ensuring a consistent parameterization across modules.
Accordingly, the phase-aligned amplitude control matrix is constructed as
$
\tilde{\mathbf{A}} = [\mathbf{A}_{i_0}, \mathbf{A}_{i_1}, \ldots, \mathbf{A}_{i_{T-1}}]^\top \in \mathbb{R}^{T \times N}.
$

We first embed the residual $\mathbf{R}$ into a hidden space as $\mathbf{R}_e = \operatorname{Embedding}(\mathbf{R}) \in \mathbb{R}^{D \times N}$, where $\operatorname{Embedding}(\cdot)$ denotes a linear projection, and $\mathbf{R}_e$ serves as the base representation for subsequent phase-conditioned modulation.
We then compute phase-dependent scaling and shifting terms from the phase-aligned amplitude control matrix $\tilde{\mathbf{A}}$ as follows:
\begin{equation}
\mathbf{S} = \operatorname{Sigmoid}(\operatorname{Linear}(\tilde{\mathbf{A}})) \in \mathbb{R}^{D \times N}, \quad
\mathbf{B} = \operatorname{Linear}(\tilde{\mathbf{A}}) \in \mathbb{R}^{D \times N}.
\end{equation}
The amplitude-modulated representation is obtained via a phase-conditioned affine transformation followed by a linear output projection:
\begin{equation}
\mathbf{Z}_{\mathrm{amp}} = \operatorname{Linear}(\mathbf{S} \odot \mathbf{R}_e + \mathbf{B}) \in \mathbb{R}^{D \times N},
\end{equation}
where $\odot$ denotes element-wise multiplication.
The amplitude scaling module captures phase-dependent magnitude variations in the residual representation via affine modulation, complementing the phase rotation module to jointly characterize residual deviations in terms of directional and magnitude components.

\vspace{-2.5mm}
\subsection{Output Projections}
\vspace{-1mm}

We fuse the outputs of the phase rotation and amplitude scaling modules as $\mathbf{H} = \mathbf{Z}_{\mathrm{phase}} + \mathbf{Z}_{\mathrm{amp}} \in \mathbb{R}^{D \times N}$, and further process $\mathbf{H}$ with a lightweight projection to introduce nonlinearity and enable feature interaction, yielding $\mathbf{H}' = \operatorname{GELU}(\operatorname{Linear}(\mathbf{H})) \in \mathbb{R}^{D \times N}$.
Meanwhile, the original input sequence $\mathbf{X}$ is projected into the same hidden space via a linear embedding, yielding $\mathbf{X}_{e} = \operatorname{Embedding}(\mathbf{X}) \in \mathbb{R}^{D \times N}$, where $\mathbf{X}_{e}$ acts as a shortcut pathway that preserves the original input information and facilitates stable optimization.
Finally, the hidden representation $\mathbf{H}'$ is combined with the input embedding $\mathbf{X}_{e}$ via a residual connection and projected to the prediction horizon:
\begin{equation}
\hat{\mathbf{Y}} = \operatorname{Linear}(\mathbf{H}' + \mathbf{X}_{e}) \in \mathbb{R}^{\tau \times N}.
\end{equation}

\vspace{-2mm}
\subsection{Complexity Analysis}
\vspace{-1mm}
Paloma adopts a channel-independent~\cite{nie2022time,zeng2023transformers} linear architecture, where the main computations are performed by linear projections along the temporal dimension and lightweight projections in the hidden space. 
The temporal projections in the Phase Rotation and Amplitude Scaling modules incur a cost of $\mathcal{O}(NTD)$, while the hidden-space projections introduce an additional cost of $\mathcal{O}(ND^2)$. The final prediction head contributes $\mathcal{O}(ND\tau)$. Other operations, including residual construction, phase-aligned prototype lookup, element-wise modulation, normalization, and the Gram--Schmidt orthogonalization step, are linear in the input or hidden size and are dominated by the projection operations. Therefore, the overall computational complexity of Paloma is
$\mathcal{O}\big(N(TD + D^2 + D\tau)\big)$.
Importantly, Paloma maintains linear complexity with respect to both the lookback length $T$ and the number of channels $N$, making it suitable for efficient multivariate time series forecasting.

%% file: 5_experiments.tex
\vspace{-2mm}
\section{Experiments}
\vspace{-2mm}
\subsection{Experimental Setup}
\input{tables/input_96_avg}

\vspace{-1.5mm}
\paragraph{Baselines and Datasets}
We consider 17 representative and widely recognized state-of-the-art forecasting models as baselines, including 
(i) Linear-based: GTR~\cite{cao2026enhancing}, RAFT~\cite{han2025retrieval}, CycleNet~\cite{lin2024cyclenet}, DLinear~\cite{zeng2023transformers}, SparseTSF~\cite{lin2024sparsetsf}, OLinear~\cite{yue2025olinear}, TimeMixer~\cite{wang2024timemixer}; 
(ii) Transformer-based: TimeXer~\cite{wang2024timexer}, iTransformer~\cite{liu2023itransformer}, PatchTST~\cite{nie2022time}; 
(iii) Frequency-based: FilterNet~\cite{yi2024filternet}, FreTS~\cite{yi2023frequency}, FITS~\cite{xu2023fits}, Fredformer~\cite{piao2024fredformer}; 
(iv) TCN-based: TimesNet~\cite{wu2022timesnet}; 
(v) Two different training frameworks: DistDF~\cite{wang2025distdf} and FreDF~\cite{wang2024fredf}.
To evaluate the effectiveness of our model, we conduct extensive experiments on eight benchmark datasets for long-term forecasting, including ETT (four subsets)~\cite{zhou2021informer}, Electricity~\cite{wu2021autoformer}, Traffic~\cite{wu2021autoformer}, Solar-Energy~\cite{zhou2021informer}, and Weather~\cite{wu2021autoformer}, and PEMS (four subsets)~\cite{liu2022scinet} for short-term forecasting.
Detailed information on the baselines and datasets can be found in Appendix~\ref{appendix_baselines} and Appendix~\ref{appendix_datasets}.
And the implementation details are provided in Appendix~\ref{appendix_implementation_detials}.


\vspace{-1mm}
\subsection{Main Results}
\vspace{-1mm}
Table~\ref{tab:input_96_avg} reports the forecasting performance for both long-term and short-term settings, where Paloma consistently outperforms a wide range of state-of-the-art methods across diverse datasets.
We further report the standard deviation of Paloma across five runs with different random seeds in Table~\ref{tab:input_96_additional}, highlighting its stable and robust performance.
We attribute this performance gain to the proposed phase–amplitude modulation mechanism for modeling structured residual deviations.
The phase rotation module captures phase-dependent variations by rotating feature representations in a learned feature space according to the cycle position, while the amplitude scaling module models the corresponding magnitude changes.
Together, they enable Paloma to capture phase-dependent residual deviations, resulting in improved accuracy and robustness.


We also observe relatively higher MSE value on the Traffic dataset, which is likely related to the complex spatiotemporal correlations and temporal lag effects inherent in traffic data, particularly affecting long-term forecasting performance. 
While Paloma adopts a channel-independent modeling strategy that effectively captures per-channel temporal patterns, it may be less suited for modeling complex channel interactions compared to methods explicitly designed for such dependencies, such as iTransformer~\cite{liu2023itransformer}. 
Nevertheless, Paloma remains highly competitive in short-term traffic forecasting, achieving 5 top-1 and 3 top-2 results across 8 evaluation metrics on the four PEMS datasets.


\input{tables/ablation_avg}
\input{tables/general_technique}

\vspace{-1mm}
\subsection{Model Analysis}
\vspace{-1mm}
\paragraph{Ablation Studies.}
We conduct ablation studies to evaluate the contributions of the Phase Rotation (PR) and Amplitude Scaling (AS) modules across six datasets, as shown in Table~\ref{tab:ablation_avg}. 
Specifically, removing either the PR or AS module leads to noticeable performance drops, underscoring the importance of modeling phase-dependent variations via directional transformations (PR) and magnitude modulation (AS).
When both modules are removed, the performance further deteriorates, revealing their complementary roles in modeling structured residual deviations. 
These results validate the effectiveness of the proposed phase--amplitude modulation mechanism and emphasize the necessity of explicitly modeling structured residual deviations.


\vspace{-2mm}
\paragraph{Paloma Tech. as plugin.}
Beyond the ablation study within the Paloma model, we further evaluate the generalizability of the proposed phase--amplitude modulation mechanism by applying it as a plugin module to other baselines. 
Specifically, we integrate the combination of the Phase Rotation and Amplitude Scaling modules, referred to as \emph{Paloma Tech.}, into several representative forecasting architectures, including Transformer-based models (iTransformer~\cite{liu2023itransformer} and PatchTST~\cite{nie2022time}), the Linear-based model DLinear~\cite{zeng2023transformers}, and the frequency-based model FilterNet~\cite{yi2024filternet}. 
As shown in Table~\ref{tab:general_tech}, Paloma Tech. consistently improves the performance of all baseline models across different datasets and prediction horizons, demonstrating that the proposed modulation mechanism is architecture-agnostic and effective for capturing phase-dependent variations, whether applied to residuals or directly to input sequences.
Notably, the performance gains are particularly significant when integrating Paloma Tech. into DLinear. 
Although DLinear benefits from seasonal--trend decomposition to separate coarse temporal components, its forecasting process still mainly relies on simple linear transformations and lacks an explicit mechanism for modeling phase-dependent temporal variations. 
By introducing the proposed phase--amplitude modulation mechanism, Paloma Tech. effectively complements this limitation, enabling DLinear to capture more structured phase-aware temporal dynamics.

\begin{figure}[h]
    \centering
    \includegraphics[width=0.93\linewidth]{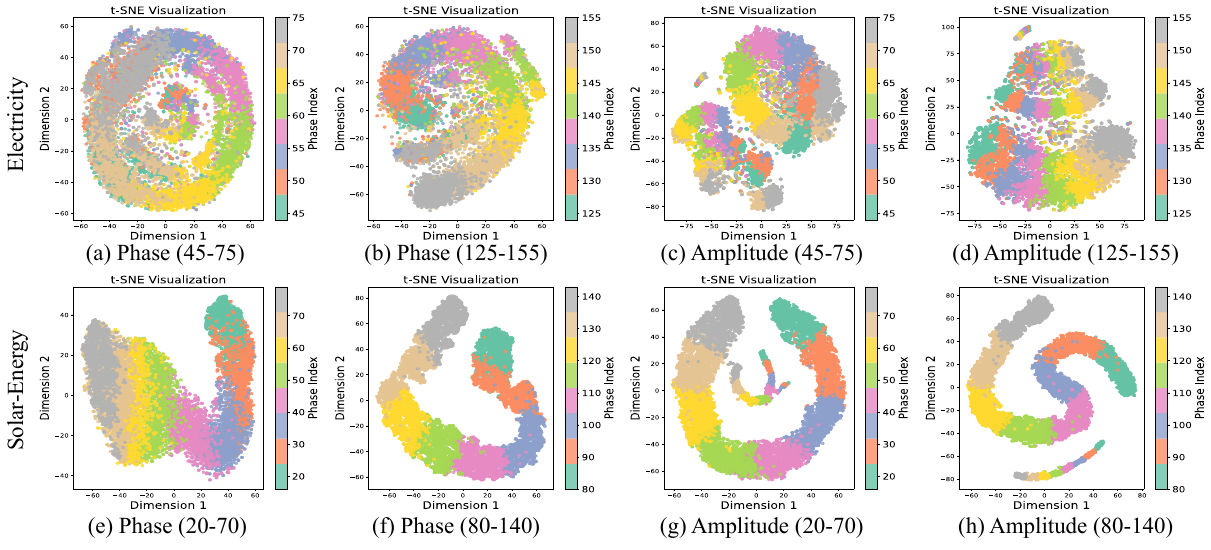}
    \caption{t-SNE visualizations of $\mathbf{Z}_{\mathrm{phase}}$ and $\mathbf{Z}_{\mathrm{amp}}$ on Electricity and Solar-Energy datasets under the setting of lookback window $T = 96$ and prediction length $\tau = 96$. 
The inherent periodicities of the Electricity and Solar-Energy datasets are 168 and 144, respectively, and the numbers in parentheses denote the corresponding ranges of phase indices.}
    \label{fig:tSNE_ECL_solar}
    \vspace{-3mm}
\end{figure}

\vspace{-2.5mm}
\subsection{Visualization Analysis}
\vspace{-1.5mm}
To better understand the roles of the proposed Phase Rotation and Amplitude Scaling modules, we conduct t-SNE visualization analysis of the learned representations $\mathbf{Z}_{\text{phase}}$ and $\mathbf{Z}_{\text{amp}}$ on the Electricity and Solar-Energy datasets, as shown in Figures~\ref{fig:tSNE_ECL_solar}. 
This analysis aims to qualitatively examine how the proposed phase--amplitude modulation mechanism organizes the representation space in a structured manner, where phase-aware structures and magnitude-related variations are modeled by complementary components

\vspace{-1.5mm}
\paragraph{Phase representations.}
As illustrated in Figures~\ref{fig:tSNE_ECL_solar}(a)--(b) and (e)--(f), the phase representations $\mathbf{Z}_{\text{phase}}$ exhibit clear phase-aware organization in the hidden space. 
Samples with nearby phase indices are arranged along continuous trajectories, forming distinct geometric patterns, such as a spiral-like structure on the Electricity dataset and an irregular U-shaped structure on the Solar-Energy dataset. 
The color progression along the trajectories indicates that the learned representations follow the intrinsic periodic ordering of the data, suggesting that the Phase Rotation module captures phase-dependent variations and arranges features according to their cycle positions.

\vspace{-2.5mm}
\paragraph{Amplitude representations.}
As shown in Figures~\ref{fig:tSNE_ECL_solar}(c)--(d) and (g)--(h), the amplitude representations $\mathbf{Z}_{\text{amp}}$ exhibit organization patterns that are clearly distinct from those of $\mathbf{Z}_{\text{phase}}$. 
On the Electricity dataset (Figures~\ref{fig:tSNE_ECL_solar}(c)--(d)), the amplitude representations appear more scattered and intermixed, reflecting the more complex and heterogeneous consumption patterns. 
In contrast, On the Solar-Energy dataset (Figures~\ref{fig:tSNE_ECL_solar}(g)--(h)), $\mathbf{Z}_{\text{amp}}$ exhibits a structured yet non-uniform geometry, forming a dominant continuous manifold alongside a thinner secondary branch. 
The main manifold captures the majority of samples and reflects the principal variation in residual magnitudes, while the secondary branch corresponds to a smaller subset with distinct amplitude characteristics. 
Taken together, these patterns suggest that the Amplitude Scaling module captures heterogeneous amplitude variations conditioned on phase.

\vspace{-2.5mm}
\subsection{Efficiency Analysis}
\vspace{-1.5mm}
To comprehensively assess both effectiveness and efficiency, we further conduct efficiency experiments on the Weather dataset (21 channels), as reported in Table~\ref{tab:efficiency_weather}. 
To ensure a fair comparison, all models are evaluated under a unified training framework with a fixed batch size of 32. 
From the results, Paloma achieves a favorable balance between predictive performance and computational cost. 
\input{tables/efficiency_weather}
Compared to lightweight linear models such as GTR, CycleNet, and DLinear, Paloma introduces only a moderate increase in parameters and MACs, while delivering notably improved forecasting accuracy. 
In contrast to Transformer-based models (e.g., iTransformer and PatchTST), Paloma demonstrates advantages in both efficiency and effectiveness, requiring substantially lower computational resources while achieving better predictive performance. 
Efficiency results on the ETTh1 dataset (7 channels) and the Traffic dataset (862 channels) are provided in Tables~\ref{tab:efficiency_ETTh1} and~\ref{tab:efficiency_traffic}, further confirming the efficiency of Paloma across datasets with varying channel sizes. 

%% file: tables/input_96_avg.tex
\begin{table*}[t]
    \centering
    \caption{Performance comparison with a fixed lookback window $T=96$. 
    The prediction lengths are set as $\tau \in \{12, 24, 48, 96\}$ for PEMS datasets for short-term forecasting, and $\tau \in \{96, 192, 336, 720\}$ for the others for long-term forecasting. 
    Results are averaged over the four prediction lengths. The best and second-best results are highlighted in \textcolor{red}{\textbf{red}} and \textcolor{blue}{\textbf{blue}}, respectively. 
    Full results are provided in Table \ref{tab:input_96_appendix}, and additional baseline results are reported in Table \ref{tab:input_96_additional}.}
    \vspace{-1.5mm}
    \scalebox{0.48}{
    \begin{tabular}{c c|c c|c c|c c| c c|c c |c c |c c |c c |c c |c c |c c  }
    \toprule[2pt]
       \multicolumn{2}{c|}{Models}  &\multicolumn{2}{c|}{\textbf{Paloma}} &\multicolumn{2}{c|}{DistDF} &\multicolumn{2}{c|}{GTR} &\multicolumn{2}{c|}{RAFT} &\multicolumn{2}{c|}{FreDF}  &\multicolumn{2}{c|}{TimeXer}  &\multicolumn{2}{c|}{CycleNet} &\multicolumn{2}{c|}{FilterNet} &\multicolumn{2}{c|}{iTransformer} &\multicolumn{2}{c|}{PatchTST}  &\multicolumn{2}{c}{DLinear}  \\

        \multicolumn{2}{c|}{}  &\multicolumn{2}{c|}{[\textbf{ours}]} &\multicolumn{2}{c|}{[ICLR'26]} &\multicolumn{2}{c|}{[ICLR'26]} &\multicolumn{2}{c|}{[ICML'25]} &\multicolumn{2}{c|}{[ICLR'25]}  &\multicolumn{2}{c|}{[NeurIPS'24]} &\multicolumn{2}{c|}{[NeurIPS'24]} &\multicolumn{2}{c|}{[NeurIPS'24]} &\multicolumn{2}{c|}{[ICLR'24]} &\multicolumn{2}{c|}{[ICLR'23]} &\multicolumn{2}{c}{[AAAI'23]} \\
       
       \cmidrule(r){1-2}\cmidrule(lr){3-4}\cmidrule(lr){5-6}\cmidrule(lr){7-8}\cmidrule(lr){9-10}\cmidrule(lr){11-12}\cmidrule(lr){13-14}\cmidrule(lr){15-16}\cmidrule(lr){17-18}\cmidrule(lr){19-20}\cmidrule(lr){21-22}\cmidrule(lr){23-24}
       
       \multicolumn{2}{c|}{Metrics} &MSE &MAE &MSE &MAE &MSE &MAE &MSE &MAE &MSE &MAE &MSE &MAE &MSE &MAE &MSE &MAE &MSE &MAE &MSE &MAE &MSE &MAE   \\
    
         \midrule[1pt]
         \multicolumn{2}{c|}{ETTm1}  &\textcolor{red}{\textbf{0.363}}   &\textcolor{red}{\textbf{0.372}}  &0.378 &0.394  &\textcolor{blue}{\textbf{0.367}} &\textcolor{blue}{\textbf{0.389}} &0.381 &0.400  &0.392 &0.399  &0.382 &0.397 &0.379 &0.396 &0.392 &0.401  &0.407  &0.410  &0.387  &0.400  &0.403  &0.407 \\  
         \midrule[1pt]

         \multicolumn{2}{c|}{ETTm2}  &\textcolor{blue}{\textbf{0.267}}   &\textcolor{red}{\textbf{0.308}}  &0.277 &0.321  &0.268  &0.315  &0.281  &0.330  &0.278 &0.319  &0.274 &0.322 &\textcolor{red}{\textbf{0.266}} &\textcolor{blue}{\textbf{0.314}}  &0.285 &0.328 &0.288 &0.332 &0.281 &0.326 &0.350 &0.401 \\ 
         \midrule[1pt]

         \multicolumn{2}{c|}{ETTh1}  &\textcolor{red}{\textbf{0.421}}   &\textcolor{red}{\textbf{0.416}}  &0.430 &\textcolor{blue}{\textbf{0.429}}  &0.439 &0.434 &\textcolor{blue}{\textbf{0.428}} &0.433  &0.437 &0.435 &0.437 &0.437 &0.457 &0.441  &0.441 &0.439  &0.454 &0.448 &0.469 &0.455 &0.456 &0.452 \\ 
         \midrule[1pt]

         \multicolumn{2}{c|}{ETTh2}  &\textcolor{red}{\textbf{0.361}}   &\textcolor{red}{\textbf{0.386}}  &\textcolor{blue}{\textbf{0.367}} &\textcolor{blue}{\textbf{0.393}}  &0.372 &0.400 &0.382 &0.410 &0.371 &0.396 &0.368 &0.396 &0.388 &0.409  &0.383 &0.407 &0.383 &0.407 &0.387 &0.407 &0.559 &0.515 \\ 
         \midrule[1pt]

         \multicolumn{2}{c|}{Electricity}  &\textcolor{red}{\textbf{0.166}}   &\textcolor{red}{\textbf{0.251}}  &0.172 &0.267 &\textcolor{red}{\textbf{0.166}} &0.260 &0.175 &0.272  &0.170 &\textcolor{blue}{\textbf{0.259}} &0.171 &0.270 &\textcolor{blue}{\textbf{0.168}} &\textcolor{blue}{\textbf{0.259}} &0.173 &0.268 &0.178 &0.270 &0.205 &0.290 &0.212 &0.300 \\ 
         \midrule[1pt]

         \multicolumn{2}{c|}{Traffic}  &0.461   &\textcolor{red}{\textbf{0.260}}  &\textcolor{blue}{\textbf{0.417}} &\textcolor{blue}{\textbf{0.279}} &0.470 &0.280 &\textcolor{red}{\textbf{0.414}} &0.284  &0.421 &\textcolor{blue}{\textbf{0.279}} &0.466 &0.287 &0.472 &0.301  &0.463 &0.310 &0.428 &0.282 &0.481 &0.300 &0.625 &0.383 \\ 
         \midrule[1pt]

         \multicolumn{2}{c|}{Weather}  &\textcolor{red}{\textbf{0.231}}   &\textcolor{red}{\textbf{0.262}}   &0.248 &0.275  &\textcolor{blue}{\textbf{0.239}} &\textcolor{blue}{\textbf{0.268}} &0.270 &0.309  &0.254 &0.274  &0.241 &0.271 &0.243 &0.271  &0.245 &0.272  &0.258 &0.278 &0.259 &0.273 &0.265 &0.317 \\ 
         \midrule[1pt]

         \multicolumn{2}{c|}{Solar-Energy}  &\textcolor{blue}{\textbf{0.197}}   &\textcolor{red}{\textbf{0.234}}   &0.207   &0.255    &\textcolor{red}{\textbf{0.194}} &\textcolor{blue}{\textbf{0.245}} &0.301 &0.303  &0.211   &0.262   &0.237 &0.302 &0.210 &0.261  &0.235 &0.266  &0.233 &0.262 &0.270 &0.307 &0.330 &0.401 \\ 
         \midrule[1pt]

        
       \multicolumn{2}{c|}{PEMS03} &\textcolor{blue}{\textbf{0.094}}   &\textcolor{blue}{\textbf{0.194}}  &0.123   &0.233    &\textcolor{red}{\textbf{0.087}} &\textcolor{red}{\textbf{0.189}} &0.144 &0.230  &0.113   &0.219   &0.112 &0.214 &0.118 &0.226 &0.145 &0.251  &0.113 &0.222 &0.180 &0.291 &0.278 &0.375 \\ 
         \midrule[1pt]

        \multicolumn{2}{c|}{PEMS04}  &\textcolor{red}{\textbf{0.085}}   &\textcolor{red}{\textbf{0.182}}   &0.092   &0.195     &\textcolor{blue}{\textbf{0.087}}  &\textcolor{blue}{\textbf{0.189}} &0.104 &0.210   &0.096   &0.202   &0.105 &0.209 &0.119 &0.232  &0.146 &0.258  &0.111 &0.221 &0.195 &0.307 &0.295 &0.388 \\ 
         \midrule[1pt]

         \multicolumn{2}{c|}{PEMS07}  &\textcolor{blue}{\textbf{0.078}}   &\textcolor{red}{\textbf{0.164}}    &0.084   &0.177     &\textcolor{red}{\textbf{0.076}} &\textcolor{blue}{\textbf{0.169}}  &0.094 &0.193    &0.089   &0.182  &0.085 &0.182 &0.113 &0.214  &0.123 &0.229  &0.101 &0.204 &0.211 &0.303 &0.329 &0.396 \\ 
         \midrule[1pt]

        \multicolumn{2}{c|}{PEMS08}  &\textcolor{red}{\textbf{0.111}}  &\textcolor{red}{\textbf{0.194}}  &0.151  &0.244  &\textcolor{blue}{\textbf{0.142}} &\textcolor{blue}{\textbf{0.222}}  &0.151 &0.234 &0.173  &0.255   &0.175 &0.250 &0.150 &0.246 &0.172 &0.260  &0.150 &0.226 &0.280 &0.321 &0.379 &0.416 \\ 
         \midrule[1pt]

        \multicolumn{2}{c|}{$1^{st}$ Count} &7  &11  &0  &0  &4  &1  &1  &0  &0  &0 &0 &0 &1  &0  &0  &0  &0  &0  &0  &0    &0  &0  \\
         \bottomrule[2pt]
    \end{tabular}
    }
    \label{tab:input_96_avg}
    \vspace{-4mm}
\end{table*}

%% file: tables/ablation_avg.tex
    

\begin{wraptable}{r}{9.8cm}
\vspace{-5mm}
\captionsetup{type=table}
\captionof{table}{Ablation study of the Phase Rotation (PR) and Amplitude Scaling (AS) modules on six datasets. 
The reported results are averaged over all prediction horizons. Detailed results are provided in Table~\ref{tab:ablation_appendix}.}
\vspace{-1.5mm}
\centering
\scalebox{0.59}{
\begin{tabular}{ccc|cc|cc|cc|cc|cc}
\toprule[2pt]
\multirow{2}{*}{Design} &\multicolumn{2}{c|}{ETTm1} &\multicolumn{2}{c|}{ETTh1}  &\multicolumn{2}{c|}{Electricity}  &\multicolumn{2}{c|}{Weather}  &\multicolumn{2}{c|}{Traffic}  &\multicolumn{2}{c}{Solar-Energy} \\
\cmidrule(lr){2-3} \cmidrule(lr){4-5} \cmidrule(lr){6-7} \cmidrule(lr){8-9} \cmidrule(lr){10-11} \cmidrule(lr){12-13} 
 & MSE & MAE & MSE & MAE & MSE & MAE & MSE & MAE & MSE & MAE & MSE & MAE  \\
\midrule
    \textbf{Paloma}   &\textbf{0.363}  &\textbf{0.372}  &\textbf{0.421}  &\textbf{0.416}  &\textbf{0.166}  &\textbf{0.251}  &\textbf{0.231}  &\textbf{0.262}   &\textbf{0.461}  &\textbf{0.260}   &\textbf{0.197}  &\textbf{0.234}   \\
    w/o PR   &0.368   &0.374  &0.427  &0.420   &0.170  &0.258   &0.238  &0.270   &0.471  &0.267   &0.206   &0.247   \\
    w/o AS   &0.369  &0.374   &0.430  &0.424    &0.170  &0.259   &0.238  &0.270   &0.468  &0.269   &0.201  &0.243    \\
    w/o PR,AS  &0.376  &0.380   &0.428  &0.425   &0.174  &0.262   &0.243  &0.280   &0.474  &0.283   &0.209  &0.248   \\
    
\bottomrule[2pt]
\end{tabular}
}\label{tab:ablation_avg}
\vspace{-4mm}
\end{wraptable}


%% file: tables/general_technique.tex
\begin{table*}[t]
    \centering
    \caption{Generalization performance of Paloma Tech. on different models. The lookback window is fixed at 96 for all the experiments.
    }
    \vspace{-1.5mm}
    \scalebox{0.405}{
    \begin{tabular}{l | c|c c|c c|c c||c c |c c |c c ||c c |c c |c c  ||c c |c c |c c }
    \toprule[2pt]
        \multicolumn{2}{c|}{LookBack} &\multicolumn{6}{c||}{iTransformer} &\multicolumn{6}{c||}{PatchTST} &\multicolumn{6}{c||}{DLinear} &\multicolumn{6}{c}{FilterNet}\\
         \midrule[1pt]
    
       \multicolumn{2}{c|}{Setup}  &\multicolumn{2}{c|}{Original} &\multicolumn{2}{c|}{+ Paloma Tech.} &\multicolumn{2}{c||}{Imporve} &\multicolumn{2}{c|}{Original} &\multicolumn{2}{c|}{+ Paloma Tech.} &\multicolumn{2}{c||}{Imporve}  &\multicolumn{2}{c|}{Original} &\multicolumn{2}{c|}{+ Paloma Tech.} &\multicolumn{2}{c||}{Imporve}   &\multicolumn{2}{c|}{Original} &\multicolumn{2}{c|}{+ Paloma Tech.} &\multicolumn{2}{c}{Imporve}\\
       
       \cmidrule(r){1-2}\cmidrule(lr){3-4}\cmidrule(lr){5-6}\cmidrule(lr){7-8}\cmidrule(lr){9-10}\cmidrule(lr){11-12}\cmidrule(lr){13-14}\cmidrule(lr){15-16}\cmidrule(lr){17-18}\cmidrule(lr){19-20}\cmidrule(lr){21-22}\cmidrule(lr){23-24}\cmidrule(lr){25-26}
       
       \multicolumn{2}{c|}{Metrics} &MSE &MAE &MSE &MAE &MSE &MAE &MSE &MAE &MSE &MAE &MSE &MAE   &MSE &MAE &MSE &MAE &MSE &MAE  &MSE &MAE &MSE &MAE &MSE &MAE \\
    
         \midrule[1pt]
         \multirow{5}*{\rotatebox{90}{Electricity}}
         & 96  &0.150  &0.241    &0.137  &0.233    &\textcolor{red}{8.67\%}  &\textcolor{red}{3.32\%}    &0.164  &0.254    &0.137  &0.234    &\textcolor{red}{16.46\%}  &\textcolor{red}{7.87\%}   &0.195  &0.277    &0.139  &0.239    &\textcolor{red}{28.72\%}  &\textcolor{red}{13.72\%}  &0.147  &0.245  &0.137  &0.239  &\textcolor{red}{6.80\%}  &\textcolor{red}{2.45\%}      \\
         & 192  &0.164  &0.255    &0.157  &0.250    &\textcolor{red}{4.27\%}  &\textcolor{red}{1.96\%}    &0.174  &0.265    &0.155  &0.250    &\textcolor{red}{10.92\%}  &\textcolor{red}{5.66\%}   &0.194  &0.279    &0.156  &0.256    &\textcolor{red}{19.59\%}  &\textcolor{red}{8.24\%}  &0.160  & 0.250  &0.151  &0.245  &\textcolor{red}{5.63\%}  &\textcolor{red}{2.00\%}   \\
         & 336  &0.178  &0.272    &0.172  &0.267    &\textcolor{red}{3.37\%}  &\textcolor{red}{1.84\%}    &0.193  &0.285    &0.176  &0.272    &\textcolor{red}{8.81\%}  &\textcolor{red}{4.56\%}   &0.207  &0.295    &0.172  &0.275    &\textcolor{red}{16.91\%}  &\textcolor{red}{6.78\%}  &0.173  &0.267  &0.164  &0.261  &\textcolor{red}{5.20\%}  &\textcolor{red}{2.25\%}  \\
         & 720  &0.210  &0.300    &0.192  &0.290    &\textcolor{red}{8.57\%}  &\textcolor{red}{3.33\%}    &0.232  &0.320    &0.220  &0.309    &\textcolor{red}{5.17\%}  &\textcolor{red}{3.44\%}  &0.244  &0.330    &0.207  &0.309    &\textcolor{red}{15.16\%}  &\textcolor{red}{6.36\%}  &0.210  &0.309  &0.201  &0.302  &\textcolor{red}{4.29\%}  &\textcolor{red}{2.27\%}   \\
         
         \cmidrule(r){2-26}
            &Avg  &0.176  &0.267    &0.165   &0.260    &\textcolor{red}{6.25\%}   &\textcolor{red}{2.62\%}    &0.191  &0.281    &0.172  &0.266    &\textcolor{red}{9.85\%}  &\textcolor{red}{5.23\%}  &0.210  &0.295    &0.169  &0.270    &\textcolor{red}{19.76\%}  &\textcolor{red}{8.63\%}  &0.173  & 0.268  &0.163  &0.262  &\textcolor{red}{5.78\%}  &\textcolor{red}{2.24\%}  \\
            
         \midrule[1pt]
         \multirow{5}*{\rotatebox{90}{PEMS03}}
         & 12  &0.078  &0.185    &0.062  &0.164    &\textcolor{red}{20.51\%}  &\textcolor{red}{11.35\%}    &0.072  &0.179    &0.065  &0.172    &\textcolor{red}{9.72\%}  &\textcolor{red}{3.91\%}   &0.105  &0.221    &0.061  &0.165    &\textcolor{red}{41.90\%}  &\textcolor{red}{25.34\%}   &0.071  &0.177  &0.062  &0.165  &\textcolor{red}{12.68\%}  &\textcolor{red}{6.78\%}  \\
         & 24  &0.089  &0.197    &0.076  &0.180    &\textcolor{red}{14.61\%}  &\textcolor{red}{8.63\%}    &0.102  &0.213    &0.079  &0.189    &\textcolor{red}{22.55\%}  &\textcolor{red}{11.27\%}   &0.183  &0.299    &0.079  &0.186    &\textcolor{red}{56.83\%}  &\textcolor{red}{37.79\%}   &0.102  &0.213  &0.086  &0.198  &\textcolor{red}{15.69\%}  &\textcolor{red}{7.04\%}  \\
         & 48  &0.176  &0.288    &0.110  &0.213    &\textcolor{red}{37.50\%}   &\textcolor{red}{26.04\%}    &0.155  &0.263    &0.107  &0.215    &\textcolor{red}{30.97\%}  &\textcolor{red}{18.25\%}   &0.315  &0.407    &0.110  &0.214    &\textcolor{red}{65.08\%}  &\textcolor{red}{47.42\%}   &0.162  & 0.272  &0.134  &0.255  &\textcolor{red}{17.28\%}  &\textcolor{red}{6.25\%}  \\
         & 96  &0.228  &0.335    &0.178  &0.280    &\textcolor{red}{21.93\%}  &\textcolor{red}{16.42\%}    &0.204  &0.305    &0.145  &0.246    &\textcolor{red}{28.92\%}  &\textcolor{red}{19.34\%}  &0.455  &0.508    &0.146  &0.245    &\textcolor{red}{67.91\%}  &\textcolor{red}{51.77\%}   &0.244  &0.340  &0.208  &0.318  &\textcolor{red}{14.75\%}  &\textcolor{red}{6.47\%}  \\
         
         \cmidrule(r){2-26}
            &Avg  &0.143  &0.251    &0.107  &0.209    &\textcolor{red}{25.42\%}  &\textcolor{red}{16.72\%}    &0.133  &0.240    &0.099  &0.206    &\textcolor{red}{25.73\%}  &\textcolor{red}{14.38\%}   &0.265  &0.359    &0.099  &0.203    &\textcolor{red}{62.57\%}  &\textcolor{red}{43.56\%}  &0.145  &0.251  &0.123  &0.234  &\textcolor{red}{15.17\%}  &\textcolor{red}{6.77\%}  \\
            \midrule[1pt]
         \multirow{5}*{\rotatebox{90}{PEMS04}}
         & 12  &0.081  &0.189    &0.067  &0.168    &\textcolor{red}{17.28\%}  &\textcolor{red}{11.11\%}    &0.087  &0.195    &0.080  &0.190    &\textcolor{red}{8.05\%}  &\textcolor{red}{2.56\%}   &0.114  &0.228    &0.076  &0.184    &\textcolor{red}{33.33\%}  &\textcolor{red}{19.30\%}   &0.082  & 0.190  &0.066  &0.178  &\textcolor{red}{19.51\%}  &\textcolor{red}{6.32\%}  \\
         & 24  &0.097  &0.207    &0.078  &0.183    &\textcolor{red}{19.59\%}  &\textcolor{red}{11.59\%}    &0.119  &0.231    &0.092  &0.205    &\textcolor{red}{22.69\%}  &\textcolor{red}{11.26\%}   &0.187  &0.297   &0.091  &0.203     &\textcolor{red}{51.34\%}  &\textcolor{red}{31.65\%}   &0.110  & 0.224  &0.083  &0.195  &\textcolor{red}{24.55\%}  &\textcolor{red}{12.95\%}  \\
         & 48  &0.128  &0.241    &0.100  &0.208    &\textcolor{red}{21.88\%}  &\textcolor{red}{13.69\%}    &0.172  &0.279    &0.113  &0.226    &\textcolor{red}{34.30\%}  &\textcolor{red}{19.00\%}  &0.319  &0.402    &0.119  &0.235    &\textcolor{red}{62.70\%}  &\textcolor{red}{41.54\%}   &0.160  & 0.276  &0.116  &0.238  &\textcolor{red}{27.50\%}  &\textcolor{red}{13.77\%}  \\
         & 96  &0.176  &0.284    &0.124  &0.236    &\textcolor{red}{29.55\%}  &\textcolor{red}{16.90\%}    &0.220  &0.323    &0.140  &0.253    &\textcolor{red}{36.36\%}  &\textcolor{red}{21.67\%}   &0.424  &0.481    &0.155  &0.269    &\textcolor{red}{63.44\%}  &\textcolor{red}{44.07\%}   &0.234  & 0.343  &0.186  &0.304  &\textcolor{red}{20.51\%}  &\textcolor{red}{11.37\%}  \\
         
         \cmidrule(r){2-26}
            &Avg  &0.121  &0.230    &0.092  &0.199    &\textcolor{red}{23.40\%}  &\textcolor{red}{13.68\%}    &0.150  &0.257    &0.106  &0.219    &\textcolor{red}{28.90\%}  &\textcolor{red}{14.98\%}  &0.261  &0.352    &0.110  &0.223    &\textcolor{red}{57.74\%}  &\textcolor{red}{36.70\%}  &0.146  &0.258  &0.113  &0.229  &\textcolor{red}{22.60\%}  &\textcolor{red}{11.24\%}  \\

         \bottomrule[2pt]
    \end{tabular}
    }\label{tab:general_tech}
    \vspace{-5mm}
\end{table*}

%% file: tables/efficiency_weather.tex
\begin{wraptable}{r}{9cm}
\vspace{-5mm}
\captionsetup{type=table}
\captionof{table}{Efficiency comparison between Paloma and other models on the Weather dataset with lookback window T = 96 and forecast horizon $\tau$ = 720. Training Time denotes the average time required per epoch for the model.}
\vspace{-1mm}
\centering
\scalebox{0.67}{
\begin{tabular}{lcccccc}
\toprule[2pt]
Model &Parameters &MACs &Training Time &Memory Usage &MSE &MAE \\
\midrule
GTR  &150.71K   &3.21M   &9.4s   &1754MB   &0.341  &0.342  \\
CycleNet  &108.32K   &2.22M   &8.2s   &1414MB    &0.344   &0.344  \\
DLinear  &139.68K   &2.94M   &7.4s    &1420MB   &0.345   &0.381  \\
FilterNet  &369.40K   &7.19M   &12.1s    &1434MB   &0.342   &0.340 \\
iTransformer  &1.29M   &27.11M   &16.6s    &1858MB   &0.358   &0.349\\
PatchTST  &2.30M   &326.33M   &21.2s    &2130MB   &0.354   &0.340  \\
TimeXer  &1.98M   &191.15M   &32.6s    &2102MB   &0.340   &0.341  \\
\midrule
Paloma(ours)  &260.38K   &5.08M   &12.9s   &1424MB    &0.324   &0.336  \\
\bottomrule[2pt]
\end{tabular}}\label{tab:efficiency_weather}
\vspace{-3.5mm}
\end{wraptable}

%% file: 6_conclusion_v2.tex
\section{Conclusion and Limitation}
\vspace{-2mm}
In this work, we show that the residual term in cycle--residual decomposition retains intrinsic phase-dependent structures rather than behaving as unstructured noise. 
Motivated by this observation, we propose \textit{Paloma}, a simple yet effective framework inspired by phase and amplitude modulation in classical signal processing. 
By introducing a Phase Rotation module for modeling directional variations and an Amplitude Scaling module for capturing magnitude variations, Paloma explicitly models structured temporal dynamics conditioned on phase.
Beyond the proposed architecture, we further abstract the two modules into a general-purpose plugin, termed \textit{Paloma technique}, which can be readily integrated into existing models to enhance phase-aware temporal modeling. 
Extensive experiments demonstrate the effectiveness, robustness, and generalizability of the proposed approach.

Nevertheless, the proposed method has several limitations. 
First, its performance depends on the availability of a reasonably accurate prior for the cycle length $W$. 
As shown in the sensitivity analysis in Appendix~\ref{appendix_para_cycle_length}, Paloma is sensitive to periodicity misalignment. 
Interestingly, this sensitivity may also serve as a diagnostic signal for estimating the underlying periodicity of a dataset, since forecasting performance tends to deteriorate when the assumed cycle length becomes closer to, yet remains misaligned with, the intrinsic period (see Appendix~\ref{appendix_para_cycle_length} for detailed analysis).
Second, the current formulation is primarily designed for periodic or quasi-periodic signals, and extending the framework to more general non-periodic scenarios remains an important direction for future work.


\section*{Acknowledgments}
This work was supported in part by the National Natural Science Foundation of China (62272048); and in part by the National Key Research and Development Program of China (2019YFB1406302).

%% file: 7_appendix.tex
\section{More Explanations about Motivation}\label{appendix_motivation}

\begin{figure}[h]
    \centering
    \includegraphics[width=1\linewidth]{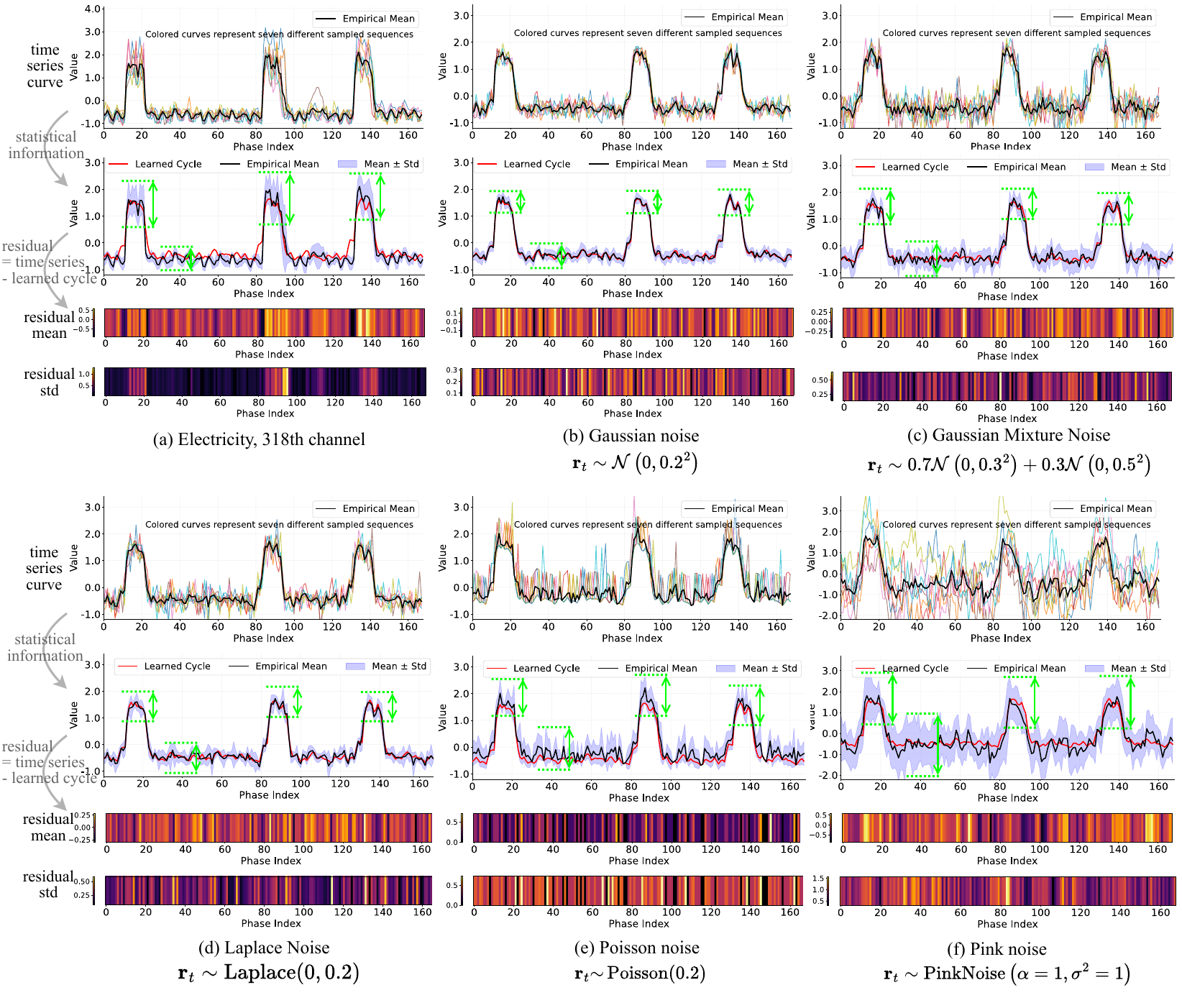}
    \vspace{-6mm}
    \caption{Comparison between real-world residuals and synthetic noise-based counterparts under different noise types. 
    Subfigure (a) shows the real-world case from the Electricity dataset (318th channel).
    Subfigures (b)–(f) present synthetic sequences constructed by combining the learned cycle in (a) with different noise types, including Gaussian, Gaussian mixture, Laplace, Poisson, and pink noise. 
    The meaning of each row is consistent with that in Figure~\ref{fig:motivation}.
    Unlike real-world data, where residual statistics exhibit clear phase-dependent patterns, all noise-based counterparts fail to produce consistent phase-aligned structure, regardless of distributional complexity or variability characteristics.}
    \label{fig:motivatiom_appendix}
    \vspace{-2mm}
\end{figure}

To further validate the observations in Figure~\ref{fig:motivation}(c), we construct additional synthetic examples by combining the learned cycle from the Electricity dataset (318th channel) with various types of noise. The corresponding visualizations are shown in Figure~\ref{fig:motivatiom_appendix}.
Specifically, we consider a diverse set of noise distributions, including Gaussian noise, Gaussian mixture noise, Laplace noise, Poisson noise, and Pink noise. For each case, the synthetic time series is generated as
$\mathbf{x}_t = \mathbf{Q}_{i_t} + \mathbf{r}_t$,
where $\mathbf{Q}_{i_t}$ denotes the learned cycle and $\mathbf{r}_t$ is sampled from the corresponding noise distribution.
The results reveal several important observations:

(1) Independently sampled noise does not induce phase-dependent variance patterns.
For noise types such as Gaussian (Figure~\ref{fig:motivatiom_appendix}(b)) and Laplace (Figure~\ref{fig:motivatiom_appendix}(d)) distributions, the residual standard deviation remains approximately homogeneous across phase indices. The residual mean is close to zero mainly by construction, as the sampled noise is zero-mean. Therefore, the key observation is not the mean behavior, but the absence of systematic phase-dependent variation in the residual standard deviation. This suggests that independently sampled noise alone cannot reproduce the phase-dependent residual variability observed in real-world data.

(2) More flexible noise distributions still fail to produce phase-aligned structure.
For more flexible noise models, such as Gaussian mixture noise (Figure~\ref{fig:motivatiom_appendix}(c)) and Poisson noise (Figure~\ref{fig:motivatiom_appendix}(e)), the residual variance may exhibit increased local variability due to the underlying distributional properties. However, these variations remain unstructured with respect to the phase index, indicating that distributional complexity alone is insufficient to induce phase-dependent patterns.
A similar observation holds for colored noise (e.g., pink noise in Figure~\ref{fig:motivatiom_appendix}(f)), where the residual exhibits smoother and more continuous fluctuations compared to independently sampled noise. Nevertheless, these variations still do not exhibit consistent dependence on the phase index, and thus fail to form structured, phase-dependent patterns.
Overall, these results demonstrate that diverse noise characteristics, including heterogeneous distributions and temporally smoother fluctuations, are insufficient to reproduce the stable, phase-dependent variance structures observed in real-world data.

Taken together, these results demonstrate that the structured residual deviations observed in real-world time series (Figure~\ref{fig:motivatiom_appendix}(a)) cannot be simply explained as noise. Instead, they reflect intrinsic, phase-dependent statistical properties of the underlying data distribution. This provides strong empirical motivation for explicitly modeling phase-dependent residual dynamics, as proposed in our method.

\begin{figure}[h]
    \centering
    \includegraphics[width=1\linewidth]{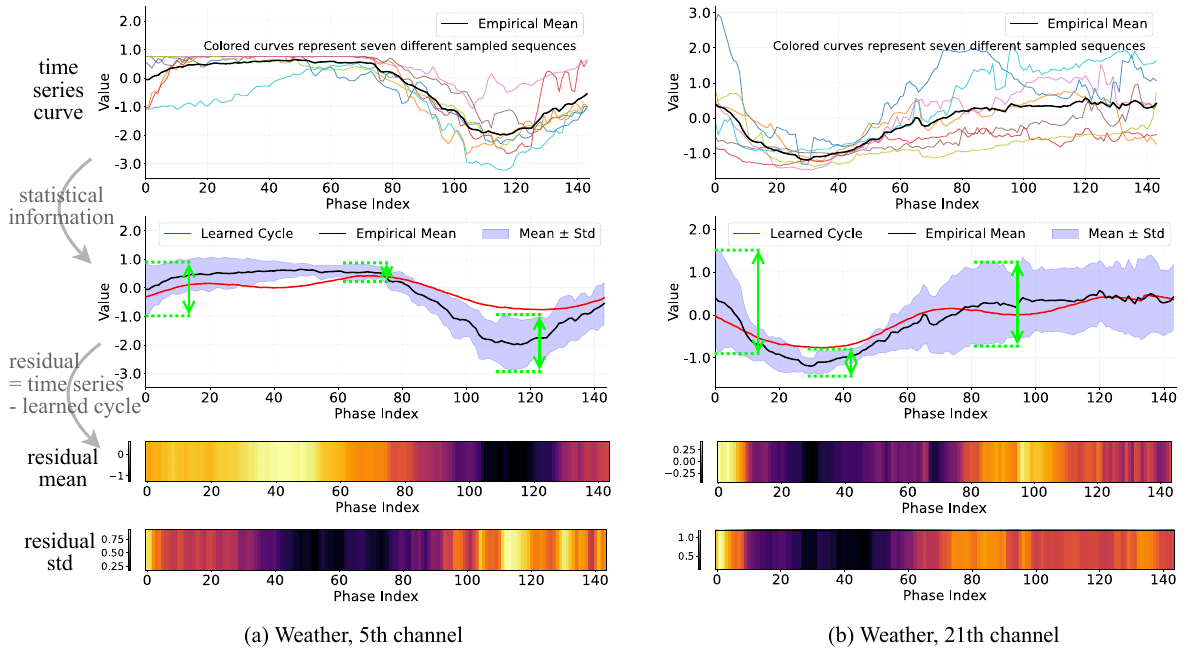}
    \vspace{-2.5mm}
    \caption{Visualization of structured residual deviations on the Weather dataset. The row definitions are consistent with those in Figure~\ref{fig:motivation}.}
    \label{fig:motivatiom_weather}
\end{figure}

Figure~\ref{fig:motivatiom_weather} provides additional visualizations of residual deviations on the Weather dataset. 
Consistent with the observations in Figure~\ref{fig:motivation}, the residual statistics exhibit clear phase-dependent patterns across different channels. 
In particular, both the residual mean and residual standard deviation vary systematically along the phase index, indicating that the residual cannot be regarded as unstructured noise after cycle removal. 
These results further support our motivation that residual dynamics retain structured phase-dependent variations beyond the learned periodic component.

\section{Theoretical Supplement}\label{appendix_Theoretical_Supplement}
\subsection{Proof of Proposition~\ref{proposition_1}}

We begin by restating the setting used in Proposition~\ref{proposition_1}.
Let $\mathbf{X} = \{\mathbf{x}_l\}_{l=0}^{T-1} \in \mathbb{R}^{T \times N}$ be a sequence with phase index sequence $\mathbf{i} = \{i_l\}_{l=0}^{T-1}$, where $i_l \in \{0,1,\dots,W-1\}$.
We consider a simplified prediction setting in which the prediction depends solely on the learnable prototype matrix $\mathbf{Q} \in \mathbb{R}^{W \times N}$, i.e., $\hat{\mathbf{x}}_l = \mathbf{Q}_{i_l}$, where $\hat{\mathbf{x}}_l \in \mathbb{R}^N$ and $\mathbf{x}_l \in \mathbb{R}^N$ denote the predicted and corresponding ground-truth values at time step $l$, respectively.

\begin{appendixPro}[\textbf{Optimal Phase Prototype under MSE}]
\label{proposition_1}
Under the above setting, consider the widely used mean squared error (MSE) objective
$
\mathcal{L}(\mathbf{Q}) = \mathbb{E}\left[\left\|\mathbf{x}_l - \mathbf{Q}_{i_l}\right\|_2^2\right].
$
Then the optimal solution satisfies
\begin{equation}
\mathbf{Q}_k^* = \mathbb{E}\left[\mathbf{x}_l \mid i_l = k\right], 
\quad \forall k \in \{0,1,\dots,W-1\}.
\end{equation}
\end{appendixPro}

\begin{proof}
Assume that $\mathbb{E}\|\mathbf{x}_l\|_2^2 < \infty$.
By the law of total expectation, the objective can be decomposed according to the phase index:
\begin{equation}
\mathcal{L}(\mathbf{Q})
= \mathbb{E}\left[\left\|\mathbf{x}_l - \mathbf{Q}_{i_l}\right\|_2^2\right]
= \sum_{k=0}^{W-1} \mathbb{P}(i_l = k)\,
\mathbb{E}\left[\left\|\mathbf{x}_l - \mathbf{Q}_{i_l}\right\|_2^2 \mid i_l = k\right].
\end{equation}
Conditioned on the event $i_l = k$, we have $\mathbf{Q}_{i_l} = \mathbf{Q}_k$, and hence
\begin{equation}
\mathcal{L}(\mathbf{Q})
= \sum_{k=0}^{W-1} \mathbb{P}(i_l = k)\,
\mathbb{E}\left[\left\|\mathbf{x}_l - \mathbf{Q}_k\right\|_2^2 \mid i_l = k\right].
\end{equation}

Therefore, for any $k$ such that $\mathbb{P}(i_l = k) > 0$, minimizing $\mathcal{L}(\mathbf{Q})$
with respect to $\mathbf{Q}_k$ is equivalent to minimizing the phase-wise term
\begin{equation}
\mathcal{L}_k(\mathbf{Q}_k)
:= \mathbb{E}\left[\left\|\mathbf{x}_l - \mathbf{Q}_k\right\|_2^2 \mid i_l = k\right].
\end{equation}

For a fixed $k \in \{0,1,\dots,W-1\}$, let
\begin{equation}
\boldsymbol{\mu}_k := \mathbb{E}\left[\mathbf{x}_l \mid i_l = k\right].
\end{equation}
Then
\begin{equation}
\mathbf{x}_l - \mathbf{Q}_k
= (\mathbf{x}_l - \boldsymbol{\mu}_k) + (\boldsymbol{\mu}_k - \mathbf{Q}_k).
\end{equation}
Substituting this into $\mathcal{L}_k(\mathbf{Q}_k)$ yields
\begin{align}
\mathcal{L}_k(\mathbf{Q}_k)
&= \mathbb{E}\left[\left\|(\mathbf{x}_l - \boldsymbol{\mu}_k) + (\boldsymbol{\mu}_k - \mathbf{Q}_k)\right\|_2^2 \mid i_l = k\right] \\
&= \mathbb{E}\left[\left\|\mathbf{x}_l - \boldsymbol{\mu}_k\right\|_2^2 \mid i_l = k\right]
+ 2\,\mathbb{E}\left[(\mathbf{x}_l - \boldsymbol{\mu}_k)^\top(\boldsymbol{\mu}_k - \mathbf{Q}_k)\mid i_l = k\right]
+ \left\|\boldsymbol{\mu}_k - \mathbf{Q}_k\right\|_2^2.
\end{align}

Since $\boldsymbol{\mu}_k - \mathbf{Q}_k$ is deterministic under the conditional distribution given $i_l = k$, the cross term becomes
\begin{equation}
\mathbb{E}\left[(\mathbf{x}_l - \boldsymbol{\mu}_k)^\top(\boldsymbol{\mu}_k - \mathbf{Q}_k)\mid i_l = k\right]
=
\mathbb{E}\left[\mathbf{x}_l - \boldsymbol{\mu}_k \mid i_l = k\right]^\top
(\boldsymbol{\mu}_k - \mathbf{Q}_k).
\end{equation}
By the definition of $\boldsymbol{\mu}_k$,
\begin{equation}
\mathbb{E}\left[\mathbf{x}_l - \boldsymbol{\mu}_k \mid i_l = k\right]
=
\mathbb{E}\left[\mathbf{x}_l \mid i_l = k\right] - \boldsymbol{\mu}_k
= \mathbf{0}.
\end{equation}
Hence the cross term vanishes, and we obtain
\begin{equation}
\mathcal{L}_k(\mathbf{Q}_k)
=
\mathbb{E}\left[\left\|\mathbf{x}_l - \boldsymbol{\mu}_k\right\|_2^2 \mid i_l = k\right]
+
\left\|\boldsymbol{\mu}_k - \mathbf{Q}_k\right\|_2^2.
\end{equation}

The first term is independent of $\mathbf{Q}_k$, while the second term is nonnegative and is minimized if and only if
\begin{equation}
\mathbf{Q}_k = \boldsymbol{\mu}_k
= \mathbb{E}\left[\mathbf{x}_l \mid i_l = k\right].
\end{equation}
Therefore, for every $k \in \{0,1,\dots,W-1\}$ such that $\mathbb{P}(i_l = k) > 0$, the optimal solution satisfies
\begin{equation}
\mathbf{Q}_k^* = \mathbb{E}\left[\mathbf{x}_l \mid i_l = k\right].
\end{equation}
This completes the proof.
\end{proof}

\subsection{Proof of Proposition~\ref{proposition_2}}
\begin{appendixPro}[\textbf{Exact Preservation of Phase-Conditional Second-Order Structure}]
\label{proposition_2}
Let $\mathbf{Q}^*$ be the optimal solution given in Proposition~\ref{proposition_1}, namely,
$\mathbf{Q}_k^* = \mathbb{E}[\mathbf{x}_l \mid i_l = k]$ for all $k \in \{0,1,\dots,W-1\}$ such that $\mathbb{P}(i_l = k) > 0$.
Define the residual as $\mathbf{r}_l = \mathbf{x}_l - \mathbf{Q}^*_{i_l}$.
Assume that the conditional second moments of $\mathbf{x}_l$ given $i_l$ exist.
Then, for any $k \in \{0,1,\dots,W-1\}$ such that $\mathbb{P}(i_l = k) > 0$, the residual satisfies
\begin{equation}
\mathbb{E}[\mathbf{r}_l \mid i_l = k] = \mathbf{0},
\end{equation}
and
\begin{equation}
\mathrm{Cov}(\mathbf{r}_l \mid i_l = k)
= \mathrm{Cov}(\mathbf{x}_l \mid i_l = k).
\end{equation}
\end{appendixPro}

\begin{proof}
Fix any $k \in \{0,1,\dots,W-1\}$ such that $\mathbb{P}(i_l = k) > 0$.
Define
\begin{equation}
\boldsymbol{\mu}_k := \mathbb{E}[\mathbf{x}_l \mid i_l = k].
\end{equation}
By Proposition~\ref{proposition_1}, we have
\begin{equation}
\mathbf{Q}_k^* = \boldsymbol{\mu}_k.
\end{equation}
Conditioned on $i_l = k$, the residual becomes
\begin{equation}
\mathbf{r}_l
= \mathbf{x}_l - \mathbf{Q}_{i_l}^*
= \mathbf{x}_l - \mathbf{Q}_k^*
= \mathbf{x}_l - \boldsymbol{\mu}_k.
\end{equation}

We first compute its conditional mean:
\begin{equation}
\mathbb{E}[\mathbf{r}_l \mid i_l = k]
= \mathbb{E}[\mathbf{x}_l - \boldsymbol{\mu}_k \mid i_l = k]
= \mathbb{E}[\mathbf{x}_l \mid i_l = k] - \boldsymbol{\mu}_k
= \boldsymbol{\mu}_k - \boldsymbol{\mu}_k
= \mathbf{0}.
\end{equation}

Next, by the definition of conditional covariance,
\begin{equation}
\mathrm{Cov}(\mathbf{r}_l \mid i_l = k)
=
\mathbb{E}\!\left[
(\mathbf{r}_l - \mathbb{E}[\mathbf{r}_l \mid i_l = k])
(\mathbf{r}_l - \mathbb{E}[\mathbf{r}_l \mid i_l = k])^\top
\mid i_l = k
\right].
\end{equation}
Since $\mathbb{E}[\mathbf{r}_l \mid i_l = k] = \mathbf{0}$, it follows that
\begin{equation}
\mathrm{Cov}(\mathbf{r}_l \mid i_l = k)
=
\mathbb{E}[\mathbf{r}_l \mathbf{r}_l^\top \mid i_l = k].
\end{equation}
Substituting $\mathbf{r}_l = \mathbf{x}_l - \boldsymbol{\mu}_k$, we obtain
\begin{equation}
\mathrm{Cov}(\mathbf{r}_l \mid i_l = k)
=
\mathbb{E}\!\left[
(\mathbf{x}_l - \boldsymbol{\mu}_k)(\mathbf{x}_l - \boldsymbol{\mu}_k)^\top
\mid i_l = k
\right].
\end{equation}
By the definition of conditional covariance and $\boldsymbol{\mu}_k = \mathbb{E}[\mathbf{x}_l \mid i_l = k]$, the right-hand side is exactly $\mathrm{Cov}(\mathbf{x}_l \mid i_l = k)$.
Therefore,
\begin{equation}
\mathrm{Cov}(\mathbf{r}_l \mid i_l = k)
=
\mathrm{Cov}(\mathbf{x}_l \mid i_l = k).
\end{equation}
This completes the proof.
\end{proof}

\subsection{Proposition~\ref{proposition_3} (Optimal Phase Prototype under MAE)}\label{appendix_proposition_MAE}
\begin{appendixPro}[\textbf{Optimal Phase Prototype under MAE}]
\label{proposition_3}
Under the same simplified prediction setting, i.e., $\hat{\mathbf{x}}_l = \mathbf{Q}_{i_l}$, consider the mean absolute error (MAE) objective
$
\mathcal{L}_{\mathrm{MAE}}(\mathbf{Q})
=
\mathbb{E}\left[\left\|\mathbf{x}_l - \mathbf{Q}_{i_l}\right\|_1\right].
$
Assume that the conditional first absolute moments exist. Then, for any 
$k \in \{0,1,\dots,W-1\}$ such that $\mathbb{P}(i_l=k)>0$, the optimal solution satisfies
\begin{equation}
\mathbf{Q}_{k,n}^*
\in 
\operatorname{Median}\left(x_{l,n} \mid i_l = k\right),
\quad
\forall n \in \{1,\dots,N\},
\end{equation}
where $x_{l,n}$ and $\mathbf{Q}_{k,n}$ denote the $n$-th channel of 
$\mathbf{x}_l$ and $\mathbf{Q}_k$, respectively.
\end{appendixPro}

\begin{proof}
By the law of total expectation, the MAE objective can be decomposed according to the phase index:
\begin{equation}
\mathcal{L}_{\mathrm{MAE}}(\mathbf{Q})
=
\sum_{k=0}^{W-1}
\mathbb{P}(i_l = k)
\mathbb{E}\left[
\left\|\mathbf{x}_l - \mathbf{Q}_{i_l}\right\|_1
\mid i_l = k
\right].
\end{equation}
Conditioned on $i_l=k$, we have $\mathbf{Q}_{i_l}=\mathbf{Q}_k$, and hence
\begin{equation}
\mathcal{L}_{\mathrm{MAE}}(\mathbf{Q})
=
\sum_{k=0}^{W-1}
\mathbb{P}(i_l = k)
\mathbb{E}\left[
\left\|\mathbf{x}_l - \mathbf{Q}_k\right\|_1
\mid i_l = k
\right].
\end{equation}

For any $k$ such that $\mathbb{P}(i_l=k)>0$, minimizing 
$\mathcal{L}_{\mathrm{MAE}}(\mathbf{Q})$ with respect to $\mathbf{Q}_k$ is equivalent to minimizing
\begin{equation}
\mathcal{L}_{k}^{\mathrm{MAE}}(\mathbf{Q}_k)
:=
\mathbb{E}\left[
\left\|\mathbf{x}_l - \mathbf{Q}_k\right\|_1
\mid i_l = k
\right].
\end{equation}
Since the $\ell_1$ norm is separable across variables, we have
\begin{equation}
\mathcal{L}_{k}^{\mathrm{MAE}}(\mathbf{Q}_k)
=
\sum_{n=1}^{N}
\mathbb{E}\left[
\left|x_{l,n} - \mathbf{Q}_{k,n}\right|
\mid i_l = k
\right].
\end{equation}
Thus, each variable can be optimized independently.

For a fixed phase $k$ and variable $n$, let $q=\mathbf{Q}_{k,n}$ and define
\begin{equation}
\phi(q)
=
\mathbb{E}\left[
\left|x_{l,n} - q\right|
\mid i_l = k
\right].
\end{equation}
A standard property of the absolute loss is that any minimizer of the expected absolute deviation is a median of the underlying distribution~\cite{hastie2009elements}. Therefore,
\begin{equation}
q^* \in \operatorname{Median}(x_{l,n}\mid i_l=k).
\end{equation}
Applying this result to every variable $n$ yields
\begin{equation}
\mathbf{Q}_{k,n}^*
\in
\operatorname{Median}(x_{l,n}\mid i_l=k),
\quad
\forall n \in \{1,\dots,N\}.
\end{equation}
This completes the proof.
\end{proof}

Proposition~\ref{proposition_3} shows that replacing the MSE objective with the MAE objective changes the optimal prototype from the phase-conditioned mean to the phase-conditioned median. 
However, phase-dependent deviations typically persist in the residual, since subtracting a phase-dependent prototype (such as the mean or median) primarily shifts the center of the distribution, without eliminating differences in variability or overall shape across phases. 
Consequently, the residual continues to reflect the underlying phase-dependent deviations.

\section{More Related Work}\label{appendix_related_work}
\paragraph{Deep Learning-based Time Series Forecasting}
Deep learning approaches have emerged as powerful tools for time series forecasting, driven by their ability to model complex nonlinear temporal dependencies~\cite{lim2021time}.
These approaches can be broadly classified into RNN-based, TCN-based, GNN-based, Transformer-based, and Linear-based architectures.
RNN-based models~\cite{lai2018modeling,lin2025segrnn} capture temporal dependencies by recursively updating hidden states, where each state encodes historical information up to the current time step. 
While effective for sequential modeling, they often suffer from difficulties in capturing long-term dependencies and limited parallelization.
TCN-based models~\cite{wu2022timesnet,luo2024moderntcn} capture temporal dependencies using causal and dilated convolutions, providing large receptive fields while respecting temporal causality. They allow efficient parallel computation and stable training, but may struggle with very long-range dependencies and do not explicitly model temporal dynamics in a step-by-step manner.
GNN-based models~\cite{huang2023crossgnn,yi2023fouriergnn} represent time series as nodes in a graph and capture dependencies by propagating information through graph structures, enabling the modeling of complex spatial–temporal relationships. They are effective for capturing structured and long-range interactions, but often rely on predefined or learned graph structures and may incur higher computational cost and sensitivity to graph quality.
Transformer-based models~\cite{liu2023itransformer,nie2022time} capture temporal dependencies using attention mechanisms, such as self-attention~\cite{vaswani2017attention} and cross-attention~\cite{zhang2023crossformer}, enabling flexible modeling of interactions across embedding tokens.
They typically exhibit strong expressiveness, but incur high computational cost as the lookback window length increases.
Linear-based methods~\cite{zeng2023transformers, li2023revisiting,wang2024timemixer, yi2023frequency} model time series using simple linear mappings, typically capturing temporal dependencies via linear projections of historical observations.
They are computationally efficient and easy to train, often achieving strong performance in practice despite their simplicity.
Recent studies have increasingly emphasized Transformer-based and linear-based approaches, both of which achieve state-of-the-art performance while exhibiting distinct characteristics. Transformer-based models are highly expressive, whereas linear-based methods provide superior computational efficiency.
Paloma, proposed in this work, is a linear-based model that effectively captures phase-dependent residual deviations while maintaining the simplicity and parameter efficiency inherent to linear architectures.



\section{Experimental Details}
\subsection{Baselines}\label{appendix_baselines}

To comprehensively evaluate the effectiveness of the proposed Paloma for time series forecasting, we consider 17 representative and state-of-the-art methods as baselines. These methods span diverse architectural paradigms, including MLP-based, Transformer-based, frequency-based, and TCN-based approaches, along with two recently proposed training frameworks. A brief overview of these methods is provided as follows:

\textbf{Global Temporal Retriever (GTR)}~\cite{cao2026enhancing} is a lightweight, plug-and-play module that addresses the limited historical context of existing multivariate time series forecasting models by enabling access to global temporal information across entire cycles. It adaptively retrieves and integrates relevant global patterns with local inputs, thereby enhancing forecasting performance without modifying the underlying model architecture.
The detailed implementation is available at \url{https://github.com/macovaseas/GTR}.
The performance of GTR reported in Table~\ref{tab:input_96_avg} (Table~\ref{tab:input_96_appendix}) and Table~\ref{tab:input_336_appendix} follows the results presented in Table 7 and Table 11 of the original GTR paper.

\textbf{RAFT}~\cite{han2025retrieval} is a retrieval-augmented time series forecasting method that enhances model capacity by retrieving historically similar patterns from the training data and incorporating their future values into prediction. By providing external pattern-based information as an inductive bias, it reduces the burden on model parameters and consistently improves forecasting performance.
The detailed implementation is available at \url{https://github.com/archon159/RAFT}.
The performance of RAFT reported in Table~\ref{tab:input_96_avg} (Table~\ref{tab:input_96_appendix}) follows the results presented in Table 7 of the GTR~\cite{cao2026enhancing} paper.

\textbf{CycleNet}~\cite{lin2024cyclenet} is a simple yet efficient time series forecasting method that models periodic patterns via Residual Cycle Forecasting (RCF) and predicts residuals using lightweight backbones, achieving state-of-the-art performance while remaining flexible for integration with existing models.
The detailed implementation is available at \url{https://github.com/ACAT-SCUT/CycleNet}.
The performance of CycleNet reported in Table~\ref{tab:input_96_avg} (Table~\ref{tab:input_96_appendix}) and Table~\ref{tab:input_336_appendix} follows the results presented in Table 7 and Table 11 of the GTR~\cite{cao2026enhancing} paper.

\textbf{SparseTSF}~\cite{lin2024sparsetsf} is an extremely lightweight long-term time series forecasting model that leverages cross-period sparse forecasting to efficiently extract periodic patterns, achieving competitive performance with fewer than 1k parameters while maintaining strong generalization in resource-constrained scenarios.
The detailed implementation is available at \url{https://github.com/lss-1138/SparseTSF}.
The performance of SparseTSF reported in Table~\ref{tab:input_336_appendix} follows the results presented in Table 11 of the GTR~\cite{cao2026enhancing} paper.

\textbf{DLinear}~\cite{zeng2023transformers} decomposes the input series into trend and seasonal components and applies separate linear layers to each, thereby improving performance by explicitly modeling trend dynamics.
The detailed implementation is available at \url{https://github.com/cure-lab/LTSF-Linear}.
The performance of DLinear reported in Table~\ref{tab:input_96_avg} (Table~\ref{tab:input_96_appendix}) follows the results presented in Table 7 of the GTR~\cite{cao2026enhancing} paper.

\textbf{OLinear}~\cite{yue2025olinear} is a linear-based forecasting model that leverages an orthogonal transformation (OrthoTrans) to decorrelate temporal dependencies and a normalized linear layer (NormLin) to capture multivariate relationships, achieving state-of-the-art performance while also enhancing existing models as a plug-in.
The detailed implementation is available at \url{https://github.com/jackyue1994/OLinear}.
The performance of OLinear reported in Table~\ref{tab:input_96_additional} follows the results presented in Table 15 of the original OLinear paper.

\textbf{TimeMixer}~\cite{wang2024timemixer} is a fully MLP-based forecasting model that captures multiscale temporal variations through decomposable mixing of seasonal and trend components and multi-predictor fusion, enabling effective modeling of complex time series with high efficiency.
The detailed implementation is available at \url{https://github.com/kwuking/TimeMixer}.
The performance of TimeMixer reported in Table~\ref{tab:input_96_additional} follows the results presented in Table 7 of the GTR~\cite{cao2026enhancing} paper.

\textbf{TimeXer}~\cite{wang2024timexer} is a Transformer-based forecasting model that incorporates exogenous variables through patch-wise self-attention, variate-wise cross-attention, and global tokens to effectively model both temporal dependencies and cross-variable relationships for improved prediction of endogenous series.
The detailed implementation is available at \url{https://github.com/thuml/TimeXer}.
The performance of TimeXer reported in Table~\ref{tab:input_96_avg} (Table~\ref{tab:input_96_appendix}) follows the results presented in Table 7 of the GTR~\cite{cao2026enhancing} paper.

\textbf{iTransformer}~\cite{liu2023itransformer} is a Transformer-based forecasting model that inverts input dimensions to treat variates as tokens, enabling effective modeling of multivariate correlations without modifying the original Transformer architecture.
The detailed implementation is available at \url{https://github.com/thuml/iTransformer}.
The performance of iTransformer reported in Table~\ref{tab:input_96_avg} (Table~\ref{tab:input_96_appendix}) follows the results presented in Table 7 of the GTR~\cite{cao2026enhancing} paper.

\textbf{PatchTST}~\cite{nie2022time} employs a patch-based, channel-independent Transformer design to efficiently model local semantics and long-range dependencies in time series forecasting.
The detailed implementation is available at \url{https://github.com/yuqinie98/patchtst}.
The performance of PatchTST reported in Table~\ref{tab:input_96_avg} (Table~\ref{tab:input_96_appendix}) and Table~\ref{tab:input_336_appendix} follows the results presented in Table 7 and Table 11 of the GTR~\cite{cao2026enhancing} paper.

\textbf{FilterNet}~\cite{yi2024filternet} adopts learnable frequency filters to model time series by selectively passing or attenuating signal components, thereby effectively capturing informative patterns across the full frequency spectrum.
The detailed implementation is available at \url{https://github.com/aikunyi/FilterNet}.
The performance of FilterNet reported in Table~\ref{tab:input_96_avg} (Table~\ref{tab:input_96_appendix}) follows the results presented in Table 15 of the OLinear~\cite{yue2025olinear} paper.

\textbf{FreTS}~\cite{yi2023frequency} adopts frequency-domain MLPs to model time series by transforming signals into the frequency domain and learning real and imaginary components, enabling effective capture of global dependencies and compact signal patterns.
The detailed implementation is available at \url{https://github.com/aikunyi/FreTS}.
The performance of FreTS reported in Table~\ref{tab:input_96_additional} follows the results presented in Table 5 of the FreDF~\cite{wang2024fredf} paper, except for the Solar-Energy and four PEMS datasets, for which the results are obtained by running the official code five times and reporting the average values.


\textbf{FITS}~\cite{xu2023fits} operates in the complex frequency domain to perform time series analysis via interpolation, achieving competitive performance with only about 10k parameters.
The detailed implementation is available at \url{https://github.com/VEWOXIC/FITS}.
The performance of FITS reported in Table~\ref{tab:input_96_additional} follows the results presented in Table 15 of the OLinear~\cite{yue2025olinear} paper.

\textbf{Fredformer}~\cite{piao2024fredformer} addresses frequency bias in Transformer-based forecasting by promoting balanced learning across frequency bands, enabling more effective capture of both high- and low-frequency patterns in time series.
The detailed implementation is available at \url{https://github.com/chenzRG/Fredformer}.
The performance of Fredformer reported in Table~\ref{tab:input_96_additional} follows the results presented in Table 15 of the OLinear~\cite{yue2025olinear} paper.

\textbf{TimesNet}~\cite{wu2022timesnet} models time series by transforming 1D sequences into 2D representations based on multi-periodicity, enabling effective capture of intra- and inter-period temporal variations through parameter-efficient convolutional blocks.
The detailed implementation is available at \url{https://github.com/thuml/TimesNet}.
The performance of TimesNet reported in Table~\ref{tab:input_96_additional} follows the results presented in Table 5 of the FreDF~\cite{wang2024fredf} paper, except for the Solar-Energy and four PEMS datasets, for which the results are obtained by running the official code five times and reporting the average values.


\textbf{DistDF}~\cite{wang2025distdf} trains forecasting models by minimizing a tractable joint-distribution Wasserstein discrepancy that upper bounds conditional distribution differences, thereby aligning forecast and label distributions without relying on biased likelihood estimation.
The detailed implementation is available at \url{https://github.com/Master-PLC/DistDF}.
The performance of DistDF reported in Table~\ref{tab:input_96_avg} (Table~\ref{tab:input_96_appendix}) follows the results presented in Table 6 of the original DistDF~\cite{wang2025distdf} paper, except for the Solar-Energy and four PEMS datasets, for which the results are obtained by running the official code five times and reporting the average values.


\textbf{FreDF}~\cite{wang2024fredf} mitigates the bias of the Direct Forecast paradigm caused by label autocorrelation by transforming forecasting into the frequency domain, enabling more accurate and unbiased time series prediction.
The detailed implementation is available at \url{https://github.com/Master-PLC/FreDF}.
The performance of FreDF reported in Table~\ref{tab:input_96_avg} (Table~\ref{tab:input_96_appendix}) follows the results presented in Table 5 of the original FreDF~\cite{wang2024fredf} paper, except for the Solar-Energy and four PEMS datasets, for which the results are obtained by running the official code five times and reporting the average values.


\subsection{Datasets}\label{appendix_datasets}
To comprehensively evaluate the performance of our model, we conduct extensive experiments on eight benchmark datasets for long-term forecasting and the PEMS datasets for short-term forecasting. Detailed information of the datasets is provided in Table~\ref{tab:datasets}, and further descriptions are given below:

\textbf{ETT} (Electricity Transformer Temperature) dataset~\cite{zhou2021informer} records seven variables related to electricity transformers from July 2016 to July 2018. It consists of four subsets with two temporal granularities: ETTh1 and ETTh2 are sampled hourly, while ETTm1 and ETTm2 are sampled every 15 minutes.

\textbf{Weather} dataset~\cite{wu2021autoformer} records 21 meteorological variables, including air temperature and humidity, collected at 10-minute intervals throughout 2020 from the Weather Station of the Max Planck Institute for Biogeochemistry in Germany.

\textbf{Electricity} dataset~\cite{wu2021autoformer} records hourly electricity consumption (in kWh) for 321 clients, covering a multi-year period.

\textbf{Traffic} dataset~\cite{wu2021autoformer} records hourly road occupancy rates collected by 862 sensors on the freeways of the San Francisco Bay Area, covering the period from 2015 to 2016.

\textbf{Solar-Energy} dataset~\cite{zhou2021informer} records solar power production from 137 photovoltaic plants in Alabama, sampled every 10 minutes in 2006.

\textbf{PEMS} dataset~\cite{liu2022scinet} provides public traffic data collected from the Caltrans Performance Measurement System (PeMS) in California, aggregated at 5-minute intervals. In this study, we use four subsets: PEMS03, PEMS04, PEMS07, and PEMS08.

\input{tables/dataset}
\subsection{Implementation Details}\label{appendix_implementation_detials}
All experiments are implemented in PyTorch~\cite{paszke2019pytorch} and conducted on a single NVIDIA RTX 4090 GPU with 24GB memory. We use the Adam optimizer~\cite{kingma2014adam}, with the learning rate selected from $\{5\mathrm{e}{-4}, 1\mathrm{e}{-3}, 2\mathrm{e}{-3}, 3\mathrm{e}{-3}, 5\mathrm{e}{-3}\}$.
Paloma is trained for up to 30 epochs, with early stopping based on validation performance using a patience of 5.
Following prior works (e.g., CARD~\cite{xue2023card} and OLinear~\cite{yue2025olinear}), we adopt a weighted loss function. 
The hidden dimension of the linear projection, $d_{\text{model}}$, is chosen from $\{64, 128, 256, 512\}$, and the batch size is selected from $\{16, 64, 128, 256\}$.

\section{Additional Results}
\subsection{Full Forecasting Results}
\paragraph{Lookback window $T=96$.}
Table~\ref{tab:input_96_appendix} provides the complete results corresponding to Table~\ref{tab:input_96_avg} in the main text. 
In addition, Table~\ref{tab:input_96_additional} (under the same experimental setting as Table~\ref{tab:input_96_appendix}) includes results from more representative baselines, enabling a more comprehensive comparison with the proposed Paloma. 
Table~\ref{tab:input_96_additional} also reports the standard deviation of Paloma over five runs with different random seeds.
Overall, the results in Tables~\ref{tab:input_96_appendix} and Table~\ref{tab:input_96_additional} demonstrate that Paloma achieves state-of-the-art performance while maintaining strong robustness.

\input{tables/input_96_appendix}
\input{tables/input_96_additional}


\paragraph{Increasing Lookback Length.}
We further examine how model performance evolves as the lookback length increases, as shown in Table~\ref{tab:input_336_appendix} and Figure~\ref{fig:longer_lookback}. 
Paloma exhibits a generally improving and stable performance trend across datasets, indicating its ability to effectively leverage extended temporal context. 
In contrast, competing methods often show diminishing or inconsistent gains with longer input lengths, reflecting difficulty in reliably utilizing long-range historical information. 
This behavior highlights the advantage of explicitly modeling phase-dependent variations when handling longer input sequences.

\begin{figure}[t]
    \centering
    \includegraphics[width=0.90\linewidth]{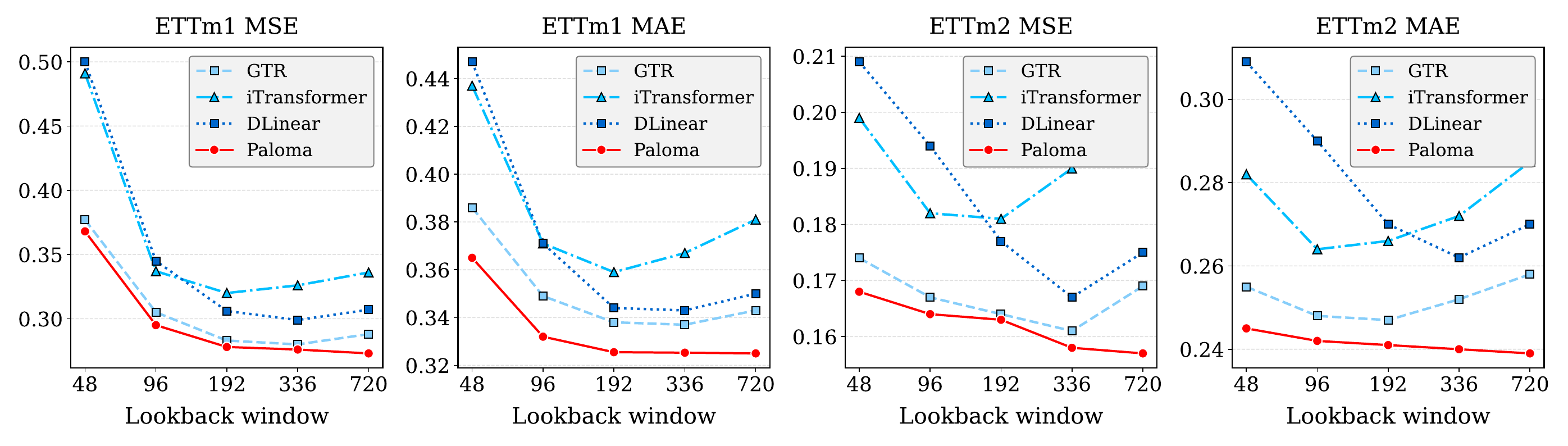}
    \caption{Forecasting performance with the lookback window $T \in\{48,96,192,336,720\}$ and fixed prediction length $\tau =96$. }
    \label{fig:longer_lookback}
\end{figure}

\input{tables/input_336_appendix}

\subsection{Full Ablation Results}
Table~\ref{tab:ablation_appendix} provides the complete results corresponding to Table~\ref{tab:ablation_avg} in the main text. 
Table~\ref{tab:ablation_appendix} presents a comprehensive ablation study of the proposed Phase Rotation (PR) and Amplitude Scaling (AS) modules under the long-term forecasting setting, including variants without PR (“w/o PR”), without AS (“w/o AS”), and without both modules (“w/o PR, AS”).
The experiments are conducted on six representative datasets, including ETTm1, ETTh1, Electricity, Weather, Traffic, and Solar-Energy.
Among them, ETTm1 and ETTh1 are selected as representative subsets from the four ETT datasets, as they cover the two temporal granularities (ETTm and ETTh).
The ablation results in Table~\ref{tab:ablation_appendix} show that removing either the PR or AS module consistently degrades performance across all datasets, indicating that both components contribute positively to the model. 
Notably, removing the PR module leads to slightly larger performance drops compared to removing the AS module, suggesting that phase modeling plays a more dominant role in capturing temporal dynamics. 
When both modules are removed, the performance further deteriorates, demonstrating that PR and AS provide complementary benefits. 
Overall, these results validate the effectiveness of the proposed phase--amplitude modulation mechanism in modeling structured residual variations.

\input{tables/ablation_appendix}

\subsection{More Efficiency Analysis}
To comprehensively evaluate the efficiency of the proposed model Paloma, we conduct three efficiency experiments on datasets with varying numbers of channels, including ETTh1 (7 channels, Table~\ref{tab:efficiency_ETTh1}), Weather (21 channels, Table~\ref{tab:efficiency_weather}), and Traffic (862 channels, Table~\ref{tab:efficiency_traffic}) datasets, covering small-, medium-, and large-scale scenarios.
All experiments are performed under a unified setting with a lookback window $T = 96$ and a prediction length $\tau = 720$. 
The batch size is fixed at 32 for ETTh1 and Weather datasets, and set to 16 for Traffic dataset due to its significantly larger scale.

The efficiency results across Tables~\ref{tab:efficiency_weather}, ~\ref{tab:efficiency_ETTh1}, and ~\ref{tab:efficiency_traffic} demonstrate that Paloma maintains a consistent balance between computational cost and predictive performance across datasets of varying scales. 
On the small-scale ETTh1 dataset (7 channels), Paloma achieves superior accuracy with comparable training time and moderate computational overhead. 
On the medium-scale Weather dataset (21 channels), Paloma continues to deliver better forecasting performance while requiring substantially fewer MACs and memory than Transformer-based models. 
On the large-scale Traffic dataset (862 channels), although the overall computational cost increases for all methods, Paloma remains significantly more efficient than Transformer-based approaches, with markedly lower MACs and training time, while incurring only a moderate increase in computational cost compared to lightweight linear models such as GTR, CycleNet, and DLinear.
These results indicate that Paloma maintains a favorable balance between efficiency and accuracy as the number of channels increases, demonstrating strong scalability to high-dimensional time series.

\input{tables/efficiency_ETTh1}

\input{tables/efficiency_traffic}

\subsection{Parameter Sensitivity Analysis}
\subsubsection{Hidden Dimension $d_{\text{model}}$}
We analyze the sensitivity of Paloma to the hidden dimension $d_{\text{model}}$, as shown in Figure~\ref{fig:hyper_para_d_model}. 
As $d_{\text{model}}$ increases from small to moderate values, the performance consistently improves, indicating that a larger representation capacity helps capture richer temporal patterns. 
However, when $d_{\text{model}}$ becomes sufficiently large, the performance tends to saturate, with only marginal gains or slight fluctuations observed.
Notably, this trend is consistent across all datasets, indicating that while performance improves with increasing $d_{\text{model}}$ at smaller scales, it becomes stable beyond a moderate range. 
This suggests that Paloma is not highly sensitive to the choice of $d_{\text{model}}$ once sufficient model capacity is reached, and can achieve strong performance without requiring excessively large dimensions.

\begin{figure}[h]
    \centering
    \subfigure[ETTm1]{
    \includegraphics[width=0.25\linewidth]{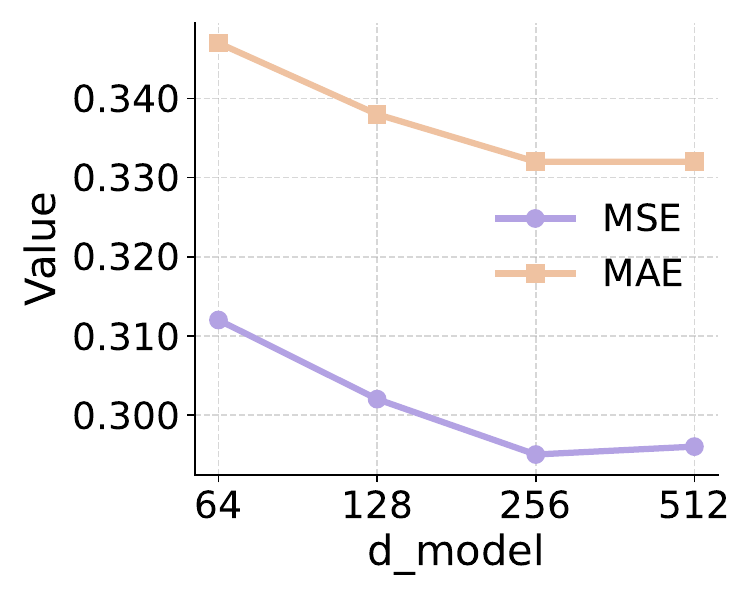}
    }
    \hspace{-4mm}
    \subfigure[ETTh1]{
    \includegraphics[width=0.25\linewidth]{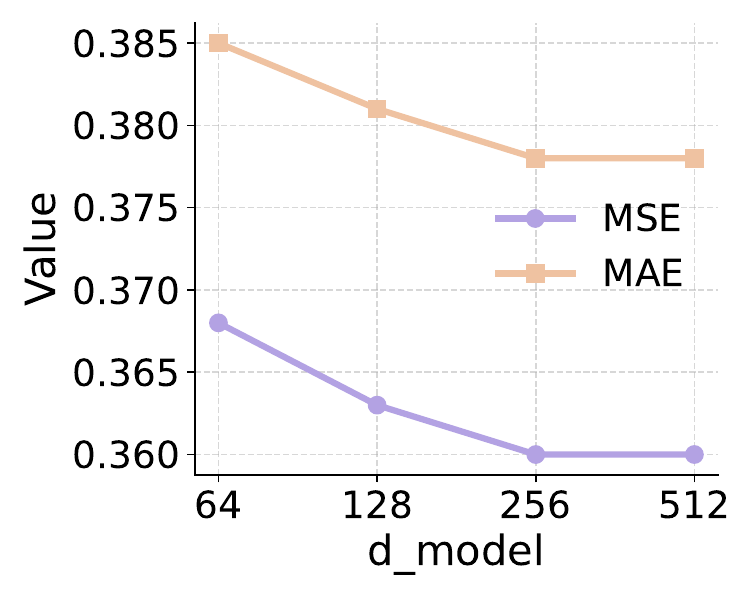}
    }
    \hspace{-4mm}
    \subfigure[Weather]{
    \includegraphics[width=0.25\linewidth]{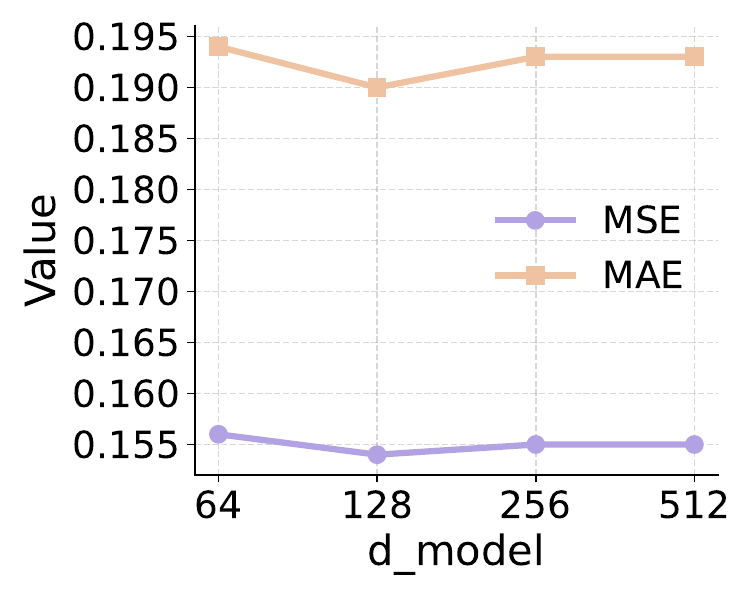}
    }
    \hspace{-4mm}
    \subfigure[Electricity]{
    \includegraphics[width=0.25\linewidth]{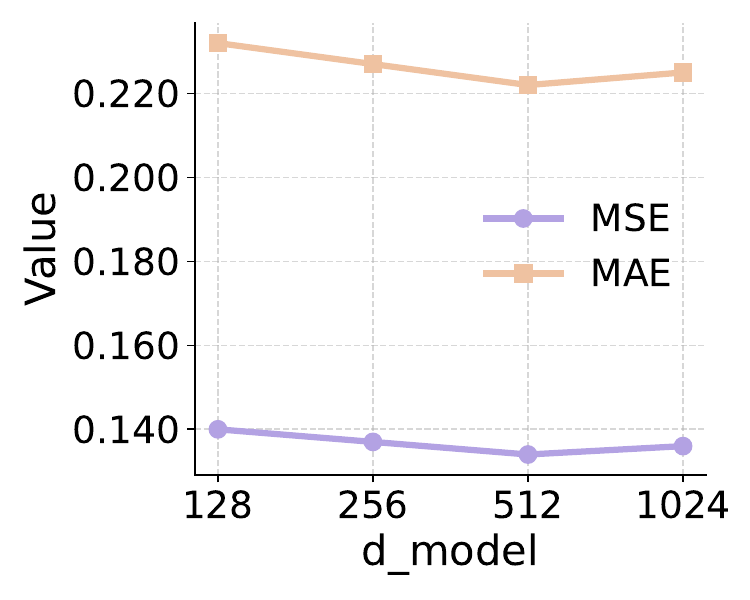}
    }
    \caption{Sensitivity analysis of the hidden dimension $d_{\mathrm{model}}$ under the setting of lookback window $T = 96$ and prediction length $\tau = 96$.}
    \label{fig:hyper_para_d_model}
\end{figure}

\subsubsection{Cycle Length $W$}\label{appendix_para_cycle_length}
We further investigate the sensitivity of Paloma to the cycle length $W$, as shown in Figure~\ref{fig:hyper_para_cycle_length}. 
From Figure~\ref{fig:hyper_para_cycle_length}, two key observations can be drawn. 
First, the best performance is consistently achieved when the cycle length $W$ matches the inherent periodicity of the dataset. 
Second, when $W$ deviates from the true cycle length, the performance degrades to varying degrees, indicating that accurate cycle specification is important for effective phase modeling. 
Notably, the degradation is more pronounced when $W$ is close to but misaligned with the true periodicity (e.g., $W=95$ for ETTm1, $W=23$ or $25$ for ETTh1, $W=143$ for Weather, and $W=167$ for Electricity), suggesting that small mismatches may introduce phase misalignment and disrupt the learned periodic structure.

These observations also reveal a limitation of Paloma, namely that its performance depends on the availability of a reasonably accurate prior for the cycle length $W$. 
When the assumed cycle length is misaligned with the true periodicity, the resulting phase inconsistency can negatively affect model performance.
From another perspective, this sensitivity can be exploited as a diagnostic signal. 
When the true periodicity is unknown, one can evaluate the model under a set of candidate cycle lengths. 
\textit{Empirically, the model tends to exhibit worse predictive performance when $W$ is closer to but not aligned with the true cycle length, which may indicate that the assumed $W$ is approaching the intrinsic periodicity of the dataset. 
This observation provides a practical heuristic for exploring the underlying periodic structure in real-world scenarios.}

\begin{figure}[h]
    \centering
    \subfigure[ETTm1 (96)]{
    \includegraphics[width=0.4\linewidth]{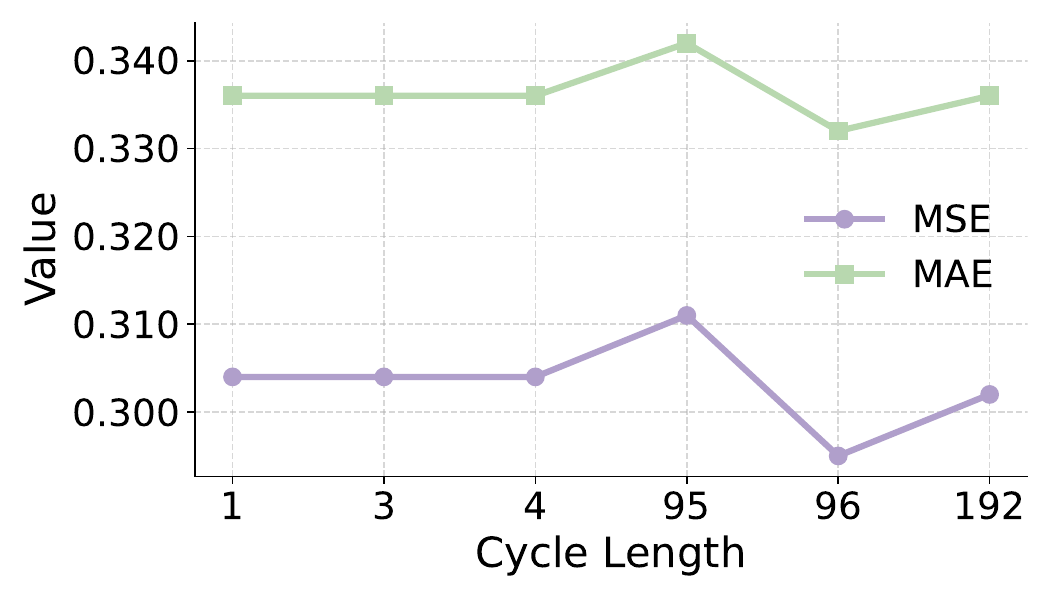}
    }
    \hspace{-2mm}
    \subfigure[ETTh1 (24)]{
    \includegraphics[width=0.4\linewidth]{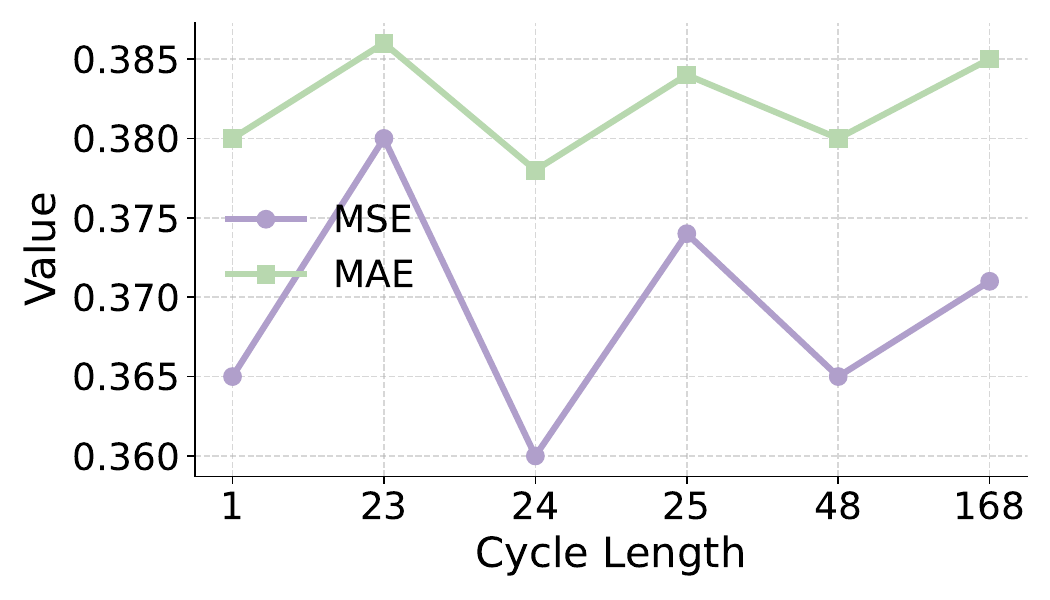}
    }\\
    \subfigure[Weather (144)]{
    \includegraphics[width=0.4\linewidth]{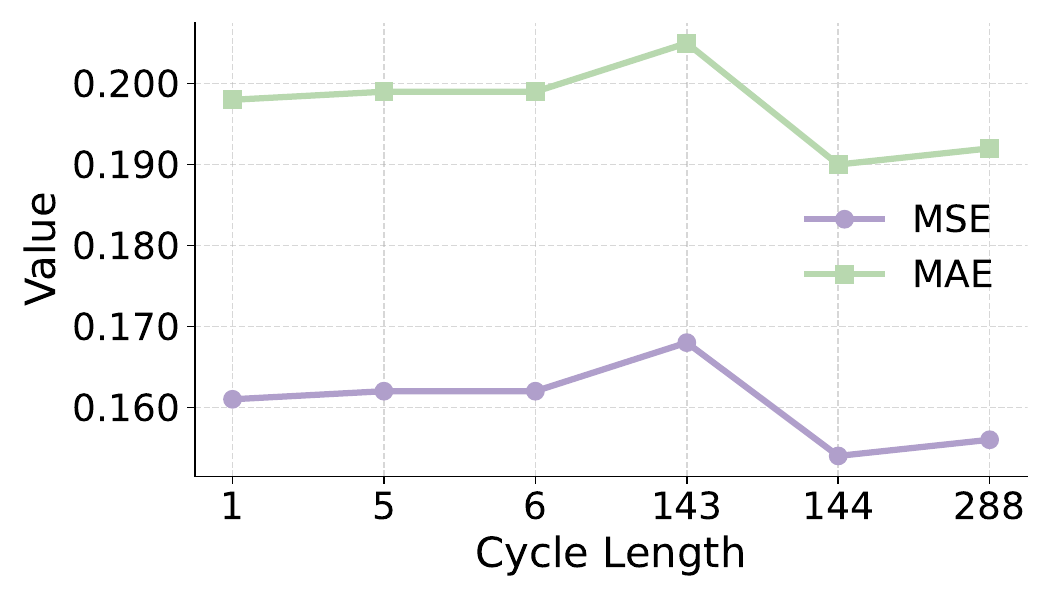}
    }
    \hspace{-2mm}
    \subfigure[Electricity (168)]{
    \includegraphics[width=0.4\linewidth]{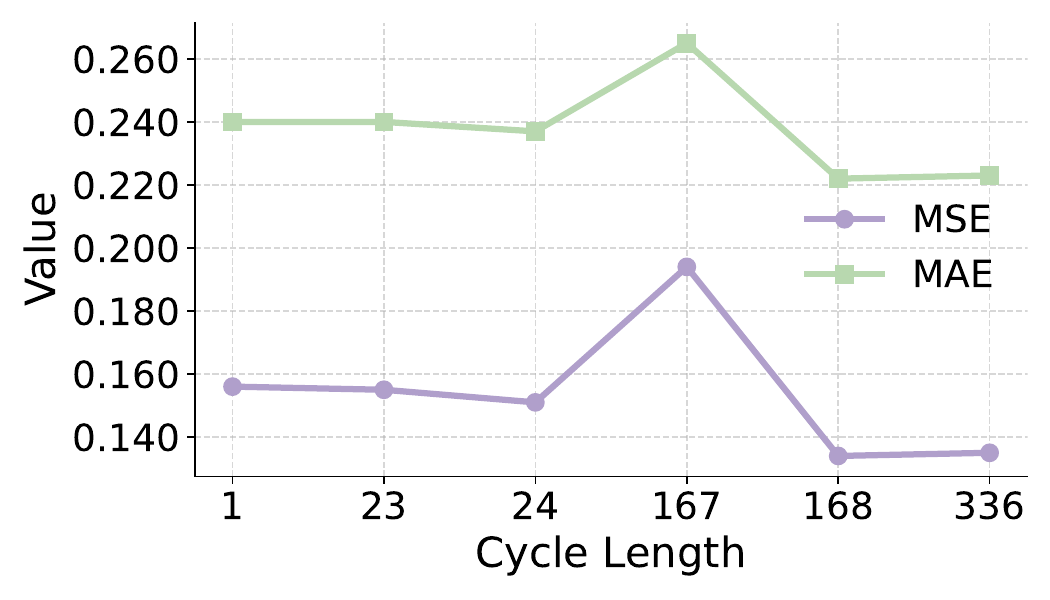}
    }
    \caption{Sensitivity analysis of the cycle length $W$ under the setting of lookback window $T = 96$ and prediction length $\tau = 96$. 
The values in parentheses indicate the inherent periodicity of each dataset.}
    \label{fig:hyper_para_cycle_length}
\end{figure}

\section{Visualizations}
To provide a more intuitive comparison with state-of-the-art methods, we present qualitative prediction results on the Solar-Energy and ETTm1 datasets, as shown in Figures~\ref{fig:forecasting_solar_96_96} and~\ref{fig:forecasting_ETTm1_96_96}. 
We select several representative models as baselines, including GTR~\cite{cao2026enhancing}, CycleNet~\cite{lin2024cyclenet}, TimeXer~\cite{wang2024timexer}, FilterNet~\cite{yi2024filternet}, and iTransformer~\cite{liu2023itransformer}, covering diverse modeling paradigms.
Compared with these approaches, Paloma produces more accurate and stable predictions of future temporal variations, demonstrating its superior ability to capture complex dynamics.

\begin{figure}[h]
    \centering
    \subfigure[Paloma]{
    \includegraphics[width=0.48\linewidth]{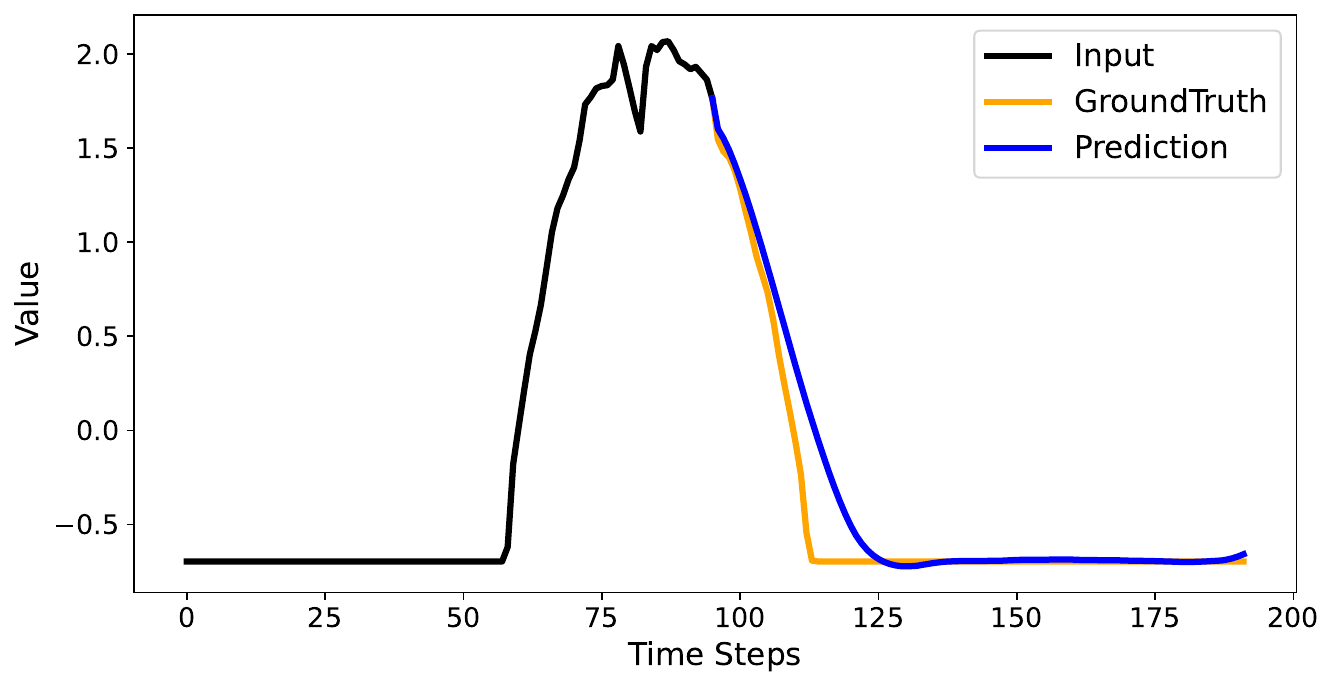}
    }
    \hspace{-4mm}
    \subfigure[GTR]{
    \includegraphics[width=0.48\linewidth]{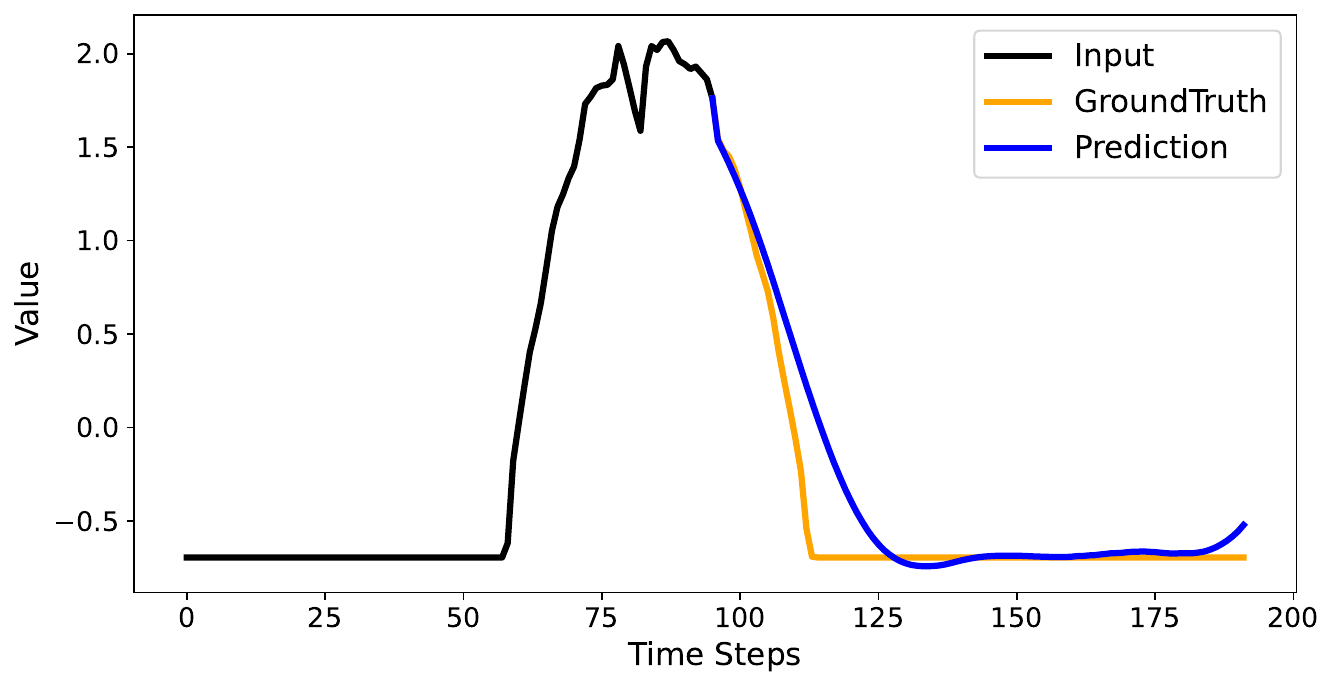}
    }\\
    \hspace{-4mm}
    \subfigure[CycleNet]{
    \includegraphics[width=0.48\linewidth]{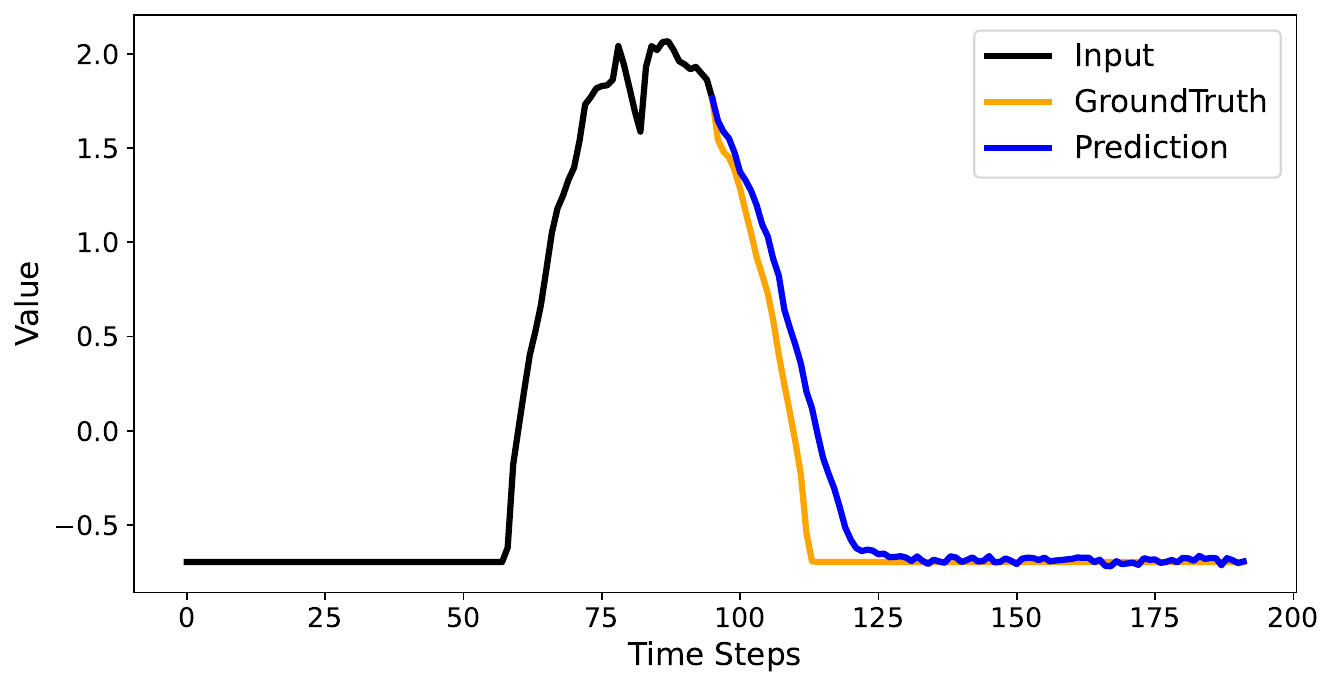}
    }
    \hspace{-4mm}
    \subfigure[TimeXer]{
    \includegraphics[width=0.48\linewidth]{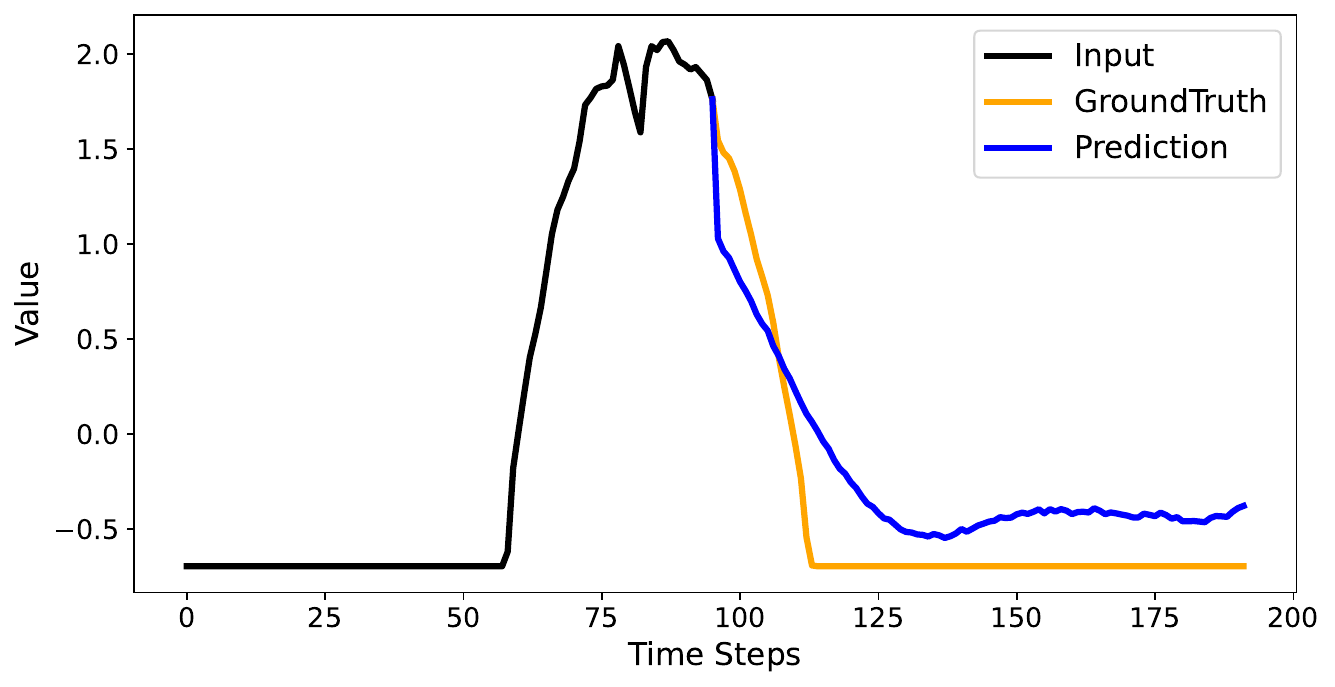}
    }\\
        \hspace{-4mm}
    \subfigure[FilterNet]{
    \includegraphics[width=0.48\linewidth]{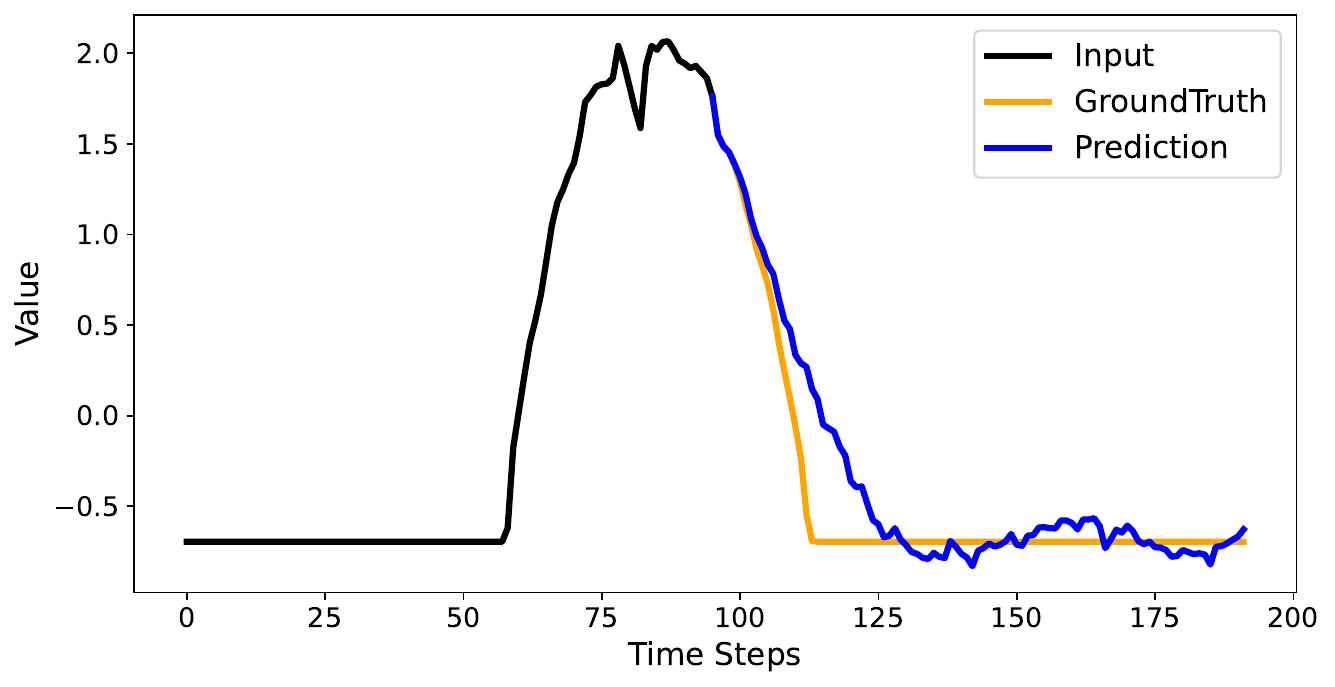}
    }
    \hspace{-4mm}
    \subfigure[iTransformer]{
    \includegraphics[width=0.48\linewidth]{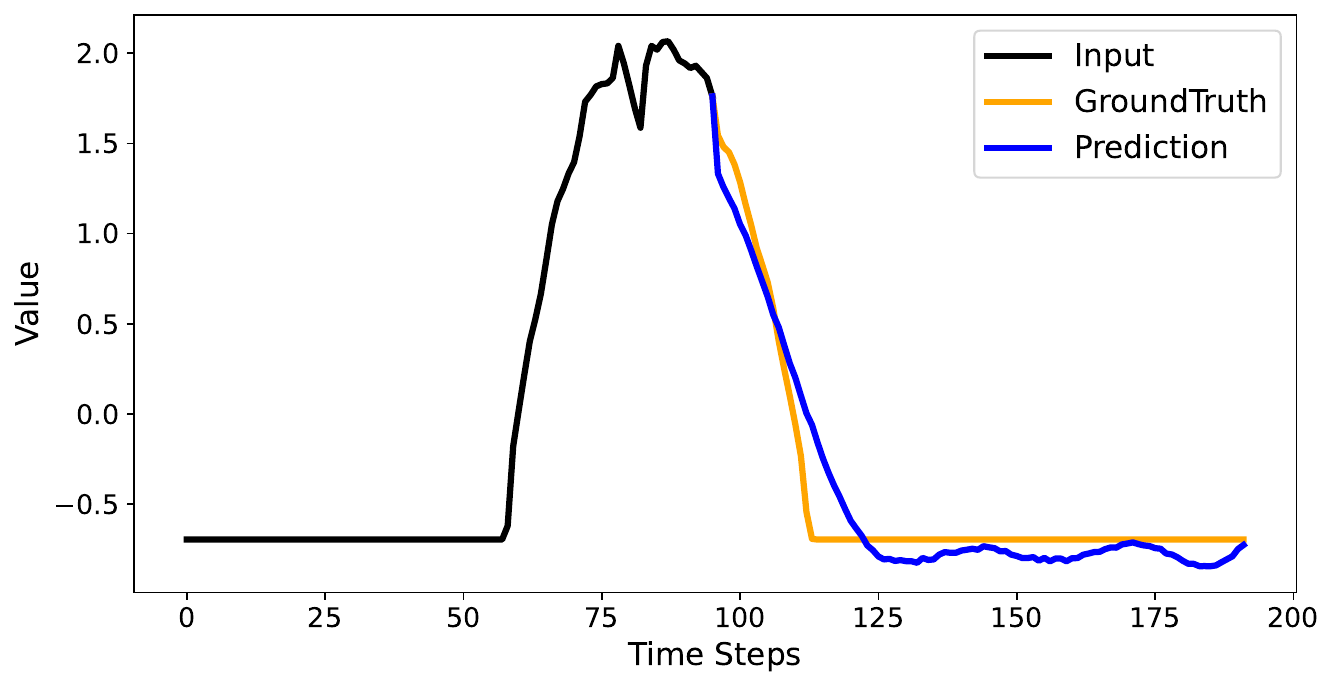}
    }
    \caption{Visualization of forecasting results on the Solar-Energy dataset with lookback window $T=96$ and prediction length $\tau=96$.}
    \label{fig:forecasting_solar_96_96}
\end{figure}

\begin{figure}[h]
    \centering
    \subfigure[Paloma]{
    \includegraphics[width=0.48\linewidth]{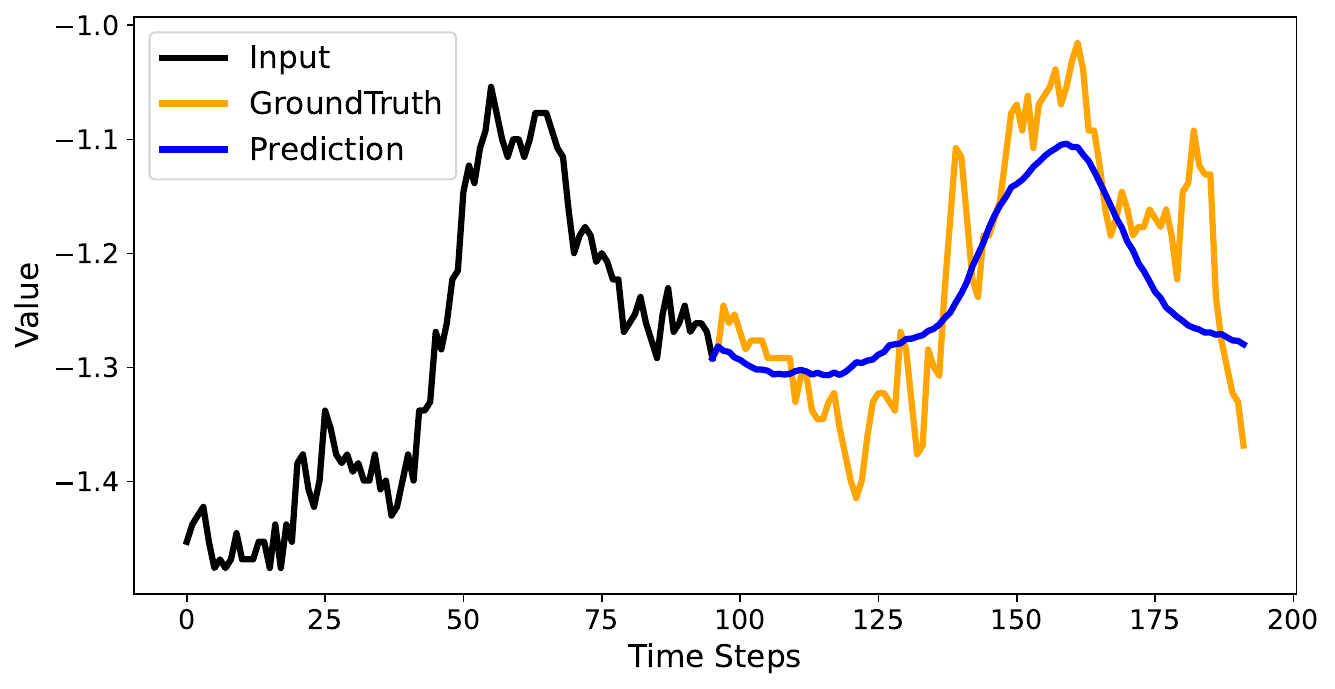}
    }
    \hspace{-4mm}
    \subfigure[GTR]{
    \includegraphics[width=0.48\linewidth]{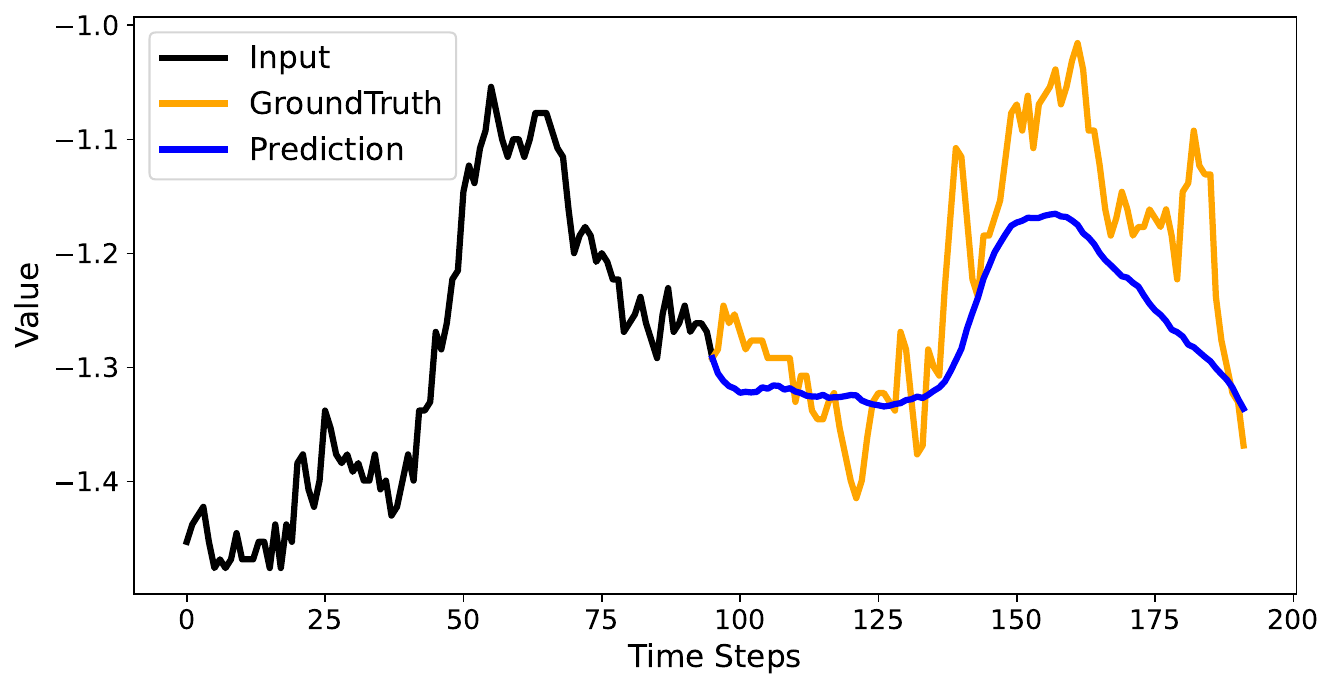}
    }\\
    \hspace{-4mm}
    \subfigure[CycleNet]{
    \includegraphics[width=0.48\linewidth]{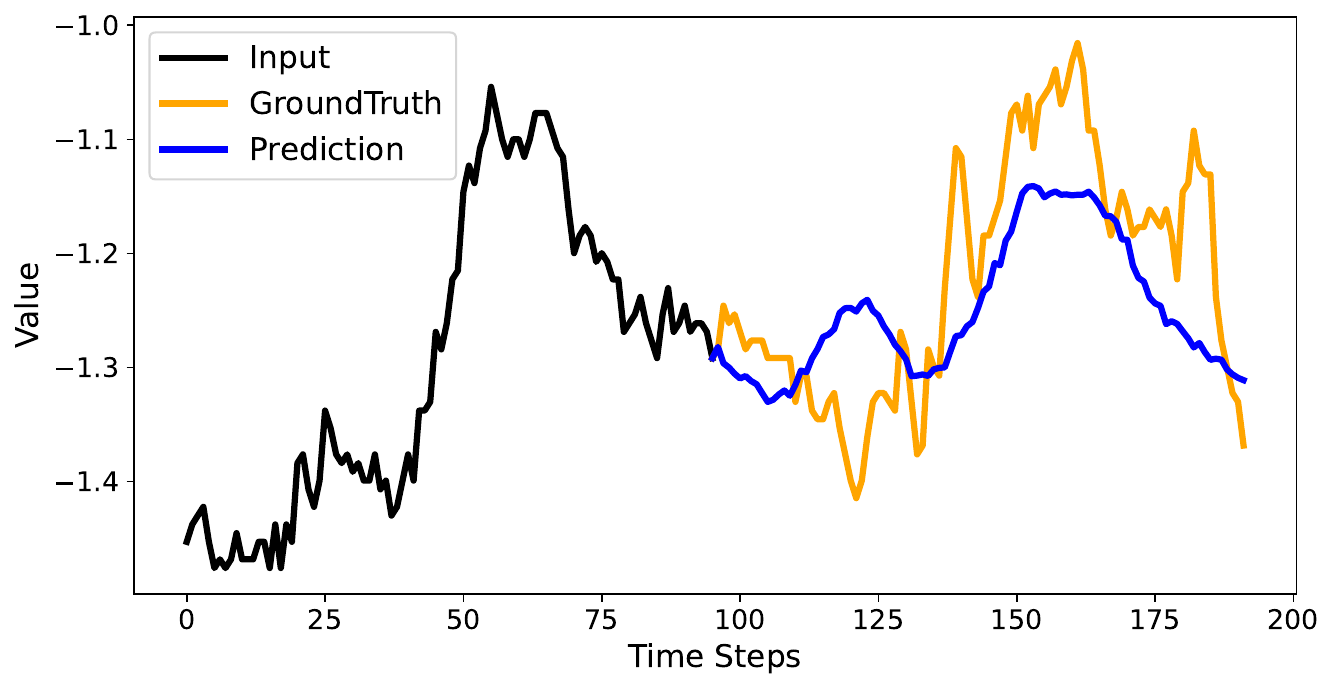}
    }
    \hspace{-4mm}
    \subfigure[TimeXer]{
    \includegraphics[width=0.48\linewidth]{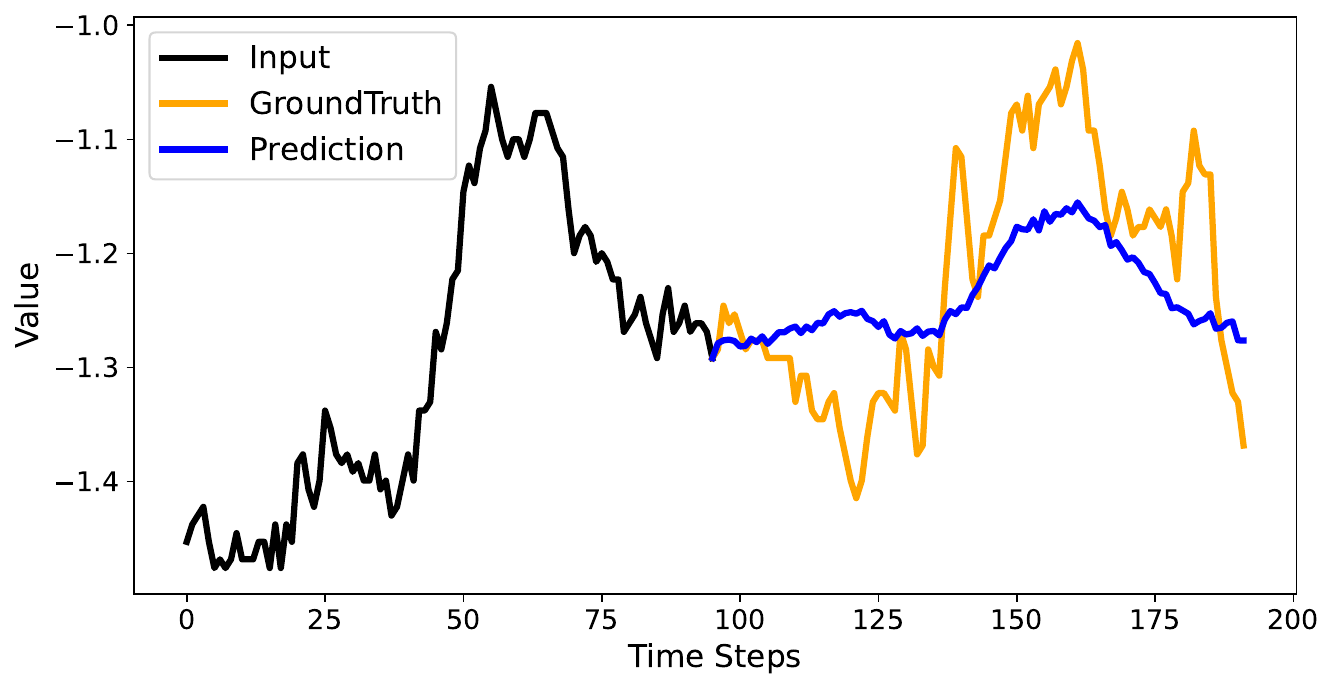}
    }\\
        \hspace{-4mm}
    \subfigure[FilterNet]{
    \includegraphics[width=0.48\linewidth]{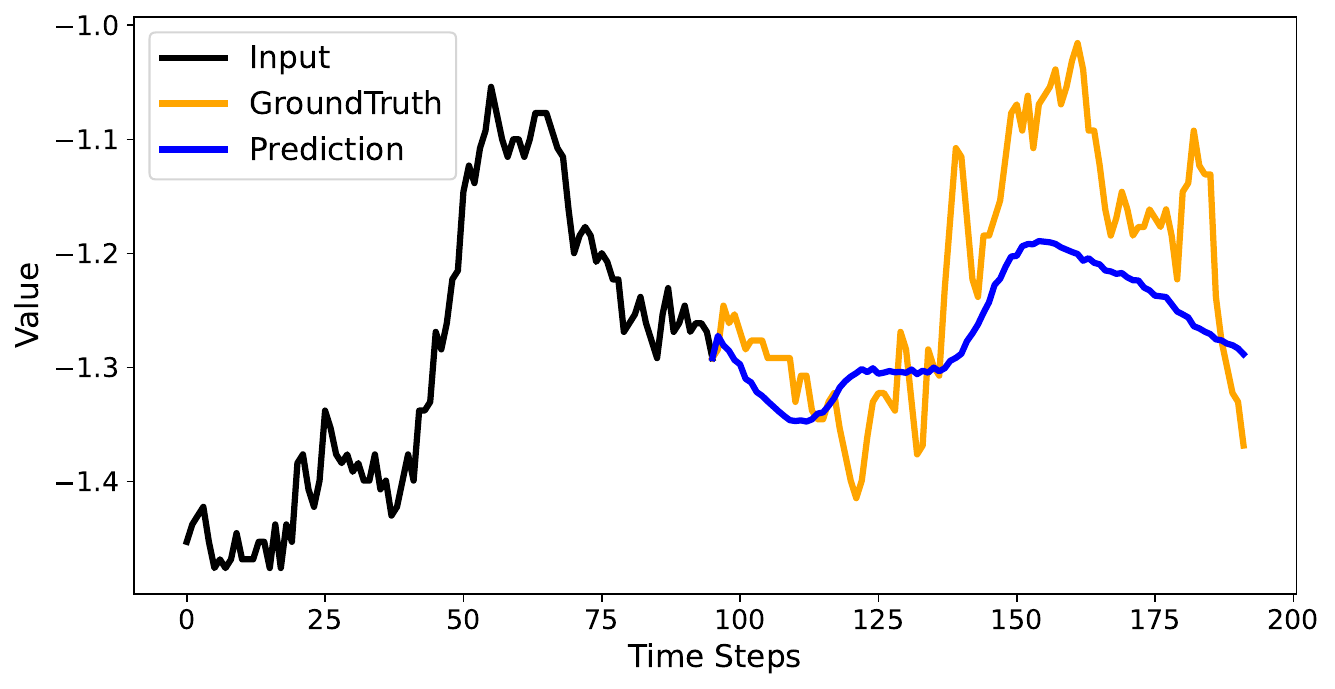}
    }
    \hspace{-4mm}
    \subfigure[iTransformer]{
    \includegraphics[width=0.48\linewidth]{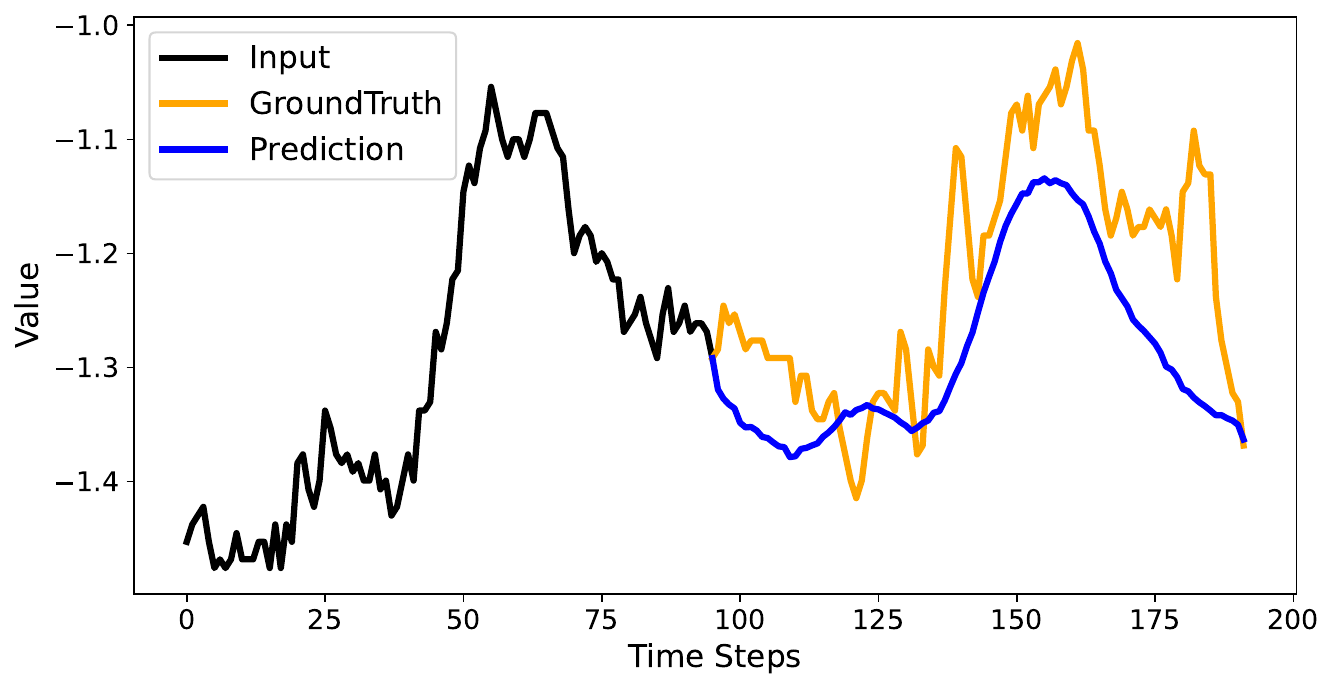}
    }
    \caption{Visualization of forecasting results on the ETTm1 dataset with lookback window $T=96$ and prediction length $\tau=96$.}
    \label{fig:forecasting_ETTm1_96_96}
\end{figure}

%% file: tables/dataset.tex
\begin{table}[h]
\centering
\caption{Detailed descriptions of the datasets. \textit{Dim} indicates the number of variates in each dataset. \textit{Prediction Length} refers to the number of future time steps to be predicted. \textit{Dataset Size} represents the total number of time points in the (Train, Validation, Test) splits, respectively. \textit{Frequency} indicates the sampling interval of the time series. \textit{Cycle Length} denotes the length of the global cycle.}
\vspace{1mm}
\label{tab:datasets}
\scalebox{0.81}{
\begin{tabular}{l c c c c c c}
\toprule[2pt]
Dataset & Dim & Prediction Length & Dataset Size & Frequency &Cycle Length & Domain \\
\midrule
ETTh1, ETTh2 & 7 & \{96, 192, 336, 720\} & (8545, 2881, 2881) & Hourly  &24   & Electricity \\
ETTm1, ETTm2 & 7 & \{96, 192, 336, 720\} & (34465, 11521, 11521) & 15min   &96  & Electricity \\
Weather & 21 & \{96, 192, 336, 720\} & (36792, 5271, 10540) & 10min   &144  & Weather \\
Electricity & 321 & \{96, 192, 336, 720\} & (18317, 2633, 5261) & Hourly   &168  & Electricity \\
Traffic & 862 & \{96, 192, 336, 720\} & (12185, 1757, 3509) & Hourly  &168   & Transportation \\
Solar-Energy & 137 & \{96, 192, 336, 720\} & (36601, 5161, 10417) & 10min   &144  & Energy \\
PEMS03 & 358 & \{12, 24, 48, 96\} & (15617, 5135, 5135) & 5min   &288  & Transportation \\
PEMS04 & 307 & \{12, 24, 48, 96\} & (10172, 3375, 3375) & 5min   &288  & Transportation \\
PEMS07 & 883 & \{12, 24, 48, 96\} & (16911, 5622, 5622) & 5min   &288  & Transportation \\
PEMS08 & 170 & \{12, 24, 48, 96\} & (10690, 3548, 3548) & 5min   &288  & Transportation \\
\bottomrule[2pt]
\end{tabular}
}
\end{table}

%% file: tables/input_96_appendix.tex
\begin{table*}[!b]
    \centering
    \caption{Performance comparison with prediction lengths $\tau \in \{12, 24, 48, 96\}$ for PEMS and $\tau \in \{96, 192, 336, 720\}$ for others and fixed lookback window $T=96$. Results are averaged over four prediction lengths. Highlight the best results in \textcolor{red}{\textbf{red}} and the second-best results in \textcolor{blue}{\textbf{blue}}.}
    \scalebox{0.49}{
    \begin{tabular}{l | c|c c|c c|c c| c c|c c |c c |c c |c c |c c |c c |c c  }
    \toprule[2pt]
       \multicolumn{2}{c|}{Models}  &\multicolumn{2}{c|}{\textbf{Paloma}} &\multicolumn{2}{c|}{DistDF} &\multicolumn{2}{c|}{GTR} &\multicolumn{2}{c|}{RAFT} &\multicolumn{2}{c|}{FreDF}  &\multicolumn{2}{c|}{TimeXer}  &\multicolumn{2}{c|}{CycleNet} &\multicolumn{2}{c|}{FilterNet} &\multicolumn{2}{c|}{iTransformer} &\multicolumn{2}{c|}{PatchTST}  &\multicolumn{2}{c}{DLinear}  \\

        \multicolumn{2}{c|}{}  &\multicolumn{2}{c|}{[\textbf{ours}]} &\multicolumn{2}{c|}{[ICLR'26]} &\multicolumn{2}{c|}{[ICLR'26]} &\multicolumn{2}{c|}{[ICML'25]} &\multicolumn{2}{c|}{[ICLR'25]}  &\multicolumn{2}{c|}{[NeurIPS'24]} &\multicolumn{2}{c|}{[NeurIPS'24]} &\multicolumn{2}{c|}{[NeurIPS'24]} &\multicolumn{2}{c|}{[ICLR'24]} &\multicolumn{2}{c|}{[ICLR'23]} &\multicolumn{2}{c}{[AAAI'23]} \\
       
       \cmidrule(r){1-2}\cmidrule(lr){3-4}\cmidrule(lr){5-6}\cmidrule(lr){7-8}\cmidrule(lr){9-10}\cmidrule(lr){11-12}\cmidrule(lr){13-14}\cmidrule(lr){15-16}\cmidrule(lr){17-18}\cmidrule(lr){19-20}\cmidrule(lr){21-22}\cmidrule(lr){23-24}
       
       \multicolumn{2}{c|}{Metrics} &MSE &MAE &MSE &MAE &MSE &MAE &MSE &MAE &MSE &MAE &MSE &MAE &MSE &MAE &MSE &MAE &MSE &MAE &MSE &MAE &MSE &MAE   \\
    
         \midrule[1pt]
         \multirow{5}*{\rotatebox{90}{ETTm1}}
         & 96  &\textcolor{red}{\textbf{0.295}}   &\textcolor{red}{\textbf{0.332}}   &0.316 &0.357  &\textcolor{blue}{\textbf{0.305}} &\textcolor{blue}{\textbf{0.349}} &0.325 &0.369  &0.324 &0.362  &0.318 &0.356 &0.319 &0.360 &0.321 &0.361   &0.334  &0.368  &0.329  &0.367  &0.345  &0.372 \\  
         
         & 192  &\textcolor{red}{\textbf{0.343}}   &\textcolor{red}{\textbf{0.358}}  &0.358 &0.380   &\textcolor{blue}{\textbf{0.349}} &\textcolor{blue}{\textbf{0.375}} &0.364 &0.388  &0.373 &0.385 &0.362 &0.383 &0.360 &0.381  &0.367 &0.387   &0.377  &0.391  &0.367  &0.385  &0.380  &0.389 \\  
         
         & 336  &\textcolor{red}{\textbf{0.376}}   &\textcolor{red}{\textbf{0.381}}  &0.392 &0.404  &\textcolor{blue}{\textbf{0.380}} &\textcolor{blue}{\textbf{0.398}} &0.391 &0.407 &0.402 &0.404  &0.395 &0.407 &0.389 &0.403 &0.401 &0.409  &0.426  &0.420  &0.399  &0.410  &0.413  &0.413 \\  
         
         & 720  &\textcolor{blue}{\textbf{0.439}}   &\textcolor{red}{\textbf{0.417}}  &0.448 &0.437  &\textcolor{red}{\textbf{0.435}} &\textcolor{blue}{\textbf{0.436}} &0.444 &0.438 &0.469 &0.444  &0.452 &0.441 &0.447 &0.441 &0.477 &0.448  &0.491  &0.459  &0.454  &0.439  &0.474  &0.453 \\  
         
         \cmidrule(lr){2-24}
         & Avg  &\textcolor{red}{\textbf{0.363}}   &\textcolor{red}{\textbf{0.372}}  &0.378 &0.394  &\textcolor{blue}{\textbf{0.367}} &\textcolor{blue}{\textbf{0.389}} &0.381 &0.400  &0.392 &0.399  &0.382 &0.397 &0.379 &0.396 &0.392 &0.401  &0.407  &0.410  &0.387  &0.400  &0.403  &0.407 \\  
         \midrule[1pt]

        \multirow{5}*{\rotatebox{90}{ETTm2}}
         & 96  &\textcolor{blue}{\textbf{0.164}}   &\textcolor{red}{\textbf{0.242}}  &0.174 &0.256  &0.168 &0.249 &0.176 &0.264  &0.173 &0.252  &0.171 &0.256 &\textcolor{red}{\textbf{0.163}} &\textcolor{blue}{\textbf{0.246}}  &0.175 &0.258  &0.180 &0.264 &0.175 &0.259 &0.193 &0.292 \\  
         
         & 192  &\textcolor{red}{\textbf{0.229}}   &\textcolor{red}{\textbf{0.284}}  &0.239 &0.298  &\textcolor{blue}{\textbf{0.232}} &0.292 &0.241 &0.306  &0.241 &0.298  &0.237 &0.299 &\textcolor{red}{\textbf{0.229}} &\textcolor{blue}{\textbf{0.290}}  &0.240 &0.301  &0.250 &0.309 &0.241 &0.302 &0.284 &0.362 \\  
         
         & 336  &0.289   &\textcolor{red}{\textbf{0.323}}  &0.300 &0.338  &\textcolor{blue}{\textbf{0.287}} &0.330 &0.302 &0.345  &0.298 &0.334  &0.296 &0.338 &\textcolor{red}{\textbf{0.284}} &\textcolor{blue}{\textbf{0.327}}  &0.311 &0.347  &0.311 &0.348 &0.305 &0.343 &0.369 &0.427 \\  
         
         & 720  &\textcolor{blue}{\textbf{0.387}}   &\textcolor{red}{\textbf{0.383}}  &0.397 &0.394  &\textcolor{red}{\textbf{0.386}} &\textcolor{blue}{\textbf{0.390}} &0.404 &0.405  &0.398 &0.393  &0.392 &0.394 &0.389 &0.391  &0.414 &0.405 &0.412 &0.407 &0.402 &0.400 &0.554 &0.522 \\  
         
         \cmidrule(lr){2-24}
         & Avg  &\textcolor{blue}{\textbf{0.267}}   &\textcolor{red}{\textbf{0.308}}  &0.277 &0.321  &0.268  &0.315  &0.281  &0.330  &0.278 &0.319  &0.274 &0.322 &\textcolor{red}{\textbf{0.266}} &\textcolor{blue}{\textbf{0.314}}  &0.285 &0.328 &0.288 &0.332 &0.281 &0.326 &0.350 &0.401 \\ 
         \midrule[1pt]

         \multirow{5}*{\rotatebox{90}{ETTh1}}
         & 96  &\textcolor{red}{\textbf{0.360}}   &\textcolor{red}{\textbf{0.378}}   &0.373 &0.393  &\textcolor{blue}{\textbf{0.367}} &\textcolor{blue}{\textbf{0.391}} &0.376 &0.401  &0.382 &0.400  &0.382 &0.403 &0.375 &0.395  &0.382 &0.402  &0.386 &0.405 &0.414 &0.419 &0.386 &0.400 \\  
         
         & 192  &\textcolor{red}{\textbf{0.414}}   &\textcolor{red}{\textbf{0.409}}   &0.428 &0.425  &0.420 &\textcolor{blue}{\textbf{0.421}} &\textcolor{blue}{\textbf{0.417}} &0.424  &0.430 &0.427  &0.429 &0.435 &0.436 &0.428  &0.430 &0.429  &0.441 &0.436 &0.460 &0.445 &0.437 &0.432 \\  
         
         & 336  &\textcolor{red}{\textbf{0.452}}   &\textcolor{red}{\textbf{0.428}}   &0.466 &0.445  &0.476 &0.447 &\textcolor{blue}{\textbf{0.454}} &\textcolor{blue}{\textbf{0.441}}  &0.474 &0.451  &0.468 &0.448 &0.496 &0.455  &0.472 &0.451  &0.487 &0.458 &0.501 &0.466 &0.481 &0.459 \\  
         
         & 720  &\textcolor{blue}{\textbf{0.458}}   &\textcolor{red}{\textbf{0.449}}   &\textcolor{red}{\textbf{0.453}} &\textcolor{blue}{\textbf{0.453}}  &0.493 &0.476 &0.465 &0.465 &0.463 &0.462  &0.469 &0.461 &0.520 &0.484  &0.481 &0.473  &0.503 &0.491 &0.500 &0.488 &0.519 &0.516 \\  
         
         \cmidrule(lr){2-24}
         & Avg  &\textcolor{red}{\textbf{0.421}}   &\textcolor{red}{\textbf{0.416}}  &0.430 &\textcolor{blue}{\textbf{0.429}}  &0.439 &0.434 &\textcolor{blue}{\textbf{0.428}} &0.433  &0.437 &0.435 &0.437 &0.437 &0.457 &0.441  &0.441 &0.439  &0.454 &0.448 &0.469 &0.455 &0.456 &0.452 \\ 
         \midrule[1pt]

         \multirow{5}*{\rotatebox{90}{ETTh2}}
         & 96  &\textcolor{red}{\textbf{0.278}}   &\textcolor{red}{\textbf{0.325}}  &0.287 &\textcolor{blue}{\textbf{0.336}}  &0.290 &0.342 &0.289 &0.342  &0.289 &0.337 &\textcolor{blue}{\textbf{0.286}} &0.338 &0.298 &0.344  &0.293 &0.343  &0.297 &0.349 &0.302 &0.348 &0.333 &0.387 \\  
         
         & 192  &\textcolor{red}{\textbf{0.353}}   &\textcolor{red}{\textbf{0.376}}  &\textcolor{blue}{\textbf{0.358}} &\textcolor{blue}{\textbf{0.381}}  & 0.362 &0.389 &0.382 &0.400  &0.363 &0.385 &0.363 &0.389 &0.372 &0.396  &0.374 &0.396  &0.380 &0.400 &0.388 &0.400 &0.477 &0.476 \\  
         
         & 336  &\textcolor{red}{\textbf{0.401}}   &\textcolor{red}{\textbf{0.413}}  &\textcolor{blue}{\textbf{0.408}} &0.421  &0.414 &0.429 &0.424 &0.439  &0.419 &0.426 &0.414 &\textcolor{blue}{\textbf{0.423}} &0.431 &0.439  &0.417 &0.430  &0.428 &0.432 &0.426 &0.433 &0.594 &0.541 \\  
         
         & 720  &\textcolor{blue}{\textbf{0.412}}   &\textcolor{red}{\textbf{0.429}}  &0.416 &0.435   &0.423 &0.442 &0.433 &0.457 &0.415 &0.437 &\textcolor{red}{\textbf{0.408}} &\textcolor{blue}{\textbf{0.432}} &0.450 &0.458  &0.449 &0.460  &0.427 &0.445 &0.431 &0.446 &0.831 &0.657 \\  
         
         \cmidrule(lr){2-24}
         & Avg  &\textcolor{red}{\textbf{0.361}}   &\textcolor{red}{\textbf{0.386}}  &\textcolor{blue}{\textbf{0.367}} &\textcolor{blue}{\textbf{0.393}}  &0.372 &0.400 &0.382 &0.410 &0.371 &0.396 &0.368 &0.396 &0.388 &0.409  &0.383 &0.407 &0.383 &0.407 &0.387 &0.407 &0.559 &0.515 \\ 
         \midrule[1pt]

         \multirow{5}*{\rotatebox{90}{Electricity}}
         & 96  &\textcolor{red}{\textbf{0.134}}   &\textcolor{red}{\textbf{0.222}}  &0.137 &0.235   &\textcolor{red}{\textbf{0.134}} &\textcolor{blue}{\textbf{0.229}} &0.152 &0.257  &0.144 &0.233  &0.140 &0.242 &\textcolor{blue}{\textbf{0.136}} &\textcolor{blue}{\textbf{0.229}}  &0.147 &0.245  &0.148 &0.240 &0.181 &0.270 &0.197 &0.282 \\  
         
         & 192  &\textcolor{red}{\textbf{0.151}}   &\textcolor{red}{\textbf{0.238}}  &0.159 &0.257  &\textcolor{blue}{\textbf{0.152}} &0.245 &0.156 &0.257  &0.159 &0.247  &0.157 &0.256 &\textcolor{blue}{\textbf{0.152}} &\textcolor{blue}{\textbf{0.244}}  &0.160 &0.250  &0.162 &0.253 &0.188 &0.274 &0.196 &0.285 \\ 
         
         & 336  &\textcolor{red}{\textbf{0.170}}   &\textcolor{red}{\textbf{0.256}}  &0.178 &0.272  &\textcolor{blue}{\textbf{0.171}} &0.264 &\textcolor{blue}{\textbf{0.171}} &0.272  &0.172 &\textcolor{blue}{\textbf{0.263}}  &0.176 &0.275 &\textcolor{red}{\textbf{0.170}} &0.264  &0.173 &0.267  &0.178 &0.269 &0.204 &0.293 &0.209 &0.301 \\  
         
         & 720  &\textcolor{blue}{\textbf{0.208}}   &\textcolor{red}{\textbf{0.289}}  &0.212 &0.302  &\textcolor{blue}{\textbf{0.208}} &0.300 &0.221 &0.300  &\textcolor{red}{\textbf{0.204}} &\textcolor{blue}{\textbf{0.294}}  &0.211 &0.306 &0.212 &0.299  &0.210 &0.309  &0.225 &0.317 &0.246 &0.324 &0.245 &0.333 \\  
         
         \cmidrule(lr){2-24}
         & Avg  &\textcolor{red}{\textbf{0.166}}   &\textcolor{red}{\textbf{0.251}}  &0.172 &0.267 &\textcolor{red}{\textbf{0.166}} &0.260 &0.175 &0.272  &0.170 &\textcolor{blue}{\textbf{0.259}} &0.171 &0.270 &\textcolor{blue}{\textbf{0.168}} &\textcolor{blue}{\textbf{0.259}} &0.173 &0.268 &0.178 &0.270 &0.205 &0.290 &0.212 &0.300 \\ 
         \midrule[1pt]

         \multirow{5}*{\rotatebox{90}{Traffic}}
         & 96  &0.425   &\textcolor{red}{\textbf{0.243}}   &\textcolor{red}{\textbf{0.380}}  &\textcolor{blue}{\textbf{0.262}}  &0.440 &0.263 &\textcolor{blue}{\textbf{0.388}} &0.265  &0.391 &0.265  &0.428 &0.271 &0.458 &0.296  &0.430 &0.294 &0.395 &0.268 &0.462 &0.290 &0.650 &0.396 \\
         
         & 192  &0.442   &\textcolor{red}{\textbf{0.254}}  &\textcolor{blue}{\textbf{0.407}} &0.275  &0.454 &0.274 &\textcolor{red}{\textbf{0.400}} &0.278  &0.410 &\textcolor{blue}{\textbf{0.273}}  &0.448 &0.282 &0.457 &0.294  &0.452 &0.307  &0.417 &0.276 &0.466 &0.290 &0.598 &0.370 \\  
         
         & 336  &0.462   &\textcolor{red}{\textbf{0.262}}  &0.429 &0.284  &0.472 &0.282 &\textcolor{red}{\textbf{0.411}}  &0.287  &\textcolor{blue}{\textbf{0.424}} &\textcolor{blue}{\textbf{0.280}}  &0.473 &0.289 &0.470 &0.299  &0.470 &0.316  &0.433 &0.283 &0.482 &0.300 &0.605 &0.373 \\  
         
         & 720  &0.514   &\textcolor{red}{\textbf{0.279}}  &\textcolor{red}{\textbf{0.452}} &\textcolor{blue}{\textbf{0.297}}  &0.514 &0.301 &\textcolor{blue}{\textbf{0.456}} &0.305  &0.460 &0.298  &0.516 &0.307 &0.502 &0.314  &0.498 &0.323 &0.467 &0.302 &0.514 &0.320 &0.645 &0.394 \\  
         
         \cmidrule(lr){2-24}
         & Avg  &0.461   &\textcolor{red}{\textbf{0.260}}  &\textcolor{blue}{\textbf{0.417}} &\textcolor{blue}{\textbf{0.279}} &0.470 &0.280 &\textcolor{red}{\textbf{0.414}} &0.284  &0.421 &\textcolor{blue}{\textbf{0.279}} &0.466 &0.287 &0.472 &0.301  &0.463 &0.310 &0.428 &0.282 &0.481 &0.300 &0.625 &0.383 \\ 
         \midrule[1pt] 

         \multirow{5}*{\rotatebox{90}{Weather}}
         & 96  &\textcolor{red}{\textbf{0.154}}   &\textcolor{red}{\textbf{0.190}}  &0.164 &0.209  &\textcolor{red}{\textbf{0.154}} &\textcolor{blue}{\textbf{0.200}} &0.185 &0.241  &0.164 &0.202  &\textcolor{blue}{\textbf{0.157}} &0.205 &0.158 &0.203  &0.162 &0.207  &0.174 &0.214 &0.177 &0.210 &0.196 &0.255 \\  
         
         & 192  &\textcolor{red}{\textbf{0.195}}   &\textcolor{red}{\textbf{0.238}}  &0.212 &0.252 &\textcolor{blue}{\textbf{0.202}} &\textcolor{blue}{\textbf{0.244}} &0.237 &0.287  &0.220 &0.253  &0.204 &0.247 &0.207 &0.247 &0.210 &0.250  &0.221 &0.254 &0.225 &0.250 &0.237 &0.296 \\  
         
         & 336  &\textcolor{red}{\textbf{0.249}}   &\textcolor{red}{\textbf{0.284}}  &0.270 &0.295  &\textcolor{blue}{\textbf{0.259}} &\textcolor{blue}{\textbf{0.287}} &0.290 &0.327 &0.275 &0.294  &0.261 &0.290 &0.262 &0.289  &0.265 &0.290  &0.278 &0.296 &0.278 &0.290 &0.283 &0.335 \\  
         
         & 720  &\textcolor{red}{\textbf{0.324}}   &\textcolor{red}{\textbf{0.336}}  &0.348 &0.345  &0.341 &0.342 &0.367 &0.380  &0.356 &0.347  &\textcolor{blue}{\textbf{0.340}} &0.341 &0.344 &0.344  &0.342 &\textcolor{blue}{\textbf{0.340}}  &0.358 &0.349 &0.354 &\textcolor{blue}{\textbf{0.340}} &0.345 &0.381 \\ 
         
         \cmidrule(lr){2-24} 
         & Avg  &\textcolor{red}{\textbf{0.231}}   &\textcolor{red}{\textbf{0.262}}   &0.248 &0.275  &\textcolor{blue}{\textbf{0.239}} &\textcolor{blue}{\textbf{0.268}} &0.270 &0.309  &0.254 &0.274  &0.241 &0.271 &0.243 &0.271  &0.245 &0.272  &0.258 &0.278 &0.259 &0.273 &0.265 &0.317 \\ 
         \midrule[1pt]

        \multirow{5}*{\rotatebox{90}{Solar-Energy}}
         & 96  &\textcolor{red}{\textbf{0.175}}   &\textcolor{red}{\textbf{0.224}}   &0.185   &0.237   &\textcolor{blue}{\textbf{0.176}}  &\textcolor{blue}{\textbf{0.235}} &0.278 &0.278   &0.188   &0.240   &0.215 &0.295 &0.190 &0.247  &0.205 &0.242  &0.203 &0.237 &0.234 &0.286 &0.290 &0.378 \\  
         
         & 192  &\textcolor{red}{\textbf{0.192}}   &\textcolor{red}{\textbf{0.240}}   &0.202   &0.250   &\textcolor{blue}{\textbf{0.193}} &\textcolor{blue}{\textbf{0.244}} &0.297 &0.297   &0.206   &0.255   &0.236 &0.301 &0.210 &0.266   &0.233 &0.265  &0.233 &0.261 &0.267 &0.310 &0.320 &0.398 \\  
         
         & 336  &\textcolor{blue}{\textbf{0.204}}   &\textcolor{red}{\textbf{0.235}}   &0.212   &0.265    &\textcolor{red}{\textbf{0.201}} &\textcolor{blue}{\textbf{0.250}} &0.311 &0.311   &0.217   &0.272   &0.252 &0.307 &0.217 &0.266  &0.249 &0.278  &0.248 &0.273 &0.290 &0.315 &0.353 &0.415 \\  
         
         & 720  &\textcolor{blue}{\textbf{0.215}}   &\textcolor{red}{\textbf{0.238}}   &0.230   &0.268    &\textcolor{red}{\textbf{0.205}} &\textcolor{blue}{\textbf{0.251}} &0.318 &0.325   &0.231   &0.282   &0.244 &0.305 &0.223 &0.266  &0.253 &0.281  &0.249 &0.275 &0.289 &0.317 &0.356 &0.413 \\  
         
         \cmidrule(lr){2-24}
         & Avg  &\textcolor{blue}{\textbf{0.197}}   &\textcolor{red}{\textbf{0.234}}   &0.207   &0.255    &\textcolor{red}{\textbf{0.194}} &\textcolor{blue}{\textbf{0.245}} &0.301 &0.303  &0.211   &0.262   &0.237 &0.302 &0.210 &0.261  &0.235 &0.266  &0.233 &0.262 &0.270 &0.307 &0.330 &0.401 \\ 
         \midrule[1pt]

         \multirow{5}*{\rotatebox{90}{PEMS03}}
         & 12  &\textcolor{blue}{\textbf{0.058}}   &\textcolor{blue}{\textbf{0.158}}  &0.068  &0.174   &\textcolor{red}{\textbf{0.057}} &\textcolor{red}{\textbf{0.156}} &0.071 &0.175  &0.068  &0.172  &0.070 &0.173 &0.066 &0.172  &0.071 &0.177  &0.071 &0.174 &0.099 &0.216 &0.122 &0.243 \\  
         
         & 24  &\textcolor{blue}{\textbf{0.074}}   &\textcolor{blue}{\textbf{0.175}}  &0.094   &0.205  &\textcolor{red}{\textbf{0.070}}  &\textcolor{red}{\textbf{0.172}} &0.099 &0.201  &0.096 &0.205  &0.092 &0.194 &0.089 &0.201  &0.102 &0.213  &0.093 &0.201 &0.142 &0.259 &0.201 &0.317 \\  
         
         & 48  &\textcolor{blue}{\textbf{0.101}}   &\textcolor{blue}{\textbf{0.205}}  &0.138   &0.252   &\textcolor{red}{\textbf{0.094}}  &\textcolor{red}{\textbf{0.198}}  &0.153 &0.243  &0.128  &0.240   &0.129 &0.229 &0.136 &0.247  &0.162 &0.272  &0.125 &0.236 &0.211 &0.319 &0.333 &0.425 \\  
         
         & 96  &\textcolor{blue}{\textbf{0.142}}   &\textcolor{blue}{\textbf{0.236}}  &0.191   &0.301  &\textcolor{red}{\textbf{0.127}}  &\textcolor{red}{\textbf{0.228}} &0.253 &0.301   &0.161   &0.269  &0.157 &0.261 &0.182 &0.282  &0.244 &0.340 &0.164 &0.275 &0.269 &0.370 &0.457 &0.515 \\  
         
         \cmidrule(lr){2-24}
         & Avg  &\textcolor{blue}{\textbf{0.094}}   &\textcolor{blue}{\textbf{0.194}}  &0.123   &0.233    &\textcolor{red}{\textbf{0.087}} &\textcolor{red}{\textbf{0.189}} &0.144 &0.230  &0.113   &0.219   &0.112 &0.214 &0.118 &0.226 &0.145 &0.251  &0.113 &0.222 &0.180 &0.291 &0.278 &0.375 \\ 
         \midrule[1pt]

         
         
         
         

         \multirow{5}*{\rotatebox{90}{PEMS04}}
         & 12  &\textcolor{red}{\textbf{0.065}}  &\textcolor{red}{\textbf{0.160}}   &\textcolor{blue}{\textbf{0.071}}   &0.167    &\textcolor{red}{\textbf{0.065}}  &\textcolor{blue}{\textbf{0.164}} &0.077 &0.179   &0.073   &0.174  &0.074 &0.178 &0.078 &0.186  &0.082 &0.190  &0.078 &0.183 &0.105 &0.224 &0.148 &0.272 \\  
         
         & 24  &\textcolor{red}{\textbf{0.074}}   &\textcolor{red}{\textbf{0.171}}   &0.080   &0.181     &\textcolor{blue}{\textbf{0.075}}  &\textcolor{blue}{\textbf{0.177}}  &0.091 &0.195   &0.085   &0.186   &0.087 &0.195 &0.099 &0.212  &0.110 &0.224 &0.095 &0.205 &0.153 &0.275 &0.224 &0.340 \\  
         
         & 48  &\textcolor{red}{\textbf{0.090}}   &\textcolor{red}{\textbf{0.189}}   &0.097   &0.205     &\textcolor{blue}{\textbf{0.093}}  &\textcolor{blue}{\textbf{0.197}}  &0.111 &0.218  &0.101   &0.212   &0.110 &0.214 &0.133 &0.248  &0.160 &0.276  &0.120 &0.233 &0.229 &0.339 &0.355 &0.437 \\  
         
         & 96  &\textcolor{red}{\textbf{0.110}}   &\textcolor{red}{\textbf{0.207}}   &0.120   &0.228     &\textcolor{blue}{\textbf{0.114}}  &\textcolor{blue}{\textbf{0.219}}  &0.137 &0.247   &0.125   &0.235   &0.148 &0.251 &0.167 &0.281  &0.234 &0.343  &0.150 &0.262 &0.291 &0.389 &0.452 &0.504 \\  
         
         \cmidrule(lr){2-24}
         & Avg  &\textcolor{red}{\textbf{0.085}}   &\textcolor{red}{\textbf{0.182}}   &0.092   &0.195     &\textcolor{blue}{\textbf{0.087}}  &\textcolor{blue}{\textbf{0.189}} &0.104 &0.210   &0.096   &0.202   &0.105 &0.209 &0.119 &0.232  &0.146 &0.258  &0.111 &0.221 &0.195 &0.307 &0.295 &0.388 \\ 
         \midrule[1pt]

         \multirow{5}*{\rotatebox{90}{PEMS07}}
         & 12  &\textcolor{red}{\textbf{0.051}}   &\textcolor{red}{\textbf{0.138}}    &\textcolor{blue}{\textbf{0.056}}   &0.149    &\textcolor{red}{\textbf{0.051}}  &\textcolor{blue}{\textbf{0.142}} &0.062 &0.159    &0.060   &0.156  &0.057 &0.152 &0.062 &0.162  &0.064 &0.163  &0.067 &0.165 &0.095 &0.207 &0.115 &0.242 \\  
         
         & 24  &\textcolor{blue}{\textbf{0.064}}   &\textcolor{red}{\textbf{0.152}}   &0.073   &0.168     &\textcolor{red}{\textbf{0.063}} &\textcolor{blue}{\textbf{0.156}}  &0.078 &0.178   &0.077   &0.170  &0.079 &0.179 &0.086 &0.192  &0.093 &0.200  &0.088 &0.190 &0.150 &0.262 &0.210 &0.329 \\  
         
         & 48  &\textcolor{blue}{\textbf{0.084}}   &\textcolor{red}{\textbf{0.172}}   &0.089   &0.183     &\textcolor{red}{\textbf{0.082}} &\textcolor{blue}{\textbf{0.178}} &0.102 &0.203  &0.093   &0.188   &0.099 &0.191 &0.128 &0.234  &0.137 &0.248  &0.110 &0.215 &0.253 &0.340 &0.398 &0.458 \\  
         
         & 96  &\textcolor{blue}{\textbf{0.112}}   &\textcolor{red}{\textbf{0.194}}   &0.119   &0.209     &\textcolor{red}{\textbf{0.110}} &\textcolor{blue}{\textbf{0.202}} &0.135 &0.231    &0.124   &0.215  &0.107 &0.205 &0.176 &0.268  &0.198 &0.306  &0.139 &0.245 &0.346 &0.404 &0.594 &0.553 \\  
         
         \cmidrule(lr){2-24}
         & Avg  &\textcolor{blue}{\textbf{0.078}}   &\textcolor{red}{\textbf{0.164}}    &0.084   &0.177     &\textcolor{red}{\textbf{0.076}} &\textcolor{blue}{\textbf{0.169}}  &0.094 &0.193    &0.089   &0.182  &0.085 &0.182 &0.113 &0.214  &0.123 &0.229  &0.101 &0.204 &0.211 &0.303 &0.329 &0.396 \\ 
         \midrule[1pt]

        \multirow{5}*{\rotatebox{90}{PEMS08}}
         & 12  &\textcolor{red}{\textbf{0.069}}   &\textcolor{red}{\textbf{0.163}}  &0.076 &0.177  &\textcolor{blue}{\textbf{0.071}}  &\textcolor{blue}{\textbf{0.168}}  &0.078 &0.181  &0.080 &0.182  &0.075 &0.176 &0.082 &0.185  &0.080 &0.182 &0.079 &0.182 &0.168 &0.232 &0.154 &0.276 \\  
         
         & 24  &\textcolor{red}{\textbf{0.086}}   &\textcolor{red}{\textbf{0.181}}  &0.107 &0.210  &\textcolor{blue}{\textbf{0.095}}  &\textcolor{blue}{\textbf{0.192}}  &0.102 &0.205  &0.118 &0.220  &0.102 &0.201 &0.117 &0.226  &0.114 &0.219  &0.115 &0.219 &0.224 &0.281 &0.248 &0.353 \\  
         
         & 48  &\textcolor{red}{\textbf{0.118}}   &\textcolor{red}{\textbf{0.203}}  &0.171  &0.265  &\textcolor{blue}{\textbf{0.148}}  &\textcolor{blue}{\textbf{0.234}}  &\textcolor{blue}{\textbf{0.148}} &0.244  &0.206   &0.293    &0.158 &0.248 &0.169 &0.268  &0.184 &0.284  &0.186 &0.235 &0.321 &0.354 &0.440 &0.470  \\  

         & 96  &\textcolor{red}{\textbf{0.171}}   &\textcolor{red}{\textbf{0.229}}   &0.250   &0.325  &0.256 &0.293 &0.277 &0.307  &0.289  &0.325  &0.366 &0.377 &0.233 &0.306  &0.309 &0.356  &\textcolor{blue}{\textbf{0.221}}  &\textcolor{blue}{\textbf{0.267}}  &0.408 &0.417 &0.674 &0.565 \\  
         
         \cmidrule(lr){2-24}
         & Avg  &\textcolor{red}{\textbf{0.111}}  &\textcolor{red}{\textbf{0.194}}  &0.151  &0.244  &\textcolor{blue}{\textbf{0.142}} &\textcolor{blue}{\textbf{0.222}}  &0.151 &0.234 &0.173  &0.255   &0.175 &0.250 &0.150 &0.246 &0.172 &0.260  &0.150 &0.226 &0.280 &0.321 &0.379 &0.416 \\ 
         \midrule[1pt]

         
         

         

        \multicolumn{2}{c|}{$1^{st}$ Count} &35  &55  &3  &0  &19  &5  &3  &0  &1  &0  &1  &0   &5  &0  &0  &0 &0  &0     &0  &0    &0  &0  \\
         \bottomrule[2pt]
    \end{tabular}
    }
    \label{tab:input_96_appendix}
\end{table*}

%% file: tables/input_96_additional.tex
\begin{table*}[!b]
    \centering
    \caption{Additional performance comparison with a fixed lookback window $T=96$. 
    The prediction lengths are set as $\tau \in \{12, 24, 48, 96\}$ for PEMS datasets for short-term forecasting, and $\tau \in \{96, 192, 336, 720\}$ for the others for long-term forecasting. 
    The best and second-best results are highlighted in \textcolor{red}{\textbf{red}} and \textcolor{blue}{\textbf{blue}}, respectively. 
    The reported results with standard deviation of Paloma are averaged from 5 runs (with different random seeds of {2024, 2025, 2026, 2027, 2028}).}
    \scalebox{0.66}{
    \begin{tabular}{l | c|c c|c c|c c| c c|c c |c c |c c }
    \toprule[2pt]
       \multicolumn{2}{c|}{Models}  &\multicolumn{2}{c|}{\textbf{Paloma}} &\multicolumn{2}{c|}{OLinear} &\multicolumn{2}{c|}{TimeMixer} &\multicolumn{2}{c|}{FITS} &\multicolumn{2}{c|}{Fredformer}  &\multicolumn{2}{c|}{FreTS}  &\multicolumn{2}{c}{TimesNet}  \\

        \multicolumn{2}{c|}{}  &\multicolumn{2}{c|}{[\textbf{ours}]} &\multicolumn{2}{c|}{[NeurIPS'25]} &\multicolumn{2}{c|}{[ICLR'24]} &\multicolumn{2}{c|}{[ICLR'24]} &\multicolumn{2}{c|}{[KDD'24]}  &\multicolumn{2}{c|}{[NeurIPS'23]}   &\multicolumn{2}{c}{[NeurIPS'23]}  \\
       
       \cmidrule(r){1-2}\cmidrule(lr){3-4}\cmidrule(lr){5-6}\cmidrule(lr){7-8}\cmidrule(lr){9-10}\cmidrule(lr){11-12}\cmidrule(lr){13-14}\cmidrule(lr){15-16}
       
       \multicolumn{2}{c|}{Metrics} &MSE &MAE &MSE &MAE &MSE &MAE &MSE &MAE &MSE &MAE &MSE &MAE &MSE &MAE   \\
    
         \midrule[1pt]
         \multirow{5}*{\rotatebox{90}{ETTm1}}
         & 96  &\textcolor{red}{\textbf{0.295}} $\pm$ 0.001    &\textcolor{red}{\textbf{0.332}} $\pm$ 0.001   &\textcolor{blue}{\textbf{0.302}}  &\textcolor{blue}{\textbf{0.334}} &0.320 &0.357 &0.353 &0.375  &0.326 &0.361  &0.339 &0.374  &0.338 &0.379 \\  
         
         & 192  &\textcolor{red}{\textbf{0.343}} $\pm$ 0.001   &\textcolor{red}{\textbf{0.358}} $\pm$ 0.001  &\textcolor{blue}{\textbf{0.357}}  &\textcolor{blue}{\textbf{0.363}} &0.361 &0.381 &0.486 &0.445  &0.363 &0.380  &0.382 &0.397  &0.389 &0.400 \\  
         
         & 336  &\textcolor{red}{\textbf{0.376}} $\pm$ 0.001   &\textcolor{red}{\textbf{0.381}} $\pm$ 0.001  &\textcolor{blue}{\textbf{0.387}}  &\textcolor{blue}{\textbf{0.385}} &0.390 &0.404 &0.531 &0.475  &0.395 &0.403  &0.421 &0.426  &0.429 &0.428 \\  
         
         & 720  &\textcolor{red}{\textbf{0.439}} $\pm$ 0.001   &\textcolor{red}{\textbf{0.417}} $\pm$ 0.002  &\textcolor{blue}{\textbf{0.452}} &\textcolor{blue}{\textbf{0.426}} &0.454 &0.441 &0.600 &0.513  &0.453 &0.438  &0.485 &0.462  &0.495 &0.464 \\  
         
         \cmidrule(lr){2-16}
         & Avg  &\textcolor{red}{\textbf{0.363}} $\pm$ 0.001   &\textcolor{red}{\textbf{0.372}} $\pm$ 0.001   &\textcolor{blue}{\textbf{0.374}}  &\textcolor{blue}{\textbf{0.377}}  &0.381 &0.395 &0.493 &0.452  &0.384 &0.395  &0.407 &0.415 &0.413 &0.418 \\  
         \midrule[1pt]

        \multirow{5}*{\rotatebox{90}{ETTm2}}
         & 96  &\textcolor{red}{\textbf{0.164}} $\pm$ 0.001   &\textcolor{red}{\textbf{0.242}} $\pm$ 0.001   &\textcolor{blue}{\textbf{0.169}}  &\textcolor{blue}{\textbf{0.249}}  &0.175 &0.258  &0.182 &0.266  &0.177 &0.259  &0.190 &0.282 &0.185 &0.264 \\  
         
         & 192  &\textcolor{red}{\textbf{0.229}} $\pm$ 0.002   &\textcolor{red}{\textbf{0.284}} $\pm$ 0.001  &\textcolor{blue}{\textbf{0.232}}  &\textcolor{blue}{\textbf{0.290}} &0.237 &0.299  &0.253 &0.312  &0.243 &0.301  &0.260 &0.329 &0.254 &0.307 \\  
         
         & 336  &\textcolor{red}{\textbf{0.289}} $\pm$ 0.001   &\textcolor{red}{\textbf{0.323}} $\pm$ 0.003  &\textcolor{blue}{\textbf{0.291}}  &\textcolor{blue}{\textbf{0.328}} &0.298 &0.340  &0.313 &0.349  &0.302 &0.340  &0.373 &0.405 &0.314 &0.345 \\ 
         
         & 720  &\textcolor{red}{\textbf{0.387}} $\pm$ 0.001   &\textcolor{red}{\textbf{0.383}} $\pm$ 0.001  &\textcolor{blue}{\textbf{0.389}}  &\textcolor{blue}{\textbf{0.387}} &0.391 &0.396  &0.416 &0.406  &0.397 &0.396  &0.517 &0.499 &0.434 &0.413 \\  
         
         \cmidrule(lr){2-16}
         & Avg  &\textcolor{red}{\textbf{0.267}} $\pm$ 0.001   & \textcolor{red}{\textbf{0.308}} $\pm$ 0.001  &\textcolor{blue}{\textbf{0.270}}  &\textcolor{blue}{\textbf{0.313}} &0.275 &0.323 &0.291 &0.333  &0.279 &0.324  &0.335 &0.379 &0.297 &0.332 \\ 
         \midrule[1pt]

         \multirow{5}*{\rotatebox{90}{ETTh1}}
         & 96  &\textcolor{red}{\textbf{0.360}} $\pm$ 0.003   &\textcolor{red}{\textbf{0.378}} $\pm$ 0.002   &\textcolor{red}{\textbf{0.360}} &\textcolor{blue}{\textbf{0.382}} &0.375 &0.400 &0.385 &0.394  &\textcolor{blue}{\textbf{0.373}} &0.392  &0.399 &0.412 &0.422 &0.433 \\ 
         
         & 192  &\textcolor{red}{\textbf{0.414}} $\pm$ 0.004   &\textcolor{red}{\textbf{0.409}} $\pm$ 0.003  &\textcolor{blue}{\textbf{0.416}}  &\textcolor{blue}{\textbf{0.414}} &0.429 &0.421 &0.434 &0.422  &0.433 &0.420  &0.453 &0.443 &0.465 &0.457 \\  
         
         & 336  &\textcolor{red}{\textbf{0.452}} $\pm$ 0.002   &\textcolor{red}{\textbf{0.428}} $\pm$ 0.005  &\textcolor{blue}{\textbf{0.457}}  &\textcolor{blue}{\textbf{0.438}}  &0.484 &0.458 &0.476 &0.444  &0.470 &0.437  &0.503 &0.475 &0.492 &0.470 \\  
         
         & 720  &\textcolor{red}{\textbf{0.458}} $\pm$ 0.006   &\textcolor{red}{\textbf{0.449}} $\pm$ 0.009  &\textcolor{blue}{\textbf{0.463}} &0.462 &0.498 &0.482 &0.465 &0.462  &0.467 &\textcolor{blue}{\textbf{0.456}}  &0.596 &0.565 &0.532 &0.502 \\  
         
         \cmidrule(lr){2-16}
         & Avg  &\textcolor{red}{\textbf{0.421}} $\pm$ 0.004   &\textcolor{red}{\textbf{0.416}} $\pm$ 0.005   &\textcolor{blue}{\textbf{0.424}}  &\textcolor{blue}{\textbf{0.424}} &0.447 &0.440 &0.440 &0.431  &0.435 &0.426  &0.488 &0.474 &0.478 &0.466 \\ 
         \midrule[1pt]

         \multirow{5}*{\rotatebox{90}{ETTh2}}
         & 96  &\textcolor{red}{\textbf{0.278}} $\pm$ 0.002    &\textcolor{red}{\textbf{0.325}} $\pm$ 0.003  &\textcolor{blue}{\textbf{0.284}}  &\textcolor{blue}{\textbf{0.329}} &0.289 &0.341 &0.292 &0.340  &0.293 &0.342  &0.350 &0.403 &0.320 &0.364 \\  
         
         & 192  &\textcolor{red}{\textbf{0.353}} $\pm$ 0.001   &\textcolor{red}{\textbf{0.376}} $\pm$ 0.002  &\textcolor{blue}{\textbf{0.360}}  &\textcolor{blue}{\textbf{0.379}} &0.372 &0.392 &0.377 &0.391  &0.371 &0.389  &0.472 &0.475 &0.409 &0.417 \\  
         
         & 336  &0.401 $\pm$ 0.003   &\textcolor{blue}{\textbf{0.413}} $\pm$ 0.002  &0.409 &0.415 &\textcolor{blue}{\textbf{0.386}} &0.414 &0.416 &0.425  &\textcolor{red}{\textbf{0.382}} &\textcolor{red}{\textbf{0.409}}  &0.564 &0.528 &0.449 &0.451 \\  
         
         & 720  &\textcolor{red}{\textbf{0.412}} $\pm$ 0.003   &\textcolor{red}{\textbf{0.429}} $\pm$ 0.004  &\textcolor{blue}{\textbf{0.415}} &\textcolor{blue}{\textbf{0.431}} &\textcolor{red}{\textbf{0.412}} &0.434 &0.418 &0.437  &\textcolor{blue}{\textbf{0.415}} &0.434  &0.815 &0.654 &0.473 &0.474 \\  
         
         \cmidrule(lr){2-16}
         & Avg  &\textcolor{red}{\textbf{0.361}} $\pm$ 0.002   &\textcolor{red}{\textbf{0.386}} $\pm$ 0.003  &0.367 &\textcolor{blue}{\textbf{0.388}} &\textcolor{blue}{\textbf{0.365}} &0.395 &0.376 &0.398  &\textcolor{blue}{\textbf{0.365}} &0.393  &0.550 &0.515 &0.413 &0.426 \\ 
         \midrule[1pt]

         \multirow{5}*{\rotatebox{90}{Electricity}}
         & 96  &\textcolor{blue}{\textbf{0.134}} $\pm$ 0.001   &\textcolor{blue}{\textbf{0.222}} $\pm$ 0.001   &\textcolor{red}{\textbf{0.131}} &\textcolor{red}{\textbf{0.221}} &0.153 &0.247  &0.198 &0.274  &0.147 &0.241  &0.189 &0.277 &0.171 &0.273 \\  
         
         & 192  &\textcolor{blue}{\textbf{0.151}} $\pm$ 0.001   &\textcolor{red}{\textbf{0.238}} $\pm$ 0.001  &\textcolor{red}{\textbf{0.150}} &\textcolor{red}{\textbf{0.238}} &0.166 &\textcolor{blue}{\textbf{0.256}}  &0.363 &0.422  &0.165 &0.258  &0.193 &0.282 &0.188 &0.289 \\  
         
         & 336  &\textcolor{blue}{\textbf{0.170}} $\pm$ 0.001   &\textcolor{blue}{\textbf{0.256}} $\pm$ 0.001  &\textcolor{red}{\textbf{0.165}} &\textcolor{red}{\textbf{0.254}} &0.185 &0.277  &0.444 &0.490  &0.177 &0.273  &0.207 &0.296 &0.208 &0.304 \\  
         
         & 720  &\textcolor{blue}{\textbf{0.208}} $\pm$ 0.001   &\textcolor{blue}{\textbf{0.289}} $\pm$ 0.002  &\textcolor{red}{\textbf{0.191}} &\textcolor{red}{\textbf{0.279}} &0.225 &0.310  &0.532 &0.551  &0.213 &0.304  &0.245 &0.332 &0.289 &0.363 \\  
         
         \cmidrule(lr){2-16}
         & Avg  &\textcolor{blue}{\textbf{0.166}} $\pm$ 0.001  &\textcolor{blue}{\textbf{0.251}} $\pm$ 0.001   &\textcolor{red}{\textbf{0.159}} &\textcolor{red}{\textbf{0.248}} &0.182 &0.273 &0.384 &0.434  &0.176 &0.269 &0.209 &0.297 &0.214 &0.307 \\ 
         \midrule[1pt]

         \multirow{5}*{\rotatebox{90}{Traffic}}
         & 96  &0.425 $\pm$ 0.001   &\textcolor{blue}{\textbf{0.243}} $\pm$ 0.001   &\textcolor{red}{\textbf{0.398}} &\textcolor{red}{\textbf{0.226}} &0.462 &0.285  &0.601 &0.361  &\textcolor{blue}{\textbf{0.406}} &0.277  &0.528 &0.341 &0.504 &0.298 \\  
         
         & 192  &0.442 $\pm$ 0.002   &\textcolor{blue}{\textbf{0.254}} $\pm$ 0.001  &\textcolor{blue}{\textbf{0.439}} &\textcolor{red}{\textbf{0.241}} &0.473 &0.296  &0.603 &0.365  &\textcolor{red}{\textbf{0.426}} &0.290  &0.531 &0.338 &0.526 &0.305 \\  
         
         & 336  &\textcolor{blue}{\textbf{0.462}} $\pm$ 0.001   &\textcolor{blue}{\textbf{0.262}} $\pm$ 0.001  &0.464 &\textcolor{red}{\textbf{0.250}} &0.498 &0.296  &0.609 &0.366  &\textcolor{red}{\textbf{0.437}} &0.292  &0.551 &0.345 &0.540 &0.310 \\  
         
         & 720  &0.514 $\pm$ 0.000   &\textcolor{blue}{\textbf{0.279}} $\pm$ 0.001  &\textcolor{blue}{\textbf{0.502}} &\textcolor{red}{\textbf{0.270}} &0.506 &0.313  &0.648 &0.387  &\textcolor{red}{\textbf{0.462}} &0.305  &0.598 &0.367 &0.570 &0.324 \\  
         
         \cmidrule(lr){2-16}
         & Avg  &0.461 $\pm$ 0.001   &\textcolor{blue}{\textbf{0.260}} $\pm$ 0.001  &\textcolor{blue}{\textbf{0.451}} &\textcolor{red}{\textbf{0.247}} &0.485 &0.298 &0.615 &0.370  &\textcolor{red}{\textbf{0.433}} &0.291  &0.552 &0.348 &0.535 &0.309 \\ 
         \midrule[1pt] 

         \multirow{5}*{\rotatebox{90}{Weather}}
         & 96  &\textcolor{blue}{\textbf{0.154}} $\pm$ 0.001   &\textcolor{red}{\textbf{0.190}} $\pm$ 0.001   &\textcolor{red}{\textbf{0.153}} &\textcolor{red}{\textbf{0.190}} &0.163 &0.209  &0.196 &0.236  &0.163 &\textcolor{blue}{\textbf{0.207}}  &0.184 &0.239 &0.178 &0.226 \\  
         
         & 192  &\textcolor{red}{\textbf{0.195}} $\pm$ 0.001   &\textcolor{blue}{\textbf{0.238}} $\pm$ 0.001  &\textcolor{blue}{\textbf{0.200}} &\textcolor{red}{\textbf{0.235}} &0.208 &0.250  &0.240 &0.271  &0.211 &0.251  &0.223 &0.275 &0.227 &0.266 \\  
         
         & 336  &\textcolor{red}{\textbf{0.249}} $\pm$ 0.001   &\textcolor{blue}{\textbf{0.284}} $\pm$ 0.001  &0.258 &\textcolor{red}{\textbf{0.280}} &\textcolor{blue}{\textbf{0.251}} &0.287  &0.292 &0.307  &0.267 &0.292  &0.272 &0.316 &0.283 &0.305 \\  
         
         & 720  &\textcolor{red}{\textbf{0.324}} $\pm$ 0.001   &\textcolor{blue}{\textbf{0.336}} $\pm$ 0.001  &\textcolor{blue}{\textbf{0.337}} &\textcolor{red}{\textbf{0.333}} &0.339 &0.341  &0.365 &0.354  &0.343 &0.341  &0.340 &0.363 &0.359 &0.355 \\  
         
         \cmidrule(lr){2-16} 
         & Avg  &\textcolor{red}{\textbf{0.231}} $\pm$ 0.001   & \textcolor{blue}{\textbf{0.262}} $\pm$ 0.001  &\textcolor{blue}{\textbf{0.237}} &\textcolor{red}{\textbf{0.260}} &0.240 &0.272  &0.273 &0.292  &0.246 &0.272  &0.255 &0.299 &0.262 &0.288 \\ 
         \midrule[1pt]

        \multirow{5}*{\rotatebox{90}{Solar-Energy}}
         & 96  &\textcolor{red}{\textbf{0.175}} $\pm$ 0.003   &\textcolor{blue}{\textbf{0.224}} $\pm$ 0.002   &\textcolor{blue}{\textbf{0.179}} &\textcolor{red}{\textbf{0.191}} &0.189 &0.259  &0.319 &0.353  &0.185 &0.233   &0.237   &0.300   &0.250 &0.292 \\  
         
         & 192  &\textcolor{red}{\textbf{0.192}} $\pm$ 0.002   &\textcolor{blue}{\textbf{0.240}} $\pm$ 0.002  &\textcolor{blue}{\textbf{0.209}} &\textcolor{red}{\textbf{0.213}} &0.222 &0.283  &0.367 &0.387  &0.227 &0.253   &0.265   &0.321  &0.296 &0.318 \\ 
         
         & 336  &\textcolor{red}{\textbf{0.204}} $\pm$ 0.005   &\textcolor{blue}{\textbf{0.235}} $\pm$ 0.008  &\textcolor{blue}{\textbf{0.231}} &\textcolor{red}{\textbf{0.229}} &\textcolor{blue}{\textbf{0.231}} &0.292  &0.408 &0.403  &0.246 &0.284   &0.283   &0.330  &0.319 &0.330 \\ 
         
         & 720  &\textcolor{red}{\textbf{0.215}}  $\pm$ 0.006  &\textcolor{blue}{\textbf{0.238}} $\pm$ 0.004  &\textcolor{blue}{\textbf{0.241}} &\textcolor{red}{\textbf{0.236}} &0.223 &0.285  &0.411 &0.395  &0.247 &0.276   &0.286   &0.326  &0.338 &0.337 \\ 
         
         \cmidrule(lr){2-16}
         & Avg  &\textcolor{red}{\textbf{0.197}} $\pm$ 0.004  &\textcolor{blue}{\textbf{0.234}} $\pm$ 0.004  &\textcolor{blue}{\textbf{0.215}} &\textcolor{red}{\textbf{0.217}} &0.216 &0.280  &0.376 &0.384  &0.226 &0.262   &0.268   &0.322  &0.301 &0.319 \\ 
         \midrule[1pt]

        \multirow{5}*{\rotatebox{90}{PEMS03}}
         & 12  &\textcolor{red}{\textbf{0.058}} $\pm$ 0.001   &\textcolor{red}{\textbf{0.158}} $\pm$ 0.001  &\textcolor{blue}{\textbf{0.060}} &\textcolor{blue}{\textbf{0.159}} &0.076 &0.188  &0.117 &0.226  &0.068 &0.174  &0.083 &0.194  &0.085 &0.192 \\  
         
         & 24  &\textcolor{red}{\textbf{0.074}} $\pm$ 0.001  &\textcolor{red}{\textbf{0.175}} $\pm$ 0.001  &\textcolor{blue}{\textbf{0.078}} &\textcolor{blue}{\textbf{0.179}} &0.113 &0.226  &0.235 &0.324  &0.093 &0.202  &0.127 &0.241  &0.118 &0.223 \\ 
         
         & 48  &\textcolor{red}{\textbf{0.101}} $\pm$ 0.002   & \textcolor{red}{\textbf{0.205}} $\pm$ 0.001  &\textcolor{blue}{\textbf{0.104}} &\textcolor{blue}{\textbf{0.210}} &0.191 &0.292  &0.541 &0.521  &0.146 &0.258   &0.204  &0.311   &0.155 &0.260 \\ 
         
         & 96  &\textcolor{red}{\textbf{0.142}} $\pm$ 0.001   & \textcolor{red}{\textbf{0.236}} $\pm$ 0.003  &\textcolor{blue}{\textbf{0.140}} &\textcolor{blue}{\textbf{0.247}} &0.288 &0.363  &1.062 &0.790  &0.228 &0.330  &0.255   &0.345   &0.228 &0.317 \\ 
         
         \cmidrule(lr){2-16}
         & Avg  &\textcolor{red}{\textbf{0.094}} $\pm$ 0.001   &\textcolor{red}{\textbf{0.194}} $\pm$ 0.002  &\textcolor{blue}{\textbf{0.095}} &\textcolor{blue}{\textbf{0.199}} &0.167 &0.267 &0.489 &0.465  &0.135 &0.243  &0.167  &0.273   &0.147 &0.248 \\ 
         \midrule[1pt]

         
         
         
         

         \multirow{5}*{\rotatebox{90}{PEMS04}}
         & 12  &\textcolor{red}{\textbf{0.065}} $\pm$ 0.001   &\textcolor{red}{\textbf{0.160}} $\pm$ 0.001   &\textcolor{blue}{\textbf{0.068}} &\textcolor{blue}{\textbf{0.163}} &0.092 &0.204  &0.129 &0.239  &0.085 &0.189   &0.098   &0.210  &0.087 &0.195 \\  
         
         & 24  &\textcolor{red}{\textbf{0.074}} $\pm$ 0.001   &\textcolor{red}{\textbf{0.171}} $\pm$ 0.001   &\textcolor{blue}{\textbf{0.079}} &\textcolor{blue}{\textbf{0.176}} &0.128 &0.243  &0.246 &0.337  &0.117 &0.224   &0.148   &0.263  &0.103 &0.215 \\  
         
         & 48  &\textcolor{red}{\textbf{0.090}} $\pm$ 0.001   &\textcolor{red}{\textbf{0.189}} $\pm$ 0.001   &\textcolor{blue}{\textbf{0.095}} &\textcolor{blue}{\textbf{0.197}} &0.213 &0.315  &0.568 &0.539  &0.174 &0.276   &0.235   &0.341  &0.136 &0.250 \\  
         
         & 96  &\textcolor{red}{\textbf{0.110}} $\pm$ 0.001   &\textcolor{red}{\textbf{0.207}} $\pm$ 0.001   &\textcolor{blue}{\textbf{0.122}} &\textcolor{blue}{\textbf{0.226}} &0.307 &0.384  &1.181 &0.843  &0.273 &0.354   &0.304   &0.393  &0.190 &0.303 \\  
         
         \cmidrule(lr){2-16}
         & Avg  &\textcolor{red}{\textbf{0.085}} $\pm$ 0.001   & \textcolor{red}{\textbf{0.182}} $\pm$ 0.001   &\textcolor{blue}{\textbf{0.091}} &\textcolor{blue}{\textbf{0.190}} &0.185 &0.287  &0.531 &0.489  &0.162 &0.261   &0.196   &0.302  &0.129 &0.241 \\ 
         \midrule[1pt] 

         \multirow{5}*{\rotatebox{90}{PEMS07}}
         & 12  &\textcolor{red}{\textbf{0.051}} $\pm$ 0.002   &\textcolor{red}{\textbf{0.138}} $\pm$ 0.001  &\textcolor{blue}{\textbf{0.052}} &\textcolor{red}{\textbf{0.138}} &0.073 &0.184  &0.109 &0.222  &0.063 &\textcolor{blue}{\textbf{0.158}}   &0.080  &0.190  &0.082 &0.181 \\  
         
         & 24  &\textcolor{red}{\textbf{0.064}} $\pm$ 0.002   & \textcolor{blue}{\textbf{0.152}} $\pm$ 0.002  &\textcolor{blue}{\textbf{0.065}} &\textcolor{red}{\textbf{0.151}} &0.111 &0.219  &0.230 &0.327  &0.089 &0.192   &0.136   &0.250  &0.101 &0.204 \\ 
         
         & 48  &\textcolor{red}{\textbf{0.084}} $\pm$ 0.001   &\textcolor{blue}{\textbf{0.172}} $\pm$ 0.001  &\textcolor{red}{\textbf{0.084}} &\textcolor{red}{\textbf{0.171}} &0.237 &0.328  &0.551 &0.531  &0.136 &0.241   &0.248   &0.340  &\textcolor{blue}{\textbf{0.134}} &0.238 \\ 
         
         & 96  &\textcolor{blue}{\textbf{0.112}} $\pm$ 0.001   & \textcolor{red}{\textbf{0.194}} $\pm$ 0.001  &\textcolor{red}{\textbf{0.108}} &\textcolor{blue}{\textbf{0.196}} &0.303 &0.354  &1.112 &0.809  &0.197 &0.298  &0.352   &0.413   &0.181 &0.279 \\ 
         
         \cmidrule(lr){2-16}
         & Avg  &\textcolor{blue}{\textbf{0.078}} $\pm$ 0.002   &\textcolor{red}{\textbf{0.164}} $\pm$ 0.001  &\textcolor{red}{\textbf{0.077}} &\textcolor{red}{\textbf{0.164}} &0.181 &0.271  &0.500 &0.472  &0.121 &\textcolor{blue}{\textbf{0.222}}   &0.204   &0.298  &0.124 &0.225 \\ 
         \midrule[1pt]

        \multirow{5}*{\rotatebox{90}{PEMS08}}
         & 12  &\textcolor{blue}{\textbf{0.069}} $\pm$ 0.000   &\textcolor{blue}{\textbf{0.163}} $\pm$ 0.000  &\textcolor{red}{\textbf{0.068}} &\textcolor{red}{\textbf{0.159}} &0.091 &0.201  &0.122 &0.233  &0.081 &0.185  &0.095 &0.204  &0.112 &0.212 \\  
         
         & 24  &\textcolor{red}{\textbf{0.086}} $\pm$ 0.001   &\textcolor{blue}{\textbf{0.181}} $\pm$ 0.001  &\textcolor{blue}{\textbf{0.089}} &\textcolor{red}{\textbf{0.178}} &0.137 &0.246  &0.236 &0.330  &0.112 &0.214  &0.150 &0.259  &0.141 &0.238 \\  
         
         & 48  &\textcolor{red}{\textbf{0.118}} $\pm$ 0.000   &\textcolor{red}{\textbf{0.203}} $\pm$ 0.001  &\textcolor{blue}{\textbf{0.123}} &\textcolor{blue}{\textbf{0.204}} &0.265 &0.343  &0.562 &0.540  &0.174 &0.267  &0.250   &0.341   &0.198 &0.283 \\  
         
         & 96  &\textcolor{red}{\textbf{0.171}} $\pm$ 0.001   &\textcolor{red}{\textbf{0.229}} $\pm$ 0.001  &\textcolor{blue}{\textbf{0.173}} &\textcolor{blue}{\textbf{0.236}} &0.410 &0.407  &1.216 &0.846  &0.277 &0.335  &0.323   &0.376  &0.320 &0.351 \\  
         
         \cmidrule(lr){2-16}
         & Avg  &\textcolor{red}{\textbf{0.111}} $\pm$ 0.001  &\textcolor{red}{\textbf{0.194}} $\pm$ 0.001  &\textcolor{blue}{\textbf{0.113}} &\textcolor{red}{\textbf{0.194}} &0.226 &0.299 &0.534 &0.487  &0.161 &\textcolor{blue}{\textbf{0.250}}  &0.205 &0.295  &0.193 &0.271 \\ 
         \midrule[1pt]

         
         
         
         

        \multicolumn{2}{c|}{$1^{st}$ Count} &45  &37  &12  &27  &1  &0  &0  &0  &5  &1  &0  &0 &0  &0   \\
         \bottomrule[2pt]
    \end{tabular}
    }
    \label{tab:input_96_additional}
\end{table*}

%% file: tables/input_336_appendix.tex
\begin{table*}[h]
    \centering
    \caption{Performance comparison with longer lookback windows $T \in \{336, 720 \}$. Highlight the best results in \textcolor{red}{\textbf{red}} and the second-best results in \textcolor{blue}{\textbf{blue}}. The baseline results follow Table 11 in GTR~\cite{cao2026enhancing}.}
        
    \scalebox{0.65}{
    \begin{tabular}{l | c|c c|c c|c c| c c||c c |c c |c c |c c   }
    \toprule[2pt]
        \multicolumn{2}{c|}{LookBack} &\multicolumn{8}{c||}{$T=336$} &\multicolumn{8}{c}{$T=720$} \\
         \midrule[1pt]

       \multicolumn{2}{c|}{Models}  &\multicolumn{2}{c|}{PatchTST} &\multicolumn{2}{c|}{CycleNet} &\multicolumn{2}{c|}{GTR} &\multicolumn{2}{c||}{\textbf{Paloma}} &\multicolumn{2}{c|}{SparseTSF} &\multicolumn{2}{c|}{CycleNet}  &\multicolumn{2}{c|}{GTR}  &\multicolumn{2}{c}{\textbf{Paloma}} \\

        \multicolumn{2}{c|}{}  &\multicolumn{2}{c|}{[ICLR'23]} &\multicolumn{2}{c|}{[NeurIPS'24]} &\multicolumn{2}{c|}{[ICLR'26]} &\multicolumn{2}{c||}{[\textbf{ours}]}  &\multicolumn{2}{c|}{[ICML'25]}  &\multicolumn{2}{c|}{[NeurIPS'24]} &\multicolumn{2}{c|}{[ICLR'26]} &\multicolumn{2}{c}{[\textbf{ours}]}   \\
         \midrule[1pt]
       
       
       \multicolumn{2}{c|}{Metrics} &MSE &MAE &MSE &MAE &MSE &MAE &MSE &MAE &MSE &MAE &MSE &MAE &MSE &MAE &MSE &MAE    \\
    
         \midrule[1pt]
         \multirow{5}*{\rotatebox{90}{ETTh1}}
         & 96  &0.385 &0.405 &0.374 &0.396 &\textcolor{blue}{\textbf{0.369}} &\textcolor{blue}{\textbf{0.395}}  &\textcolor{red}{\textbf{0.355}}  &\textcolor{red}{\textbf{0.380}}   &\textcolor{blue}{\textbf{0.362}} &\textcolor{red}{\textbf{0.388}} &0.379 &0.403 &0.368 &0.400  &\textcolor{red}{\textbf{0.359}}  &\textcolor{blue}{\textbf{0.389}} \\
         
         & 192  &0.414 &0.421 &\textcolor{blue}{\textbf{0.406}} &\textcolor{blue}{\textbf{0.415}} &0.409 &0.420   &\textcolor{red}{\textbf{0.399}}  &\textcolor{red}{\textbf{0.408}}  &\textcolor{red}{\textbf{0.403}} &\textcolor{red}{\textbf{0.411}} &0.416 &0.425 &0.409 &0.424  &\textcolor{blue}{\textbf{0.406}}  &\textcolor{blue}{\textbf{0.418}} \\
         
         & 336  &0.440 &0.440 &\textcolor{blue}{\textbf{0.431}} &\textcolor{blue}{\textbf{0.430}} &0.432 &0.434   &\textcolor{red}{\textbf{0.428}}  &\textcolor{red}{\textbf{0.425}}  &\textcolor{red}{\textbf{0.434}} &\textcolor{red}{\textbf{0.428}} &0.447 &0.445 &0.443 &0.449  &\textcolor{blue}{\textbf{0.439}}  &\textcolor{blue}{\textbf{0.438}} \\
         
         & 720  &0.456 &0.470 &\textcolor{blue}{\textbf{0.450}} &0.464 &0.451 &\textcolor{blue}{\textbf{0.463}}   &\textcolor{red}{\textbf{0.434}}  &\textcolor{red}{\textbf{0.449}}  &\textcolor{red}{\textbf{0.426}} &\textcolor{red}{\textbf{0.447}} &0.477 &0.483 &0.464 &0.476  &\textcolor{blue}{\textbf{0.442}}  &\textcolor{blue}{\textbf{0.460}} \\
         
         \cmidrule(r){2-18}
         & Avg  &0.424 &0.434 &\textcolor{blue}{\textbf{0.415}} &\textcolor{blue}{\textbf{0.426}} &\textcolor{blue}{\textbf{0.415}} &0.428   &\textcolor{red}{\textbf{0.404}}  &\textcolor{red}{\textbf{0.416}}   &\textcolor{red}{\textbf{0.406}} &\textcolor{red}{\textbf{0.419}} &0.430 &0.439 &0.421 &0.437  &\textcolor{blue}{\textbf{0.412}}  &\textcolor{blue}{\textbf{0.426}} \\
         \midrule[1pt]

        \multirow{5}*{\rotatebox{90}{ETTh2}}
         & 96  &0.275 &\textcolor{blue}{\textbf{0.337}} &0.279 &0.341 &\textcolor{blue}{\textbf{0.274}} &0.341  &\textcolor{red}{\textbf{0.273}}  &\textcolor{red}{\textbf{0.331}}   &0.294 &0.346 &\textcolor{blue}{\textbf{0.271}}  &\textcolor{blue}{\textbf{0.337}} &0.275 &0.343  &\textcolor{red}{\textbf{0.269}}  &\textcolor{red}{\textbf{0.328}} \\
         
         & 192  &\textcolor{red}{\textbf{0.338}} &\textcolor{blue}{\textbf{0.379}} &0.342 &0.385 &\textcolor{blue}{\textbf{0.340}} &0.385   &\textcolor{blue}{\textbf{0.340}}  &\textcolor{red}{\textbf{0.374}}  &0.339 &\textcolor{blue}{\textbf{0.377}} &\textcolor{red}{\textbf{0.332}} &0.380 &0.340 &0.388  &\textcolor{blue}{\textbf{0.336}}  &\textcolor{red}{\textbf{0.372}} \\
         
         & 336  &\textcolor{red}{\textbf{0.365}} &\textcolor{red}{\textbf{0.398}} &\textcolor{blue}{\textbf{0.371}} &0.413 &0.363 &0.409   &0.374  &\textcolor{blue}{\textbf{0.404}}  &\textcolor{red}{\textbf{0.359}} &\textcolor{red}{\textbf{0.397}} &\textcolor{blue}{\textbf{0.362}} &0.408 &0.377 &0.416  &0.367  &\textcolor{blue}{\textbf{0.398}} \\
         
         & 720  &\textcolor{red}{\textbf{0.391}} &\textcolor{blue}{\textbf{0.429}} &0.426 &0.451 &0.405 &0.439   &\textcolor{blue}{\textbf{0.397}}  &\textcolor{red}{\textbf{0.427}}  &\textcolor{red}{\textbf{0.383}} &\textcolor{blue}{\textbf{0.424}} &0.415 &0.449 &0.406 &0.442  &\textcolor{blue}{\textbf{0.386}}  &\textcolor{red}{\textbf{0.420}} \\
         
         \cmidrule(r){2-18}
         & Avg  &\textcolor{red}{\textbf{0.342}} &\textcolor{blue}{\textbf{0.386}} &0.355 &0.398 &\textcolor{blue}{\textbf{0.345}} &0.393   &0.346  &\textcolor{red}{\textbf{0.384}}  &\textcolor{blue}{\textbf{0.344}} &\textcolor{blue}{\textbf{0.386}} &0.345 &0.394 &0.349 &0.397  &\textcolor{red}{\textbf{0.340}}  &\textcolor{red}{\textbf{0.380}} \\
         \midrule[1pt]

         \multirow{5}*{\rotatebox{90}{ETTm1}}
         & 96   &0.291 &0.343 &0.299 &0.348 &\textcolor{blue}{\textbf{0.283}} &\textcolor{blue}{\textbf{0.337}}   &\textcolor{red}{\textbf{0.276}}  &\textcolor{red}{\textbf{0.325}}  &0.312 &0.354 &0.307 &0.353 &\textcolor{blue}{\textbf{0.290}}  &\textcolor{blue}{\textbf{0.345}}  &\textcolor{red}{\textbf{0.273}}  &\textcolor{red}{\textbf{0.325}} \\
         
         & 192  &0.334 &0.370 &0.334 &0.367 &\textcolor{blue}{\textbf{0.323}}  &\textcolor{blue}{\textbf{0.365}}   &\textcolor{red}{\textbf{0.313}}  &\textcolor{red}{\textbf{0.349}}  &0.347 &0.376 &0.337 &\textcolor{blue}{\textbf{0.371}} &\textcolor{blue}{\textbf{0.328}} &\textcolor{blue}{\textbf{0.371}}  &\textcolor{red}{\textbf{0.313}}  &\textcolor{red}{\textbf{0.352}} \\
         
         & 336  &0.367 &0.392 &0.368 &0.386 &\textcolor{blue}{\textbf{0.355}}  &\textcolor{blue}{\textbf{0.384}}   &\textcolor{red}{\textbf{0.345}}  &\textcolor{red}{\textbf{0.370}}  &0.367 &\textcolor{blue}{\textbf{0.386}} &0.364 &0.387 &\textcolor{blue}{\textbf{0.357}} &0.388  &\textcolor{red}{\textbf{0.344}}  &\textcolor{red}{\textbf{0.374}} \\
         
         & 720  &0.422 &0.426 &\textcolor{blue}{\textbf{0.417}} &\textcolor{blue}{\textbf{0.414}} &0.420 &0.422   &\textcolor{red}{\textbf{0.404}}  &\textcolor{red}{\textbf{0.403}}  &0.419 &0.413 &\textcolor{blue}{\textbf{0.410}} &\textcolor{blue}{\textbf{0.411}} &0.417 &0.420  &\textcolor{red}{\textbf{0.407}}  &\textcolor{red}{\textbf{0.406}} \\
         
         \cmidrule(r){2-18}
         & Avg  &0.354 &0.383 &0.355 &0.379 &\textcolor{blue}{\textbf{0.345}} &\textcolor{blue}{\textbf{0.377}}   &\textcolor{red}{\textbf{0.335}}  &\textcolor{red}{\textbf{0.362}}  &0.361 &0.382 &0.355 &\textcolor{blue}{\textbf{0.381}} &\textcolor{blue}{\textbf{0.348}} &\textcolor{blue}{\textbf{0.381}}  &\textcolor{red}{\textbf{0.334}}  &\textcolor{red}{\textbf{0.364}} \\
         \midrule[1pt]

         \multirow{5}*{\rotatebox{90}{ETTm2}}
         & 96   &0.164 &0.254 &\textcolor{blue}{\textbf{0.159}} &\textcolor{blue}{\textbf{0.247}} &0.161 &0.252   &\textcolor{red}{\textbf{0.158}}  &\textcolor{red}{\textbf{0.240}}  &0.163 &0.252 &\textcolor{blue}{\textbf{0.159}} &\textcolor{blue}{\textbf{0.249}} &0.169 &0.258  &\textcolor{red}{\textbf{0.157}}  &\textcolor{red}{\textbf{0.239}} \\
         
         & 192  &0.221 &0.293 &\textcolor{red}{\textbf{0.214}} &\textcolor{blue}{\textbf{0.286}} &0.221 &0.292   &\textcolor{blue}{\textbf{0.216}}  &\textcolor{red}{\textbf{0.281}}  &\textcolor{blue}{\textbf{0.217}} &0.290 &\textcolor{red}{\textbf{0.214}} &\textcolor{blue}{\textbf{0.289}} &0.220 &0.295  &\textcolor{red}{\textbf{0.214}}  &\textcolor{red}{\textbf{0.282}} \\
         
         & 336  &0.276 &0.328 &\textcolor{red}{\textbf{0.269}} &\textcolor{blue}{\textbf{0.322}} &0.278 &0.326   &\textcolor{blue}{\textbf{0.270}}  &\textcolor{red}{\textbf{0.316}}  &0.270 &0.327 &\textcolor{blue}{\textbf{0.268}} &\textcolor{blue}{\textbf{0.326}} &0.270 &0.328  &\textcolor{red}{\textbf{0.263}}  &\textcolor{red}{\textbf{0.316}} \\
         
         & 720  &0.366 &0.383 &0.363 &\textcolor{blue}{\textbf{0.382}} &\textcolor{blue}{\textbf{0.359}} &0.388   &\textcolor{red}{\textbf{0.356}}  &\textcolor{red}{\textbf{0.372}}  &\textcolor{red}{\textbf{0.352}} &\textcolor{blue}{\textbf{0.379}} &\textcolor{blue}{\textbf{0.353}} &0.384 &0.355 &0.389  &\textcolor{red}{\textbf{0.352}}  &\textcolor{red}{\textbf{0.375}} \\
         
         \cmidrule(r){2-18}
         & Avg  &0.257 &0.315 &\textcolor{blue}{\textbf{0.251}} &\textcolor{blue}{\textbf{0.309}} &0.254 &0.314   &\textcolor{red}{\textbf{0.250}}  &\textcolor{red}{\textbf{0.303}}  &0.251 &\textcolor{blue}{\textbf{0.312}} &\textcolor{blue}{\textbf{0.249}} &\textcolor{blue}{\textbf{0.312}} &0.253 &0.317  &\textcolor{red}{\textbf{0.247}}  &\textcolor{red}{\textbf{0.304}} \\
         \midrule[1pt]

         \multirow{5}*{\rotatebox{90}{Electricity}}
         & 96   &0.131 &0.225 &\textcolor{blue}{\textbf{0.128}} &0.223 &\textcolor{red}{\textbf{0.127}} &\textcolor{blue}{\textbf{0.222}}   &\textcolor{red}{\textbf{0.127}}  &\textcolor{red}{\textbf{0.216}}  &\textcolor{blue}{\textbf{0.138}} &0.233 &\textcolor{red}{\textbf{0.128}} &\textcolor{blue}{\textbf{0.223}} &\textcolor{red}{\textbf{0.128}} &0.226  &\textcolor{red}{\textbf{0.128}}  &\textcolor{red}{\textbf{0.220}} \\
         
         & 192  &0.148 &0.240 &\textcolor{red}{\textbf{0.144}} &\textcolor{blue}{\textbf{0.237}} &\textcolor{blue}{\textbf{0.147}} &0.242   &\textcolor{blue}{\textbf{0.147}}  &\textcolor{red}{\textbf{0.234}}  &0.151 &\textcolor{blue}{\textbf{0.244}} &\textcolor{red}{\textbf{0.143}} &\textcolor{red}{\textbf{0.237}} &\textcolor{blue}{\textbf{0.146}} &\textcolor{blue}{\textbf{0.244}}  &0.148  &\textcolor{red}{\textbf{0.237}} \\
         
         & 336  &\textcolor{blue}{\textbf{0.165}} &0.259 &\textcolor{red}{\textbf{0.160}} &\textcolor{blue}{\textbf{0.254}} &\textcolor{blue}{\textbf{0.165}} &0.260   &0.166  &\textcolor{red}{\textbf{0.252}}  &0.166 &\textcolor{blue}{\textbf{0.260}} &\textcolor{red}{\textbf{0.159}} &\textcolor{red}{\textbf{0.254}} &\textcolor{blue}{\textbf{0.162}} &0.261  &0.163  &\textcolor{red}{\textbf{0.254}} \\
         
         & 720  &0.202 &0.291 &\textcolor{red}{\textbf{0.198}} &\textcolor{blue}{\textbf{0.287}} &\textcolor{blue}{\textbf{0.199}} &0.292   &0.202  &\textcolor{red}{\textbf{0.284}}  &0.205 &0.293 &\textcolor{red}{\textbf{0.197}} &\textcolor{blue}{\textbf{0.287}} &\textcolor{red}{\textbf{0.197}} &0.295  &\textcolor{blue}{\textbf{0.199}}  &\textcolor{red}{\textbf{0.285}} \\
         
         \cmidrule(r){2-18}
         & Avg  &0.162 &0.254 &\textcolor{red}{\textbf{0.158}} &\textcolor{blue}{\textbf{0.250}} &\textcolor{blue}{\textbf{0.159}} &0.254   &0.161  &\textcolor{red}{\textbf{0.247}}  &0.165 &0.258 &\textcolor{red}{\textbf{0.157}} &\textcolor{blue}{\textbf{0.250}} &\textcolor{blue}{\textbf{0.158}} &0.256  &0.160  &\textcolor{red}{\textbf{0.249}} \\
         \midrule[1pt]

         \multirow{5}*{\rotatebox{90}{Solar-Energy}}
         & 96   &\textcolor{blue}{\textbf{0.190}}  &0.278 &0.200 &0.250 &0.180 &\textcolor{blue}{\textbf{0.237}}   &\textcolor{red}{\textbf{0.174}}  &\textcolor{red}{\textbf{0.214}}  &0.195 &0.243 &0.194 &0.255 &\textcolor{blue}{\textbf{0.178}} &\textcolor{blue}{\textbf{0.242}}  &\textcolor{red}{\textbf{0.175}}  &\textcolor{red}{\textbf{0.206}} \\
         
         & 192  &0.206 &\textcolor{blue}{\textbf{0.252}}  &0.221 &0.261 &\textcolor{blue}{\textbf{0.196}} &0.254   &\textcolor{red}{\textbf{0.190}}  &\textcolor{red}{\textbf{0.230}}  &0.215 &0.254 &0.205 &\textcolor{blue}{\textbf{0.251}} &\textcolor{blue}{\textbf{0.198}} &0.260  &\textcolor{red}{\textbf{0.188}}  &\textcolor{red}{\textbf{0.216}} \\
         
         & 336  &0.217 &\textcolor{blue}{\textbf{0.254}} &0.236 &0.272 &\textcolor{blue}{\textbf{0.201}} &0.257   &\textcolor{red}{\textbf{0.198}}  &\textcolor{red}{\textbf{0.225}}  &0.232 &0.262 &0.218 &\textcolor{blue}{\textbf{0.257}} &\textcolor{blue}{\textbf{0.202}} &0.264  &\textcolor{red}{\textbf{0.201}}  &\textcolor{red}{\textbf{0.225}} \\
         
         & 720  &0.219 &\textcolor{blue}{\textbf{0.255}} &0.245 &0.277 &\textcolor{red}{\textbf{0.205}} &0.258   &\textcolor{blue}{\textbf{0.208}}  &\textcolor{red}{\textbf{0.221}}  &0.237 &\textcolor{blue}{\textbf{0.263}} &0.239 &0.278 &\textcolor{blue}{\textbf{0.209}} &0.269  &\textcolor{red}{\textbf{0.207}}  &\textcolor{red}{\textbf{0.226}} \\
         
         \cmidrule(r){2-18}
         & Avg  &0.208 &0.260 &0.226 &0.265 &\textcolor{blue}{\textbf{0.196}}  &\textcolor{blue}{\textbf{0.252}}   &\textcolor{red}{\textbf{0.193}}  &\textcolor{red}{\textbf{0.223}}  &0.220 &\textcolor{blue}{\textbf{0.256}} &0.214 &0.260 &\textcolor{blue}{\textbf{0.197}} &0.259  &\textcolor{red}{\textbf{0.193}}  &\textcolor{red}{\textbf{0.218}} \\
         \midrule[1pt]

         \multirow{5}*{\rotatebox{90}{Weather}}
         & 96   &0.155 &0.204 &0.167 &0.221 &\textcolor{blue}{\textbf{0.147}} &\textcolor{blue}{\textbf{0.199}}   &\textcolor{red}{\textbf{0.146}}  &\textcolor{red}{\textbf{0.185}}  &0.169 &0.223 &0.164 &0.220 &\textcolor{blue}{\textbf{0.146}} &\textcolor{blue}{\textbf{0.200}}  &\textcolor{red}{\textbf{0.144}}  &\textcolor{red}{\textbf{0.189}} \\
         
         & 192  &0.195 &0.242 &0.212 &0.258 &\textcolor{blue}{\textbf{0.192}} &\textcolor{blue}{\textbf{0.240}}   &\textcolor{red}{\textbf{0.188}}  &\textcolor{red}{\textbf{0.233}}  &0.214 &0.262 &0.209 &0.258 &\textcolor{blue}{\textbf{0.192}} &\textcolor{blue}{\textbf{0.244}}  &\textcolor{blue}{\textbf{0.190}}  &\textcolor{blue}{\textbf{0.236}} \\
         
         & 336  &0.249 &0.283 &0.260 &0.293 &\textcolor{blue}{\textbf{0.243}} &\textcolor{blue}{\textbf{0.281}}   &\textcolor{red}{\textbf{0.237}}  &\textcolor{red}{\textbf{0.276}}  &0.257 &0.293 &0.255 &0.292 &\textcolor{blue}{\textbf{0.243}} &\textcolor{blue}{\textbf{0.283}}  &\textcolor{red}{\textbf{0.238}}  &\textcolor{red}{\textbf{0.275}} \\
         
         & 720  &\textcolor{blue}{\textbf{0.321}}  &\textcolor{blue}{\textbf{0.334}} &0.328 &0.339 &0.324 &0.337   &\textcolor{red}{\textbf{0.313}}  &\textcolor{red}{\textbf{0.329}}  &0.321 &0.340 &0.320 &0.338 &\textcolor{blue}{\textbf{0.317}} &\textcolor{blue}{\textbf{0.336}}  &\textcolor{red}{\textbf{0.302}}  &\textcolor{red}{\textbf{0.327}} \\
         
         \cmidrule(r){2-18}
         & Avg  &0.230 &0.266 &0.242 &0.278 &\textcolor{blue}{\textbf{0.227}} &\textcolor{blue}{\textbf{0.264}}   &\textcolor{red}{\textbf{0.221}}  &\textcolor{red}{\textbf{0.256}}  &0.240 &0.280 &0.237 &0.277 &\textcolor{blue}{\textbf{0.225}}  &\textcolor{blue}{\textbf{0.266}}  &\textcolor{red}{\textbf{0.219}}  &\textcolor{red}{\textbf{0.257}}  \\
         \midrule[1pt]

        \multicolumn{2}{c|}{$1^{st}$ Count} &4  &1  &6  &0  &2  &0  &24  &34  &7  &6  &7  &2  &2  &0  &23  &28  \\
         \bottomrule[2pt]
    \end{tabular}
    }
    \label{tab:input_336_appendix}
\end{table*}

%% file: tables/ablation_appendix.tex
\begin{table}[htbp]
\centering
\caption{Full ablation study of the Phase Rotation (PR) and Amplitude Scaling (AS) modules on six datasets. The lookback window is fixed at 96, and results are averaged over four prediction lengths.}
\vspace{2mm}
\scalebox{0.68}{
\begin{tabular}{c|c|cc|cc|cc|cc|cc|cc}
\toprule[2pt]
\multicolumn{2}{c|}{Design} &\multicolumn{2}{c|}{ETTm1} &\multicolumn{2}{c|}{ETTh1} &\multicolumn{2}{c|}{Electricity}  &\multicolumn{2}{c|}{Weather}  &\multicolumn{2}{c|}{Traffic}  &\multicolumn{2}{c}{Solar-Energy} \\
\midrule

 \multicolumn{2}{c|}{Metric}   & MSE & MAE & MSE & MAE & MSE & MAE & MSE & MAE & MSE & MAE & MSE & MAE \\
    \midrule
     \multirow{5}*{\rotatebox{90}{Paloma}}   
     &96   &0.295  &0.332      &0.360  &0.378     &0.134   &0.222  &0.154  &0.190   &0.425  &0.243   &0.175  &0.224  \\
     &192  &0.343  &0.358     &0.414  &0.409    &0.151  &0.238   &0.195  &0.238   &0.442   &0.254  &0.192  &0.240  \\
     &336  &0.376  &0.381      &0.452  &0.428    &0.170  &0.256   &0.249  &0.284   &0.462   &0.262   &0.204   &0.235  \\
     &720  &0.439  &0.417    &0.458  &0.449   &0.208  &0.289    &0.324  &0.336   &0.514  &0.279   &0.215  &0.238  \\
     \cmidrule(lr){2-14}
     &Avg  &0.363  &0.372    &0.421  &0.416    &0.166  &0.251   &0.231  &0.262   &0.461  &0.260   &0.197  &0.234  \\

    \midrule
     \multirow{5}*{\rotatebox{90}{w/o PR}}   
     &96 &0.300  &0.335   &0.365  &0.383   &0.140  &0.231   &0.160  &0.197   &0.434  &0.250   &0.179  &0.232   \\
     &192 &0.346  &0.360   &0.420  &0.414   &0.156  &0.245   &0.203  &0.246   &0.452  &0.262   &0.205  &0.252   \\
     &336 &0.382  &0.385   &0.458  &0.432   &0.172  &0.261   &0.256  &0.290   &0.474  &0.269   &0.220  &0.258   \\
     &720 &0.445  &0.417   &0.464  &0.451   &0.211  &0.294   &0.332  &0.345   &0.522  &0.288   &0.221  &0.244   \\
     \cmidrule(lr){2-14}
     &Avg &0.368  &0.374   &0.427  &0.420   &0.170  &0.258   &0.238  &0.270   &0.471  &0.267   &0.206  &0.247 \\
     
    \midrule
     \multirow{5}*{\rotatebox{90}{w/o AS}}   
     &96  &0.300  &0.336   &0.366  &0.381   &0.139  &0.230   &0.163  &0.202   &0.430  &0.249   &0.178  &0.231  \\
     &192 &0.348  &0.359   &0.423  &0.420   &0.154  &0.243   &0.205  &0.245   &0.450  &0.261   &0.198  &0.246  \\
     &336 &0.378  &0.384   &0.465  &0.440   &0.174  &0.264   &0.254  &0.291   &0.471  &0.275  &0.207  &0.251  \\
     &720 &0.448  &0.418   &0.464  &0.454   &0.214  &0.297   &0.330  &0.342   &0.522  &0.290   &0.219  &0.242 \\
     \cmidrule(lr){2-14}
     &Avg &0.369  &0.374   &0.430  &0.424   &0.170  &0.259   &0.238  &0.270   &0.468  &0.269   &0.201  &0.243 \\

    \midrule
     \multirow{5}*{\rotatebox{90}{w/o PR,AS}}   
     &96 &0.314  &0.346   &0.368  &0.382   &0.142  &0.232   &0.166  &0.207   &0.443  &0.264   &0.185  &0.237   \\
     &192 &0.351  &0.362   &0.420  &0.421   &0.159  &0.248   &0.213  &0.259   &0.453  &0.281   &0.202  &0.249    \\
     &336 &0.385  &0.388   &0.461  &0.441   &0.178  &0.266   &0.259  &0.300   &0.473  &0.288    &0.225  &0.261     \\
     &720 &0.454  &0.423   &0.463  &0.455   &0.218  &0.301   &0.335  &0.353   &0.526  &0.298   &0.225  &0.244   \\
     \cmidrule(lr){2-14}
     &Avg &0.376  &0.380   &0.428  &0.425   &0.174  &0.262   &0.243  &0.280   &0.474  &0.283   &0.209  &0.248   \\
    
\bottomrule[2pt]
\end{tabular}
}
\label{tab:ablation_appendix}
\end{table}

%% file: tables/efficiency_ETTh1.tex
\begin{table}[h]
\centering
\caption{Efficiency comparison between Paloma and other models on the ETTh1 dataset with lookback window T = 96 and forecast horizon $\tau$ = 720. Training Time denotes the average time required per epoch for the model.}
\vspace{2mm}
\scalebox{0.85}{
\begin{tabular}{lcccccc}
\toprule[2pt]
Model &Parameters &MACs &Training Time &Memory Usage &MSE &MAE \\
\midrule
GTR  &350.99K  &2.5M  &2.4s   &1738MB   &0.493   &0.476   \\
CycleNet  &210.04K   &1.47M   &2.0s    &1400MB   &0.520   &0.484  \\
DLinear  &139.68K   &979.44K   &1.8s   &1398MB    &0.519   &0.516  \\
FilterNet  &369.37K   &2.4M   &2.6s    &1412MB   &0.481   &0.473  \\
iTransformer  &737.49K   &5.16M   &2.7s    &2016MB   &0.503   &0.491  \\
PatchTST  &2.75M   &60.74M   &2.9s    &1600MB   &0.500   &0.488  \\
TimeXer  &2.11M   &37.04M   &4.3s   &2046MB    &0.469   &0.461  \\
\midrule
Paloma(ours)  &591.4K   &4.07M   &2.5s  &1406MB    &0.458   &0.449  \\

\bottomrule[2pt]
\end{tabular}}\label{tab:efficiency_ETTh1}
\end{table}

%% file: tables/efficiency_traffic.tex
\begin{table}[h]
\centering
\caption{Efficiency comparison between Paloma and other models on the Traffic dataset with lookback window T = 96 and forecast horizon $\tau$ = 720. Training Time denotes the average time required per epoch for the model.}
\vspace{2mm}
\scalebox{0.85}{
\begin{tabular}{lcccccc}
\toprule[2pt]
Model &Parameters &MACs &Training Time &Memory Usage &MSE &MAE \\
\midrule
GTR  &1.1M   &827.95M   &30.3s  &2336MB   &0.514  &0.301   \\
CycleNet  &563.84K   & 362.08M   &25.5s   &1764MB  &0.502    &0.314  \\
DLinear  &139.7K   &120.61M   &22.6s   &1696MB   &0.645   &0.394  \\
FilterNet  &502.41K   &408.64M   &28.5s   &2078MB    &0.498   &0.323  \\
iTransformer  &6.72M   &5.79G   &42.3s   &5860MB   &0.467   &0.302  \\
PatchTST  &10.74M   &69.46G   &205.8s    &16614MB   &0.514   &0.320  \\
TimeXer  &11.49M   &42.16G   &257.9s    &13644MB   &0.516   &0.307   \\
\midrule
Paloma(ours)  &2.0M   &1.34G   &29.9s   &2422MB   &0.514   &0.279  \\

\bottomrule[2pt]
\end{tabular}}\label{tab:efficiency_traffic}
\end{table}

%% file: checklist.tex
\section*{NeurIPS Paper Checklist}

\begin{enumerate}

\item {\bf Claims}
    \item[] Question: Do the main claims made in the abstract and introduction accurately reflect the paper's contributions and scope?
    \item[] Answer: \answerYes{} 
    \item[] Justification: The main claims are clearly written in the abstract and introduction.
    \item[] Guidelines:
    \begin{itemize}
        \item The answer \answerNA{} means that the abstract and introduction do not include the claims made in the paper.
        \item The abstract and/or introduction should clearly state the claims made, including the contributions made in the paper and important assumptions and limitations. A \answerNo{} or \answerNA{} answer to this question will not be perceived well by the reviewers. 
        \item The claims made should match theoretical and experimental results, and reflect how much the results can be expected to generalize to other settings. 
        \item It is fine to include aspirational goals as motivation as long as it is clear that these goals are not attained by the paper. 
    \end{itemize}

\item {\bf Limitations}
    \item[] Question: Does the paper discuss the limitations of the work performed by the authors?
    \item[] Answer: \answerYes{} 
    \item[] Justification: We have discussed the limitation of our method in the Conclusion and Limitation section.
    \item[] Guidelines:
    \begin{itemize}
        \item The answer \answerNA{} means that the paper has no limitation while the answer \answerNo{} means that the paper has limitations, but those are not discussed in the paper. 
        \item The authors are encouraged to create a separate ``Limitations'' section in their paper.
        \item The paper should point out any strong assumptions and how robust the results are to violations of these assumptions (e.g., independence assumptions, noiseless settings, model well-specification, asymptotic approximations only holding locally). The authors should reflect on how these assumptions might be violated in practice and what the implications would be.
        \item The authors should reflect on the scope of the claims made, e.g., if the approach was only tested on a few datasets or with a few runs. In general, empirical results often depend on implicit assumptions, which should be articulated.
        \item The authors should reflect on the factors that influence the performance of the approach. For example, a facial recognition algorithm may perform poorly when image resolution is low or images are taken in low lighting. Or a speech-to-text system might not be used reliably to provide closed captions for online lectures because it fails to handle technical jargon.
        \item The authors should discuss the computational efficiency of the proposed algorithms and how they scale with dataset size.
        \item If applicable, the authors should discuss possible limitations of their approach to address problems of privacy and fairness.
        \item While the authors might fear that complete honesty about limitations might be used by reviewers as grounds for rejection, a worse outcome might be that reviewers discover limitations that aren't acknowledged in the paper. The authors should use their best judgment and recognize that individual actions in favor of transparency play an important role in developing norms that preserve the integrity of the community. Reviewers will be specifically instructed to not penalize honesty concerning limitations.
    \end{itemize}

\item {\bf Theory assumptions and proofs}
    \item[] Question: For each theoretical result, does the paper provide the full set of assumptions and a complete (and correct) proof?
    \item[] Answer: \answerYes{} 
    \item[] Justification: The complete proofs of the propositions are provided in Appendix~\ref{appendix_Theoretical_Supplement}.
    \item[] Guidelines:
    \begin{itemize}
        \item The answer \answerNA{} means that the paper does not include theoretical results. 
        \item All the theorems, formulas, and proofs in the paper should be numbered and cross-referenced.
        \item All assumptions should be clearly stated or referenced in the statement of any theorems.
        \item The proofs can either appear in the main paper or the supplemental material, but if they appear in the supplemental material, the authors are encouraged to provide a short proof sketch to provide intuition. 
        \item Inversely, any informal proof provided in the core of the paper should be complemented by formal proofs provided in appendix or supplemental material.
        \item Theorems and Lemmas that the proof relies upon should be properly referenced. 
    \end{itemize}

    \item {\bf Experimental result reproducibility}
    \item[] Question: Does the paper fully disclose all the information needed to reproduce the main experimental results of the paper to the extent that it affects the main claims and/or conclusions of the paper (regardless of whether the code and data are provided or not)?
    \item[] Answer: \answerYes{} 
    \item[] Justification: We have included the implementation details in Appendix~\ref{appendix_implementation_detials} and provided the code as supplementary material.
    \item[] Guidelines:
    \begin{itemize}
        \item The answer \answerNA{} means that the paper does not include experiments.
        \item If the paper includes experiments, a \answerNo{} answer to this question will not be perceived well by the reviewers: Making the paper reproducible is important, regardless of whether the code and data are provided or not.
        \item If the contribution is a dataset and\slash or model, the authors should describe the steps taken to make their results reproducible or verifiable. 
        \item Depending on the contribution, reproducibility can be accomplished in various ways. For example, if the contribution is a novel architecture, describing the architecture fully might suffice, or if the contribution is a specific model and empirical evaluation, it may be necessary to either make it possible for others to replicate the model with the same dataset, or provide access to the model. In general. releasing code and data is often one good way to accomplish this, but reproducibility can also be provided via detailed instructions for how to replicate the results, access to a hosted model (e.g., in the case of a large language model), releasing of a model checkpoint, or other means that are appropriate to the research performed.
        \item While NeurIPS does not require releasing code, the conference does require all submissions to provide some reasonable avenue for reproducibility, which may depend on the nature of the contribution. For example
        \begin{enumerate}
            \item If the contribution is primarily a new algorithm, the paper should make it clear how to reproduce that algorithm.
            \item If the contribution is primarily a new model architecture, the paper should describe the architecture clearly and fully.
            \item If the contribution is a new model (e.g., a large language model), then there should either be a way to access this model for reproducing the results or a way to reproduce the model (e.g., with an open-source dataset or instructions for how to construct the dataset).
            \item We recognize that reproducibility may be tricky in some cases, in which case authors are welcome to describe the particular way they provide for reproducibility. In the case of closed-source models, it may be that access to the model is limited in some way (e.g., to registered users), but it should be possible for other researchers to have some path to reproducing or verifying the results.
        \end{enumerate}
    \end{itemize}

\item {\bf Open access to data and code}
    \item[] Question: Does the paper provide open access to the data and code, with sufficient instructions to faithfully reproduce the main experimental results, as described in supplemental material?
    \item[] Answer:  \answerYes{} 
    \item[] Justification: We have provided the code as supplementary material, which also includes the sources of the datasets.
    \item[] Guidelines:
    \begin{itemize}
        \item The answer \answerNA{} means that paper does not include experiments requiring code.
        \item Please see the NeurIPS code and data submission guidelines (\url{https://neurips.cc/public/guides/CodeSubmissionPolicy}) for more details.
        \item While we encourage the release of code and data, we understand that this might not be possible, so \answerNo{} is an acceptable answer. Papers cannot be rejected simply for not including code, unless this is central to the contribution (e.g., for a new open-source benchmark).
        \item The instructions should contain the exact command and environment needed to run to reproduce the results. See the NeurIPS code and data submission guidelines (\url{https://neurips.cc/public/guides/CodeSubmissionPolicy}) for more details.
        \item The authors should provide instructions on data access and preparation, including how to access the raw data, preprocessed data, intermediate data, and generated data, etc.
        \item The authors should provide scripts to reproduce all experimental results for the new proposed method and baselines. If only a subset of experiments are reproducible, they should state which ones are omitted from the script and why.
        \item At submission time, to preserve anonymity, the authors should release anonymized versions (if applicable).
        \item Providing as much information as possible in supplemental material (appended to the paper) is recommended, but including URLs to data and code is permitted.
    \end{itemize}

\item {\bf Experimental setting/details}
    \item[] Question: Does the paper specify all the training and test details (e.g., data splits, hyperparameters, how they were chosen, type of optimizer) necessary to understand the results?
    \item[] Answer: \answerYes{} 
    \item[] Justification: Please see Appendix~\ref{appendix_implementation_detials} for implementation details.
    \item[] Guidelines:
    \begin{itemize}
        \item The answer \answerNA{} means that the paper does not include experiments.
        \item The experimental setting should be presented in the core of the paper to a level of detail that is necessary to appreciate the results and make sense of them.
        \item The full details can be provided either with the code, in appendix, or as supplemental material.
    \end{itemize}

\item {\bf Experiment statistical significance}
    \item[] Question: Does the paper report error bars suitably and correctly defined or other appropriate information about the statistical significance of the experiments?
    \item[] Answer: \answerYes{} 
    \item[] Justification: We have reported the standard deviations of the results for our method in Table~\ref{tab:input_96_additional}.
    \item[] Guidelines:
    \begin{itemize}
        \item The answer \answerNA{} means that the paper does not include experiments.
        \item The authors should answer \answerYes{} if the results are accompanied by error bars, confidence intervals, or statistical significance tests, at least for the experiments that support the main claims of the paper.
        \item The factors of variability that the error bars are capturing should be clearly stated (for example, train/test split, initialization, random drawing of some parameter, or overall run with given experimental conditions).
        \item The method for calculating the error bars should be explained (closed form formula, call to a library function, bootstrap, etc.)
        \item The assumptions made should be given (e.g., Normally distributed errors).
        \item It should be clear whether the error bar is the standard deviation or the standard error of the mean.
        \item It is OK to report 1-sigma error bars, but one should state it. The authors should preferably report a 2-sigma error bar than state that they have a 96\% CI, if the hypothesis of Normality of errors is not verified.
        \item For asymmetric distributions, the authors should be careful not to show in tables or figures symmetric error bars that would yield results that are out of range (e.g., negative error rates).
        \item If error bars are reported in tables or plots, the authors should explain in the text how they were calculated and reference the corresponding figures or tables in the text.
    \end{itemize}

\item {\bf Experiments compute resources}
    \item[] Question: For each experiment, does the paper provide sufficient information on the computer resources (type of compute workers, memory, time of execution) needed to reproduce the experiments?
    \item[] Answer: \answerYes{} 
    \item[] Justification: We have reported the computational resource requirements of our method in Table~\ref{tab:efficiency_weather}, Table~\ref{tab:efficiency_ETTh1}, and Table~\ref{tab:efficiency_traffic}.
    \item[] Guidelines:
    \begin{itemize}
        \item The answer \answerNA{} means that the paper does not include experiments.
        \item The paper should indicate the type of compute workers CPU or GPU, internal cluster, or cloud provider, including relevant memory and storage.
        \item The paper should provide the amount of compute required for each of the individual experimental runs as well as estimate the total compute. 
        \item The paper should disclose whether the full research project required more compute than the experiments reported in the paper (e.g., preliminary or failed experiments that didn't make it into the paper). 
    \end{itemize}
    
\item {\bf Code of ethics}
    \item[] Question: Does the research conducted in the paper conform, in every respect, with the NeurIPS Code of Ethics \url{https://neurips.cc/public/EthicsGuidelines}?
    \item[] Answer: \answerYes{} 
    \item[] Justification: Our research aligns with the NeurIPS Code of Ethics.
    \item[] Guidelines:
    \begin{itemize}
        \item The answer \answerNA{} means that the authors have not reviewed the NeurIPS Code of Ethics.
        \item If the authors answer \answerNo, they should explain the special circumstances that require a deviation from the Code of Ethics.
        \item The authors should make sure to preserve anonymity (e.g., if there is a special consideration due to laws or regulations in their jurisdiction).
    \end{itemize}

\item {\bf Broader impacts}
    \item[] Question: Does the paper discuss both potential positive societal impacts and negative societal impacts of the work performed?
    \item[] Answer: \answerNA{} 
    \item[] Justification: This work aims to improve the accuracy of time series forecasting.
To the best of our knowledge, it is not expected to introduce significant negative societal impacts.
    \item[] Guidelines:
    \begin{itemize}
        \item The answer \answerNA{} means that there is no societal impact of the work performed.
        \item If the authors answer \answerNA{} or \answerNo, they should explain why their work has no societal impact or why the paper does not address societal impact.
        \item Examples of negative societal impacts include potential malicious or unintended uses (e.g., disinformation, generating fake profiles, surveillance), fairness considerations (e.g., deployment of technologies that could make decisions that unfairly impact specific groups), privacy considerations, and security considerations.
        \item The conference expects that many papers will be foundational research and not tied to particular applications, let alone deployments. However, if there is a direct path to any negative applications, the authors should point it out. For example, it is legitimate to point out that an improvement in the quality of generative models could be used to generate Deepfakes for disinformation. On the other hand, it is not needed to point out that a generic algorithm for optimizing neural networks could enable people to train models that generate Deepfakes faster.
        \item The authors should consider possible harms that could arise when the technology is being used as intended and functioning correctly, harms that could arise when the technology is being used as intended but gives incorrect results, and harms following from (intentional or unintentional) misuse of the technology.
        \item If there are negative societal impacts, the authors could also discuss possible mitigation strategies (e.g., gated release of models, providing defenses in addition to attacks, mechanisms for monitoring misuse, mechanisms to monitor how a system learns from feedback over time, improving the efficiency and accessibility of ML).
    \end{itemize}
    
\item {\bf Safeguards}
    \item[] Question: Does the paper describe safeguards that have been put in place for responsible release of data or models that have a high risk for misuse (e.g., pre-trained language models, image generators, or scraped datasets)?
    \item[] Answer: \answerNA{} 
    \item[] Justification: This paper does not have this risk.
    \item[] Guidelines:
    \begin{itemize}
        \item The answer \answerNA{} means that the paper poses no such risks.
        \item Released models that have a high risk for misuse or dual-use should be released with necessary safeguards to allow for controlled use of the model, for example by requiring that users adhere to usage guidelines or restrictions to access the model or implementing safety filters. 
        \item Datasets that have been scraped from the Internet could pose safety risks. The authors should describe how they avoided releasing unsafe images.
        \item We recognize that providing effective safeguards is challenging, and many papers do not require this, but we encourage authors to take this into account and make a best faith effort.
    \end{itemize}

\item {\bf Licenses for existing assets}
    \item[] Question: Are the creators or original owners of assets (e.g., code, data, models), used in the paper, properly credited and are the license and terms of use explicitly mentioned and properly respected?
    \item[] Answer: \answerYes{} 
    \item[] Justification: The code and datasets used in the paper are publicly available and properly credited.
    \item[] Guidelines:
    \begin{itemize}
        \item The answer \answerNA{} means that the paper does not use existing assets.
        \item The authors should cite the original paper that produced the code package or dataset.
        \item The authors should state which version of the asset is used and, if possible, include a URL.
        \item The name of the license (e.g., CC-BY 4.0) should be included for each asset.
        \item For scraped data from a particular source (e.g., website), the copyright and terms of service of that source should be provided.
        \item If assets are released, the license, copyright information, and terms of use in the package should be provided. For popular datasets, \url{paperswithcode.com/datasets} has curated licenses for some datasets. Their licensing guide can help determine the license of a dataset.
        \item For existing datasets that are re-packaged, both the original license and the license of the derived asset (if it has changed) should be provided.
        \item If this information is not available online, the authors are encouraged to reach out to the asset's creators.
    \end{itemize}

\item {\bf New assets}
    \item[] Question: Are new assets introduced in the paper well documented and is the documentation provided alongside the assets?
    \item[] Answer: \answerNA{} 
    \item[] Justification: N/A.
    \item[] Guidelines:
    \begin{itemize}
        \item The answer \answerNA{} means that the paper does not release new assets.
        \item Researchers should communicate the details of the dataset\slash code\slash model as part of their submissions via structured templates. This includes details about training, license, limitations, etc. 
        \item The paper should discuss whether and how consent was obtained from people whose asset is used.
        \item At submission time, remember to anonymize your assets (if applicable). You can either create an anonymized URL or include an anonymized zip file.
    \end{itemize}

\item {\bf Crowdsourcing and research with human subjects}
    \item[] Question: For crowdsourcing experiments and research with human subjects, does the paper include the full text of instructions given to participants and screenshots, if applicable, as well as details about compensation (if any)? 
    \item[] Answer: \answerNA{} 
    \item[] Justification: N/A.
    \item[] Guidelines:
    \begin{itemize}
        \item The answer \answerNA{} means that the paper does not involve crowdsourcing nor research with human subjects.
        \item Including this information in the supplemental material is fine, but if the main contribution of the paper involves human subjects, then as much detail as possible should be included in the main paper. 
        \item According to the NeurIPS Code of Ethics, workers involved in data collection, curation, or other labor should be paid at least the minimum wage in the country of the data collector. 
    \end{itemize}

\item {\bf Institutional review board (IRB) approvals or equivalent for research with human subjects}
    \item[] Question: Does the paper describe potential risks incurred by study participants, whether such risks were disclosed to the subjects, and whether Institutional Review Board (IRB) approvals (or an equivalent approval/review based on the requirements of your country or institution) were obtained?
    \item[] Answer: \answerNA{} 
    \item[] Justification: N/A.
    \item[] Guidelines:
    \begin{itemize}
        \item The answer \answerNA{} means that the paper does not involve crowdsourcing nor research with human subjects.
        \item Depending on the country in which research is conducted, IRB approval (or equivalent) may be required for any human subjects research. If you obtained IRB approval, you should clearly state this in the paper. 
        \item We recognize that the procedures for this may vary significantly between institutions and locations, and we expect authors to adhere to the NeurIPS Code of Ethics and the guidelines for their institution. 
        \item For initial submissions, do not include any information that would break anonymity (if applicable), such as the institution conducting the review.
    \end{itemize}

\item {\bf Declaration of LLM usage}
    \item[] Question: Does the paper describe the usage of LLMs if it is an important, original, or non-standard component of the core methods in this research? Note that if the LLM is used only for writing, editing, or formatting purposes and does \emph{not} impact the core methodology, scientific rigor, or originality of the research, declaration is not required.
    \item[] Answer: \answerNA{} 
    \item[] Justification: LLMs are used only for grammar and spell checking.
    \item[] Guidelines:
    \begin{itemize}
        \item The answer \answerNA{} means that the core method development in this research does not involve LLMs as any important, original, or non-standard components.
        \item Please refer to our LLM policy in the NeurIPS handbook for what should or should not be described.
    \end{itemize}

\end{enumerate}